\documentclass{article}
\usepackage{matharticle, amsthm, amsmath}
\usepackage{fullpage}
\usepackage{amsmath,amsfonts}
\usepackage{graphicx}
\usepackage{subfigure}
\usepackage{booktabs} % for professional tables
\usepackage{algorithmic}%x}
\usepackage{thmtools}
\usepackage{thm-restate}
\usepackage[symbol]{footmisc}

\usepackage{hyperref}
\usepackage{color}

\newcommand{\sH}{\mathcal{H}}

\newcommand{\cY}{\mathcal{Y}}
\newcommand{\bR}{\mathbb{R}}
\newcommand{\bE}{\mathbb{E}}
\DeclareMathOperator{\Cov}{Cov}

\newcommand{\oo}{O}

\renewcommand{\epsilon}{\varepsilon}

\theoremstyle{definition}
\newtheorem{assumption}[theorem]{Assumption}
\newtheorem{claim}[theorem]{Claim}

\newcommand{\xCen}{{\bf l2}}
\newcommand{\noDef}{{\bf noDefense}}
\newcommand{\Loss}{{\bf loss}}
\newcommand{\gUnc}{{\bf gradient}}
\newcommand{\gCen}{{\bf gradientCentered}}
\newcommand{\ransac}{{\bf RANSAC}}

\newcommand{\sever}{{\textsc{Sever}}}
\newcommand{\filter}{{\textsc{Filter}}}

\newcommand{\badfrac}{\epsilon}
\newcommand{\strconv}{\xi}
\newcommand{\strsmooth}{\beta}
\newcommand{\lip}{L}
\newcommand{\Ef}{\bar{f}}
\newcommand{\fbar}{\Ef}

\newcommand{\Efunc}{\Ef}
\newcommand{\fhat}{\hat{f}}
\newcommand{\hatf}{\fhat}
\newcommand{\goodset}{I_{\mathrm{good}}}
\newcommand{\Sgood}{S_{\mathrm{good}}}
\newcommand{\Sbad}{S_{\mathrm{bad}}}
\newcommand{\mugood}{\mu_{\mathrm{good}}}
\newcommand{\mubad}{\mu_{\mathrm{bad}}}

\newcommand{\dom}{\mathcal{H}}
\newcommand{\what}{\widehat{w}}
\newcommand{\wstar}{w^*}
\newcommand{\radius}{r}

\newcommand{\citep}{\cite}
\newcommand{\citet}{\cite}

\newcommand{\eqdef}{\stackrel{\text{def}}{=}}
\newcommand{\sL}{\mathcal{L}}

\def\colorful{0}

\ifnum\colorful=1
\newcommand{\new}[1]{{\color{red} #1}}

\else
\newcommand{\new}[1]{{#1}}

\fi

\title{\sever: A Robust Meta-Algorithm for Stochastic Optimization}

\author {
Ilias Diakonikolas \thanks{Supported by NSF Award CCF-1652862 (CAREER) and a Sloan Research Fellowship.}\\
CS, USC \\
\tt{diakonik@usc.edu}
\and
Gautam Kamath \thanks{Work done while a graduate student at MIT and a Microsoft Research Fellow, as part of the Simons-Berkeley Research Fellowship program, and supported by NSF Award CCF-1617730, CCF-1650733, CCF-1741137, and ONR N00014-12-1-0999.} \\
CS, University of Waterloo \\ 
\tt{g@csail.mit.edu}
\and
Daniel M. Kane \thanks{Supported by NSF Award CCF-1553288 (CAREER) and a Sloan Research Fellowship.} \\
CSE \& Math, UCSD \\
\tt{dakane@cs.ucsd.edu}
\and
Jerry Li \thanks{Work done while a graduate student at MIT and a VMware Research Fellow, as part of the Simons-Berkeley Research Fellowship program, and supported by NSF Award CCF-1453261 (CAREER), CCF-1565235, a Google Faculty Research Award, and an NSF Graduate Research Fellowship.}\\
Microsoft Research AI \\
\tt{jerryzli@mit.edu}
\and
Jacob Steinhardt \thanks{Work done while a graduate student at Stanford University, and supported by a Fannie \& John Hertz Foundation Fellowship, a NSF Graduate Research Fellowship, and a Future of Life Institute grant.}\\
Statistics, UC Berkeley \\
\tt{jsteinha@stanford.edu}
\and
Alistair Stewart \thanks{Work done while a postdoc at USC, and supported by a USC startup grant.}\\
Web3 Foundation \\
\tt{stewart.al@gmail.com}
}

\usepackage{tikz}
\usetikzlibrary{positioning}
\tikzset{%
            base/.style = {rectangle, rounded corners, draw=black,
                           minimum width=1.8cm, minimum height=1.15cm,
                           text centered, font=\sffamily},
}

\begin{document}
\maketitle \footnotetext{Authors are in alphabetical order.}
\footnotetext{Code is available at~\url{https://github.com/hoonose/sever}.}
%!TEX root = ./main.tex 

\begin{abstract}
In high dimensions, most machine learning methods are brittle to even a small 
fraction of structured outliers. To address this, we introduce a new 
meta-algorithm that can take in a \emph{base learner} such as least squares or stochastic 
gradient descent, and harden the learner to be resistant to outliers.
Our method, \sever, possesses strong theoretical guarantees yet is also highly scalable---beyond 
running the base learner itself, it only requires computing the top singular vector of a certain
$n \times d$ matrix. %where $n$ is the number of data points and $d$ is the dimension.
We apply \sever{} on a drug design dataset and a spam classification dataset, and 
find that in both cases it has substantially greater robustness than several baselines.
On the spam dataset, with $1\%$ corruptions, we achieved $7.4\%$ test error, compared to $13.4\%-20.5\%$ for the baselines, and $3\%$ error on the uncorrupted dataset.
Similarly, on the drug design dataset, with $10\%$ corruptions, we achieved $1.42$ mean-squared error test error, compared to $1.51-2.33$ for the baselines, and $1.23$ error on the uncorrupted dataset. 
\end{abstract}

%!TEX root = ./main.tex

\section{Introduction}

Learning in the presence of outliers is a ubiquitous challenge in machine learning; 
nevertheless, most machine learning methods are very sensitive to outliers in 
high dimensions. The focus of this work is on designing algorithms that are 
outlier robust while remaining competitive in terms of accuracy and running time.

We highlight two motivating applications. The first is biological data (such as gene expression 
data), where mislabeling or measurement errors can create systematic outliers \citep{RP-Gen02,Li-Science08} that require 
painstaking manual effort to remove \citep{Pas-MG10}. Detecting outliers in such settings is often important
either because the outlier observations are of interest themselves 
or because they might contaminate the downstream statistical analysis.
The second motivation is machine learning security, where outliers can be introduced through 
\emph{data poisoning} attacks \citep{barreno2010security} in which an adversary inserts fake data into the training set 
(e.g., by creating a fake user account).
Recent work has shown that for high-dimensional datasets, even a small fraction 
of outliers can substantially degrade the learned model \citep{biggio2012poisoning,
newell2014practicality,koh2017understanding,steinhardt2017certified,koh2018stronger}.

Crucially, in both the biological and security settings above, the outliers are not ``random'' 
but are instead highly correlated, and could have a complex internal structure that is difficult 
to model. This leads us to the following conceptual question underlying the present work:
\emph{Can we design training algorithms that are robust to the presence of an 
$\epsilon$-fraction of arbitrary (and potentially adversarial) outliers?}

% [3] Schölkopf, Bernhard, et al. "Support vector method for novelty detection." Advances in neural information processing systems. 2000.

Estimation in the presence of outliers is a prototypical goal in robust 
statistics and has been systematically studied since the pioneering work of Tukey \citet{tukey1960survey}.
Popular methods include RANSAC \citep{fischler1981random}, minimum covariance determinant \citep{rousseeuw1999fast}, 
removal based on $k$-nearest neighbors \citep{breunig2000lof}, and Huberizing the loss \citep{owen2007robust} 
(see \citet{hodge2004survey} for a comprehensive survey).
However, these classical methods either break down in high dimensions, or only handle ``benign'' 
outliers that are obviously different from the rest of the data (see Section~\ref{sec:related-work} 
for futher discussion of these points).

%Thus, until recently, all known efficient estimators suffered from large 
%errors in the presence of a small fraction of outliers, 
%even for the most basic high-dimensional tasks that are common in machine learning. 
Motivated by this, 
recent work in theoretical computer science has developed efficient robust estimators
for classical problems such as linear classification \citep{klivans2009learning,awasthi2014power}, 
mean and covariance estimation \citep{DKKLMS16, LaiRV16}, clustering \citep{CSV17}, 
and regression \citep{BhatiaJK15, BhatiaJKK17, BDLS17}. 
%
%Despite this recent algorithmic progress, 
Nevertheless, 
the promise of practical high-dimensional 
robust estimation is yet to be realized; indeed, the aforementioned results generally suffer from one of 
two shortcomings--either they use sophisticated convex optimization algorithms 
that do not scale to large datasets, or they are tailored to specific problems of interest 
or specific distributional assumptions on the data, and hence do not have good accuracy 
on real data.

In this work, we address these shortcomings. We propose an algorithm, \sever, that is:
\begin{itemize}
\item {\bf Robust:} it can handle arbitrary outliers with only a small increase in error, even in high dimensions.
\item {\bf General:} it can be applied to most common learning problems including regression and classification, 
and handles non-convex models such as neural networks.
\item {\bf Practical:} the algorithm can be implemented with standard machine learning libraries.
\end{itemize}
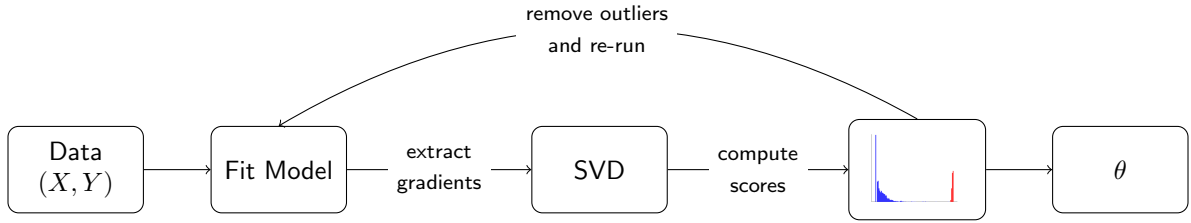
\begin{figure*}[t!]
\centering
\begin{tikzpicture}[node distance=2.7cm,
    every node/.style={fill=white, font=\sffamily}, align=center]
  % Specification of nodes (position, etc.)
  \node (Data)            [base]                          {Data \\ $(X,Y)$ };
  \node (FitModel)        [base, right of=Data]          {Fit Model};
  \node (SVD)             [base, right=2.45cm of FitModel]   {SVD};
  \node (Histogram)       [base, right=2.4cm of SVD]   { \includegraphics[height=1.1cm]{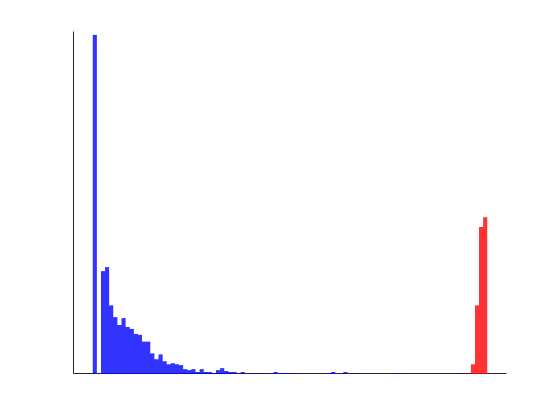} };
  \node (Params)          [base, right of=Histogram] {$\theta$};
  % Specification of lines between nodes specified above
  % with aditional nodes for description 
  \draw[->]             (Data) -- (FitModel);
  \draw[-]     (FitModel) edge[->] node {\footnotesize extract \\ \footnotesize gradients} (SVD);
  \draw[->]      (SVD) edge[->] node {\footnotesize compute \\ \footnotesize scores} (Histogram);
  \draw[->]      (Histogram) -- (Params);
  \draw[-]      (Histogram.north) edge[bend right,->] node {\footnotesize remove outliers \\ \footnotesize and re-run} (FitModel.north);
%  \draw[->]     (onResumeBlock) -- (activityRuns);
%  \draw[->]      (activityRuns) -- node[text width=4cm]
%                                   {Another activity comes in
%                                    front of the activity} (onPauseBlock);
%  \draw[->]      (onPauseBlock) -- node {The activity is no longer visible}
%                                   (onStopBlock);
%  \draw[->]       (onStopBlock) -- node {The activity is shut down by
%                                   user or system} (onDestroyBlock);
%  \draw[->]    (onRestartBlock) -- (onStartBlock);
%  \draw[->]       (onStopBlock) -| node[yshift=1.25cm, text width=3cm]
%                                   {The activity comes to the foreground}
%                                   (onRestartBlock);
%  \draw[->]    (onDestroyBlock) -- (ActivityDestroyed);
%  \draw[->]      (onPauseBlock) -| node(priorityXMemory)
%                                   {higher priority $\rightarrow$ more memory}
%                                   (ActivityEnds);
%  \draw           (onStopBlock) -| (priorityXMemory);
%  \draw[->]     (ActivityEnds)  |- node [yshift=-2cm, text width=3.1cm]
%                                    {User navigates back to the activity}
%                                    (onCreateBlock);
%  \draw[->] (onPauseBlock.east) -- ++(2.6,0) -- ++(0,2) -- ++(0,2) --                
%     node[xshift=1.2cm,yshift=-1.5cm, text width=2.5cm]
%     {The activity comes to the foreground}(onResumeBlock.east);
  \end{tikzpicture}
  \caption{Illustration of the \sever{} pipeline. We first use any machine learning algorithm to fit a model to the data. Then, we extract gradients for each data point at the learned parameters, and take the singular value decomposition of the gradients. We use this to compute an outlier score for each data point. If we detect outliers, we remove them and re-run the learning algorithm; otherwise, we output the learned parameters.} % \new{save line}}
  \label{fig:pipeline}
\end{figure*}

At a high level, our algorithm (depicted in Figure~\ref{fig:pipeline} and 
described in detail in Section~\ref{sec:alg}) is a simple ``plug-in'' outlier detector--first, 
run whatever learning procedure would be run normally (e.g., least squares in the case of linear regression). 
Then, consider the matrix of gradients at the optimal parameters, and compute the top singular 
vector of this matrix. Finally, remove any points whose projection onto this singular vector 
is too large (and re-train if necessary). 

Despite its simplicity, our algorithm possesses strong theoretical guarantees:
As long as the real (non-outlying) data is not too heavy-tailed, 
\sever{} is provably robust to outliers--see Section~\ref{sec:theory} for detailed 
statements of the theory. At the same time, we show that our algorithm 
works very well in practice and outperforms a number of natural baseline outlier detectors. 
In line with our original motivating biological and security applications, 
we implement our method on two tasks--a linear regression task 
for predicting protein activity levels %based on the ChEMBL database 
\citep{olier2018qsar}, 
and a spam classification task based on emails from the Enron corporation 
\citep{metsis2006spam}. Even with a small fraction of outliers, baseline methods perform 
poorly on these datasets; for instance, on the Enron spam dataset with 
a $1\%$ fraction of outliers, baseline errors range from $13.4\%$ to $20.5\%$, 
while \sever{} incurs only $7.3\%$ error (in comparison, the error is $3\%$ in the absence 
of outliers).
Similarly, on the drug design dataset, with $10\%$ corruptions, \sever{} achieved $1.42$ mean-squared error test error, compared to $1.51$-$2.33$ for the baselines, and $1.23$ error on the uncorrupted dataset.

%\new{TODO: add some bridge / outline of rest of paper}

%\input{intro-results}

\subsection{Comparison to Prior Work}
\label{sec:related-work}

As mentioned above, the myriad classical approaches to robust estimation perform 
poorly in high dimensions or in the presence of worst-case outliers. For instance, 
RANSAC \citep{fischler1981random} works by removing enough points at random that no outliers remain with 
decent probability; since we need at least $d$ points to fit a $d$-dimensional model, this requires the 
number of outliers to be $\oo(1/d)$. 
$k$-nearest neighbors \citep{breunig2000lof} similarly 
suffers from the curse of dimensionality when $d$ is large. 
The minimum covariance determinant estimator \citep{rousseeuw1999fast} 
only applies when the number of data points $n$ exceeds $2d$, which does not hold for 
the datasets we consider (it also has other issues such as computational intractability). 
A final natural approach is to limit the effect of points with large loss (via e.g.~Huberization 
\citep{owen2007robust}), but as \citet{koh2018stronger} show (and we confirm in our experiments), 
correlated outliers often have \emph{lower} loss than the real data under the learned model.

These issues have motivated work on high-dimensional robust statistics going back to 
Tukey \citep{Tukey75}. However, it was not until much later that efficient algorithms with 
favorable properties were first proposed. 
\citep{klivans2009learning} gave the first efficient algorithms
for robustly classification %learning origin-centered linear separators with respect to the 0-1 loss 
under the assumption that the distribution of the good data is isotropic and log-concave. 
Subsequently,~\cite{awasthi2014power} obtained an improved and nearly optimal robust algorithm 
for this problem.
%that achieves 
%nearly tight error bounds for this learning problem.
Two concurrent works~\citep{DKKLMS16, LaiRV16} gave the first efficient robust 
estimators for several other tasks including mean and covariance estimation.
There has since been considerable study of 
algorithmic robust estimation in high dimensions,
including learning graphical models~\citep{DiakonikolasKS16b}, 
understanding computation-robustness tradeoffs~\citep{DKS17-sq, DKKLMS17}, 
establishing connections to PAC learning~\citep{DKS17-nasty}, 
tolerating more noise by outputting a list of hypotheses~\citep{CSV17, meister2017data, DKS17-mixtures}, 
robust estimation of discrete structures~\citep{steinhardt2017clique,qiao2017learning,steinhardt2018resilience},
and robust estimation via sum-of-squares~\citep{KS17, HL17, KStein17}. 

Despite this progress, these recent theoretical papers typically focus on designing specialized algorithms for
specific settings (such as mean estimation or linear classification for specific families of distributions) 
rather than on designing general algorithms. %\todo{previous sentence is verbose, can trim}
The only exception is \citep{CSV17}, which provides a robust meta-algorithm for stochastic convex optimization in a 
similar setting to ours. However, that algorithm (i) requires solving a large semidefinite program
and (ii) incurs a significant loss in performance relative to standard training {\em even in the absence of outliers}.
On the other hand, \citep{DKK+17} provide a practical implementation of 
the robust mean and covariance estimation algorithms of \citep{DKKLMS16}, 
but do not consider more general learning tasks.

%In the concrete settings of regression and classification considered in our experiments, there 
%has been a substantial number of related works. 
A number of papers~\citep{nasrabadi2011robust, nguyen2013exact, BhatiaJK15, BhatiaJKK17} 
have proposed efficient algorithms for a type of robust linear regression.  
However, these works consider a restrictive corruption model 
that only allows adversarial corruptions to the responses (but not the covariates).
On the other hand, \citep{BDLS17} studies (sparse) linear regression and, more broadly, 
generalized linear models (GLMs) under a robustness model very similar 
to the one considered here. The main issues with this algorithm are that (i) 
it requires running the ellipsoid method (hence does not scale) and (ii) it crucially assumes 
Gaussianity of the covariates, which is unlikely to hold in practice.
%The three major differences between the results of \cite{BDLS17}
%and our work are as follows: (1) The algorithm employed in \cite{BDLS17}, 
%building on the convex programming method of~\cite{DKKLMS16}, makes essential use of the 
%ellipsoid method (whose separation oracle is another convex program), 
%hence does not scale in practice. (2) The distribution of the covariates is assumed to be a 
%standard Gaussian distribution, while our results apply under much weaker moment conditions. 
%(3) The sample complexity of their algorithm scales at least quadratically in the dimension, 
%while our algorithms have near-linear sample complexity for GLMS. \new{[Check]}

In a related direction, \citet{steinhardt2017certified} provide a method for analyzing outlier 
detectors in the context of linear classification, either certifying robustness or 
generating an attack if the learner is not robust.
The outlier detector they analyze is brittle in high dimensions, 
motivating the need for the robust algorithms presented in the current work. 
Later work by the same authors showed how to bypass a number of common outlier detection methods 
\citep{koh2018stronger}. We use these recent strong attacks as part of our evaluation and 
show that our algorithm is more robust.
%\todo{a bit wordy, not sure if easy to follow?}
%In fact, the attacks on the spam classifier were generated based 
%on an improved version of the algorithm from \citet{steinhardt2017certified} communicated to 
%us by the authors.

\paragraph{Concurrent Works.} \cite{PSBR18} independently obtained a robust algorithm
for stochastic convex optimization by combining gradient descent with robust mean estimation. 
This algorithm is similar to the one we present in Appendix~\ref{sec:general-algo}, and in that section we discuss in more detail the comparison between these two techniques.
For the case of linear
regression,~\cite{DKoS18} provide efficient robust algorithms with near-optimal error guarantees under various distributional
assumptions and establish matching computational-robustness tradeoffs.

%!TEX root = ./main.tex

\section{Framework and Algorithm}
\label{sec:theory}

In this section, we describe our formal framework as well as the \sever{} algorithm.

\subsection{Formal Setting}

We will consider stochastic optimization tasks, where there is some true
distribution $p^{\ast}$ over functions $f : \sH \to \bR$, and our goal is to
find a parameter vector $w^{\ast} \in \sH$ minimizing $\overline{f}(w) \eqdef \bE_{f \sim p^{\ast}}[f(w)]$.
Here we assume $\sH \subseteq \bR^d$ is a space of possible parameters.
As an example, we consider linear regression with squared loss, where $f(w) = \frac{1}{2}(w \cdot x - y)^2$
for $(x,y)$ drawn from the data distribution; or support vector machines with hinge loss, where
$f(w) = \max\{ 0, 1 - y (w \cdot x) \}$.
We will use the former as a running example for the theory part of the body of this paper.

To help us learn the parameter vector $w^{\ast}$, we have access to a \emph{training set}
of $n$ functions $f_{1:n} \eqdef \{f_1, \ldots, f_n\}$. (For linear regression, we would have
$f_i(w) = \frac{1}{2}(w \cdot x_i - y_i)^2$, where $(x_i, y_i)$ is an observed data point.)
However, unlike the classical (uncorrupted) setting where we assume that
$f_1, \ldots, f_n \sim p^{\ast}$, we allow for an $\epsilon$-fraction of the points to be
arbitrary outliers:

\begin{definition}[$\epsilon$-contamination model] \label{def:eps-contam}
Given $\eps > 0$ and a distribution $p^{\ast}$ over functions $f : \sH \to \bR$, data is generated as follows:
first, $n$ clean samples $f_1, \ldots, f_{n}$ are drawn from $p^{\ast}$.
Then, an \emph{adversary} is allowed to inspect the samples and replace
any $\epsilon n$ of them with arbitrary samples.
The resulting set of points is then given to the algorithm. 
\new{We will call such a set of samples {\em $\eps$-corrupted (with respect to $p^{\ast}$)}.}
\end{definition}

In the $\epsilon$-contamination model, the adversary is allowed to both add and remove points.
Our theoretical results hold in this strong robustness model.
Our experimental evaluation uses corrupted instances
in which the adversary is only allowed to add corrupted points. Additive corruptions
essentially correspond to Huber's contamination model~\citet{Huber64}
in robust statistics.

%An illustration of the $\epsilon$-contamination model is shown in Figure~\ref{fig:model}.
%\todo{show something about linear regression with outliers}

%\begin{definition}[$\epsilon$-outlier model] \label{def:adv-outlier}
%Given $\eps > 0$ and a true distribution $p^{\ast}$, data is generated as follows:
%first, $(1-\epsilon) n$ clean samples $f_1, \ldots, f_{(1-\epsilon) n}$ are drawn from $p^{\ast}$.
%Then, an \emph{adversary} is allowed to inspect the samples and insert
%$\epsilon n$ additional arbitrary samples.
%The resulting set of points is then given to the algorithm.
%\todo{we should try to keep somewhat-consistent terminology between
%functions, samples, and points}
%\end{definition}

Finally, we will often assume access to a black-box learner, which we denote
by $\sL$, which takes in functions $f_1, \ldots, f_n$ and outputs a
parameter vector $w \in \sH$. We want to stipulate that $\sL$ approximately
minimizes $\frac{1}{n} \sum_{i=1}^n f_i(w)$. For this purpose, we introduce the
following definition:

\begin{restatable}[$\gamma$-approximate critical point]{definition}{approxcrit}
\label{def:approx-crit}
Given a function $f:\dom\rightarrow \R$, a $\gamma$-approximate critical point of $f$,
is a point $w\in \dom$ so that for all unit vectors $v$ where $w+\delta v\in \dom$
for arbitrarily small positive $\delta$, we have that $v\cdot \nabla f(w) \geq -\gamma$.
\end{restatable}

Essentially, the above definition means that the value of $f$ cannot be decreased
much by changing the input $w$ locally, while staying within the domain.
The condition enforces that moving in any direction $v$ either causes us
to leave $\dom$ or causes $f$ to decrease at a rate at most $\gamma$.
\new{It should be noted that when $\dom = \R^d$, our above notion of approximate
critical point reduces to the standard notion of approximate stationary point
(i.e., a point where the magnitude of the gradient is small).}

We are now ready to define the notion of a \emph{$\gamma$-approximate} learner:

\begin{restatable}[$\gamma$-approximate learner]{definition}{approxlearner}
\label{def:approx-learner}
A learning algorithm $\sL$ is called \emph{$\gamma$-approximate} if, for any
functions $f_1, \ldots, f_n : \dom \to \bR$ each bounded below on a closed domain $\dom$, the output ${w} = \sL(f_{1:n})$ of
$\sL$ is a $\gamma$-approximate critical point of $f(x):=\frac{1}{n}\sum_{i=1}^n f_i(x)$.
\end{restatable}
In other words, $\sL$ always finds an approximate critical point of the empirical
learning objective. We note that most common learning algorithms (such as stochastic gradient
descent) satisfy the $\gamma$-approximate learner property.
For our example of linear regression, gradient descent could be performed using the gradient $\frac1n\sum_{i=1}^n x_i(w \cdot x_i - y_i)$.
However, in some cases, a more efficient and direct method is to set the gradient equal to 0 and solve for $w$.
In our linear regression example, this gives us a closed form solution for the optimal parameter vector.

\subsection{Algorithm and Theory}
\label{sec:alg}

As outlined in Figure~\ref{fig:pipeline}, our algorithm works by post-processing
the gradients of a black-box learning algorithm.
The basic intuition is as follows: we want to ensure that the outliers do not have a
large effect on the learned parameters. Intuitively, for the outliers to have such
an effect, their corresponding gradients should be (i) large in magnitude and (ii)
systematically pointing in a specific direction. We can detect this via singular value
decomposition--if both (i) and (ii) hold then the outliers should be responsible for a
large singular value in the matrix of gradients, which allows us to detect and remove them.

%Intuitively, if the learner ends up at parameters
%$\hat{w}$ that are far away from $w^{\ast}$, the average gradient across the clean data must be large
%(since it will be pointing strongly in the direction of $w^{\ast}$). On the other hand,
%the average gradient at $\hat{w}$ must equal zero because it is the optimum of the learning
%objective.

This is shown more formally via the pseudocode in Algorithm~\ref{alg:sever}.

%!TEX root = ./main.tex

\begin{algorithm}[ht]
   \caption{\sever{}$(f_{1:n}, \sL, \sigma)$}
   \label{alg:sever}
\begin{algorithmic}[1]
\STATE {\bfseries Input:} Sample functions $f_1, \ldots, f_n : \dom \to \bR$, bounded below on a closed domain $\dom$, $\gamma$-approximate learner $\sL$, and parameter $\sigma \in \R_+$.
\STATE Initialize $S \gets \{1,\ldots,n\}$.
\REPEAT
  \STATE ${w} \gets \sL(\{f_i\}_{i \in S})$. $\triangleright$ Run approximate learner on points in $S$.
  \STATE Let $\widehat{\nabla} = \frac{1}{|S|} \sum_{i\in S} \nabla f_i(w)$.
  \STATE Let $G = [\nabla f_i({w}) - \widehat{\nabla}]_{i \in S}$ be the $|S| \times d$ matrix of centered gradients.
  \STATE Let $v$ be the top right singular vector of $G$.
  \STATE Compute the vector $\tau$ of \emph{outlier scores} defined via
  $\tau_i = \left((\nabla f_i({w}) - \widehat{\nabla}) \cdot v\right)^2$.
  \STATE $S' \gets S$
  %\STATE \label{filter-step} Remove the $i$ with the largest scores $s_i$ from $S$. $\triangleright$ see Algorithm~\ref{alg:filter}
  \STATE \label{filter-step} $S \gets \filter(S', \tau, \sigma)$ $\triangleright$ Remove some $i$'s
  with the largest scores $\tau_i$ from $S$; see Algorithm~\ref{alg:filter}.

\UNTIL{$S = S'$.} \label{until-step}
\STATE Return $w$.
\end{algorithmic}
\end{algorithm}

\begin{algorithm}[ht]
   \caption{\filter$(S, \tau, \sigma)$}
   \label{alg:filter}
\begin{algorithmic}[1]
\STATE {\bfseries Input:} Set $S \subseteq [n]$, vector $\tau$ of outlier scores, and parameter $\sigma \in \R_+$.
  \STATE If $\sum_i \tau_i \leq c \cdot \sigma$, for some constant $c>1$, 
  return $S$ $\triangleright$ We only filter out points if the variance is larger than an appropriately chosen threshold.

  \STATE Draw $T$ from the uniform distribution on $[0,\max_i \tau_i]$.
  \STATE Return $\{i \in S: \tau_i < T \}$.
\end{algorithmic}
\end{algorithm}

%%Draw T from the uniform distribution on [0,\max_i ?_i]
%Return {i \in S: ?_i < T }

For concreteness, we describe how the algorithm would work for our running example of linear regression.
First, we would solve for the optimal parameter vector on the dataset, disregarding issues of robustness. 
Specifically, we let $\hat w$ be the solution to $\sum_{i=1}^n x_i (\hat w \cdot x_i - y_i) = 0$: setting the gradient equal to $0$ will give us a critical point, as desired.
We compute the average gradient, $\frac1n\sum_{i=1}^n x_i (\hat w \cdot x_i - y_i),$ and use this to compute the matrix of centered gradients $G$.
That is, the $j$th row of $G$, $G_j$, is the vector $x_j (\hat w \cdot x_j - y_j) - \frac1n\sum_{i=1}^n x_i (\hat w \cdot x_i - y_i)$.
We compute the top right singular vector of $G$, project the data into this direction, and square the resulting magnitudes to derive a score for each point: $\tau_j = (G_j\cdot v)^2$.
With these scores in place, we run Algorithm~\ref{alg:filter}, to (randomly) remove some of the points with the largest scores.
We re-run the entire procedure on this subset of points, until Algorithm~\ref{alg:filter} does not remove any points, at which point we terminate.

\paragraph{Theoretical Guarantees.}
Our first theoretical result says that as long as the data is not too heavy-tailed,
\sever{} will find an approximate critical point of the true function $\overline{f}$,
even in the presence of outliers.

\begin{theorem} \label{thm:stationary-point-inf}
Suppose that functions $f_1,\ldots,f_n,\Ef:\dom\rightarrow \R$ are bounded below on a closed domain $\dom$,
and suppose that they satisfy the following deterministic regularity conditions: There exists a set $\goodset \subseteq [n]$
with $|\goodset| \geq (1-\eps)n$ and $\sigma>0$ such that
\begin{itemize}
\item[(i)] $\Cov_{\goodset}[\nabla f_i(w)] \preceq \sigma^2 I$, $w \in \dom$,
\item[(ii)] $\|\nabla \hatf(w) - \nabla \Ef(w)\|_2 \leq \sigma \sqrt{\badfrac}$,
$w \in \dom$, where $\hatf \eqdef (1/|\goodset|) \sum_{i \in \goodset} f_i$.
\end{itemize}
Then our algorithm \sever{} applied to $f_1,\ldots,f_n, \sigma$ returns a point $w \in \dom$
that, with probability at least $9/10$, is a $(\gamma+O(\sigma \sqrt{\eps}))$-approximate critical point of $\Ef$.
\end{theorem}

The key take-away from Theorem~\ref{thm:stationary-point-inf}
is that the error guarantee has no dependence on
the underlying dimension $d$. In contrast, most natural algorithms incur an
error that grows with $d$, and hence have poor robustness in high dimensions.

\new{We show that under some niceness assumptions on $p^{\ast}$,
the deterministic regularity conditions are satisfied with high probability
with polynomially many samples:

\begin{proposition}[Informal] \label{prop:sample-bound-inf}
Let $\dom \subset \R^d$ be a closed bounded set with diameter at most $r$.
Let $p^{\ast}$ be a distribution over functions $f:\dom\rightarrow \R$ and $\Ef=\E_{f \sim p^{\ast}}[f]$.
Suppose that for each $w\in \dom$ and unit vector $v$ we have
$\E_{f \sim p^{\ast}}[(v\cdot (\nabla f(w)-\Ef(w)))^2] \leq \sigma^2.$
Under appropriate Lipschitz and smoothness assumptions,
for $n = \Omega( d\log(r/(\sigma^2\eps))/(\sigma^2 \eps))$,
an $\eps$-corrupted set of functions drawn i.i.d. from $p^\ast$, $f_1,\ldots,f_n$  with high probability
satisfy conditions (i) and (ii).
\end{proposition}

\noindent The reader is referred to Proposition~\ref{prop:sample-bound} in the appendix for a detailed formal statement.

}

While Theorem~\ref{thm:stationary-point-inf} is very general and holds even for non-convex
loss functions, we might in general hope for more than an approximate critical point.
In particular, for convex problems, we can guarantee that we find an approximate global minimum.
This follows as a corollary of Theorem~\ref{thm:stationary-point-inf}:

\begin{corollary}\label{cor:convex-sever-inf}
Suppose that $f_1, \ldots, f_n: \dom \to \R$ satisfy the regularity conditions (i) and (ii),
and that $\dom$ is convex with $\ell_2$-radius r.
Then, with probability at least $9/10$, the output of \sever{} satisfies the following:
\begin{enumerate}
\item[(i)] If $\Efunc$ is convex, the algorithm finds a $w \in \dom$ such that
       $\Ef(w) - \Ef(\wstar) = O((\sigma \sqrt{\badfrac} + \gamma) r)$.
\item[(ii)] If $\Efunc$ is $\strconv$-strongly convex,  the algorithm finds a $w \in \dom$ such that
$       \Ef(w) - \Ef(\wstar) = O\left((\eps\sigma^2 + \gamma^2)/{\strconv} \right)$.
\end{enumerate}
\end{corollary}

\paragraph{Practical Considerations.}
For our theory to hold, we need to use the randomized filtering algorithm
shown in Algorithm~\ref{alg:filter} (which is essentially the robust mean estimation
algorithm of~\cite{DKK+17}), and filter until the stopping condition
in line \ref{until-step} of Algorithm~\ref{alg:sever} is satisfied. However,
in practice we found that the following simpler algorithm worked well: in
each iteration simply remove the top $p$ fraction of outliers according to the scores
$\tau_i$, and instead of using a specific stopping condition, simply repeat
the filter for $r$ iterations in total. This is the version of \sever{} that
we use in our experiments in Section~\ref{sec:experiments}.

\paragraph{Concrete Applications}
We also provide several concrete applications of our general theorem, 
particularly involved with optimization problems related to learning generalized linear models. 
In this setting, we are given a set of pairs $(X,Y)$ where $X\in \R^d$ and $Y$ is in some (usually discrete) set. 
One then tries to find some vector $w$ that minimizes some appropriate loss function $L(w,(X,Y)) = \sigma_Y(w\cdot X)$. 
For example, the standard hinge-loss has $Y\in \{\pm 1\}$ and $L = \max(0,1-Y(w\cdot X))$. 
Similarly, the logistic loss function is $-\log(1+\exp(-Y(w\cdot X)))$. In both cases, we show that an approximate 
minimizer to the empirical loss function can be found with a near-optimal number of samples, 
even under $\eps$-corruptions (for exact theorem statements see Theorems \ref{thm:result-svm} and \ref{thm:logreg}).
\begin{theorem}[Informal Statement]
Let $D_{X,Y}$ be a distribution over $\R^d\times\{\pm 1\}$ so that $\E[XX^T]\preceq I$ 
and so that not too many $X$ values lie near any hyperplane. 
Let $(X_1,Y_1),\ldots,(X_n,Y_n)$ be $n=\tilde O(d/\eps)$ $\eps$-corrupted samples from $D_{X,Y}$. 
Let $L$ be either the hinge loss or logistic loss function. 
Then there exists a polynomial time algorithm that with probability $9/10$ returns a vector $w$ that minimizes
$
\E_{(X,Y)\sim D_{X,Y}}(L(w,(X,Y)))
$ up to an additive $\tilde O(\eps^{1/4})$ error.
\end{theorem}

Another application allows us to use the least-squares loss function ($L(w,(X,Y))=(Y-w\cdot X)^2$) 
to perform linear regression under somewhat more restrictive assumptions (see Theorem \ref{thm:linreg} for the full statement):
\begin{theorem}[Informal Statement]
Let $D_{X,Y}$ be a distribution over $\R^d\times \R$ where $Y=w^{\ast}\cdot X+e$ for some independent $e$ 
with mean $0$ and variance $1$. Assume furthermore, that $\E[XX^T]\preceq I$ and that $X$ has bounded fourth moments. 
Then there exists an algorithm that given $O(d^5/\eps^2)$ $\eps$-corrupted samples from $D$, 
computes a value $w\in \R^d$ so that with high probability $\|w-w^{\ast}\|_2 = O(\sqrt{\eps})$.
\end{theorem}

\subsection{Overview of \sever{} and its Analysis}
For simplicity of the exposition, we restrict ourselves
to the important special case where the functions involved are convex.
We have a probability distribution $p^{\ast}$ over convex functions
on some convex domain $\mathcal{H} \subseteq \R^d$
and we wish to minimize the function $\bar{f} = \mathbb{E}_{f \sim p^{\ast}}[f]$.
%i.e., find a point $w$ such that $\bar{f}(w)-\bar{f}(w^{\ast})$ is small, where $w^{\ast}$ is the true minimum.
This problem is well-understood in the absence of corruptions:
Under mild assumptions, if we take sufficiently many samples from $p^{\ast}$,
their average $\hat{f}$ approximates $\bar{f}$ pointwise with high probability. Hence,
we can use standard methods from convex optimization
to find an approximate minimizer for $\hat{f}$, which will in turn serve as an approximate
minimizer for $\bar{f}$.

In the robust setting, stochastic optimization becomes quite challenging:
Even for the most basic special cases of this problem (e.g., mean estimation, linear regression)
a {\em single} adversarially corrupted sample can substantially
change the location of the minimum for $\hat{f}$. Moreover, naive outlier removal
methods can only tolerate a negligible fraction $\eps$ of corruptions (corresponding to $\eps = O(d^{-1/2})$).

A first idea to get around this obstacle is the following: We consider the standard
(projected) gradient descent method used to find the minimum of $\hat{f}$.
This algorithm would proceed by repeatedly computing the gradient of $\hat{f}$
at appropriate points and using it to update the current location.
The issue is that adversarial corruptions can completely compromise
this algorithm's behavior, since they can substantially
change the gradient of $\hat{f}$ at the chosen points.
The key observation is that approximating the gradient of $\bar{f}$ at a given point,
given access to an $\eps$-corrupted set of samples,
can be viewed as a robust mean estimation problem.
We can thus use the robust mean estimation algorithm of~\cite{DKK+17},
which succeeds under fairly mild assumptions about the good samples.
Assuming that the covariance matrix of $\nabla f(w)$, $f \sim p^{\ast}$,
is bounded, we can thus ``simulate'' gradient descent and compute an
approximate minimum for $\bar{f}$.

In summary, the first algorithmic idea is to use a robust mean estimation routine as a
black-box in order to robustly estimate the gradient at {\em each} iteration of (projected) gradient descent.
This yields a simple robust method for stochastic optimization
with polynomial sample complexity and running time in a very general setting
(See Appendix~\ref{sec:general-algo} for details.)

We are now ready to describe \sever{} (Algorithm~\ref{alg:sever}) and the main insight behind it.
Roughly speaking, \sever{}  only calls our robust mean estimation routine 
(which is essentially the filtering method of~\cite{DKK+17} for outlier removal)
each time the algorithm reaches an approximate critical point of $\hat{f}$.
There are two main motivations for this approach:
First, we empirically observed that if we iteratively filter samples,
keeping the subset with the samples removed, then few iterations of the filter remove points.
Second, an iteration of the filter subroutine (Algorithm~\ref{alg:filter})
is more expensive than an iteration of gradient descent. Therefore, it is advantageous
to run many steps of gradient descent on the current set of corrupted samples
between consecutive filtering steps. This idea is further improved by
using stochastic gradient descent, rather than computing the average at each step.

\new{
An important feature of our analysis is that \sever{} does not use a robust mean estimation routine
as a black box. In contrast, we take advantage of the performance guarantees of
our filtering algorithm. The main idea for the analysis is as follows:
Suppose that we have reached an approximate critical point $w$ of $\hat{f}$ and
at this step we apply our filtering algorithm. By the performance guarantees of the latter algorithm
we are in one of two cases: either the filtering algorithm removes a set of corrupted functions
or it certifies that the gradient of $\hat{f}$ is ``close'' to the gradient of $\Ef$ at $w$.
In the first case, we make progress as we produce a ``cleaner'' set of functions.
In the second case, our certification implies that the point $w$ is also an approximate critical point of $\Ef$
and we are done.}

%It turns out that our second algorithm has similar provable guarantees as our first algorithm.
%Via a more careful analysis, we prove that we
%can use {\em the same set of samples} (drawn in the beginning of the algorithm)
%for all steps of the gradient descent. For this to be possible,
%we require that the set of samples that we have drawn from $p^{\ast}$
%satisfy reasonable bounds on their first and second
%moments at every point in our domain.

%!TEX root = ./main.tex 

\section{Experiments}
\label{sec:experiments}
In this section we apply \sever{} to regression and classification problems. 
Our code is available at~\url{https://github.com/hoonose/sever}.
As our base learners, we used ridge regression and an SVM, respectively. We 
implemented the latter as a quadratic program, using Gurobi~\citep{gurobi2016} 
as a backend solver and YALMIP~\citep{lofberg2004} as the modeling language.

In both cases, we ran the base learner and then extracted gradients for each data point 
at the learned parameters. We then centered the gradients and ran MATLAB's \texttt{svds} method 
to compute the top singular vector $v$, and removed the top $p$ fraction of points $i$ with the 
largest \emph{outlier score} $\tau_i$, 
computed as the squared magnitude of the projection onto $v$ 
(see Algorithm~\ref{alg:sever}). 
%We refer to the magnitude of this projection as the \emph{score}
%of the point. 
We repeated this for $r$ iterations in total. For classification, 
we centered the gradients separately (and removed points separately) for each class, 
which improved performance.

We compared our method to six baseline methods. All but one of these all have the same high-level form 
as \sever{} (run the base learner then filter top $p$ fraction of points with the largest score), 
but use a different definition of the score $\tau_i$ for deciding which points to filter:
\begin{itemize}
\item \noDef: no points are removed.
\item \xCen: remove points where the covariate $x$ has large $\ell_2$ distance from the mean.
\item \Loss: remove points with large loss (measured at the parameters output by the base learner).
\item \gUnc: remove points with large gradient (in $\ell_2$-norm).
\item \gCen: remove points whose gradients are far from the mean gradient in $\ell_2$-norm.
\item \ransac: repeatedly subsample points uniformly at random, and find the best fit with the subsample. Then, choose the best fit amongst this set of learners. Note that this method is not ``filter-based''.\footnote{In practice, heuristics must often be applied to choose the best fit. In our experiments, we ``cheat'' slightly by in fact choosing the best fit post-hoc by reporting the best error achieved by any learner in this way. Despite strengthening \ransac{} in this way, we observe that it still has poor performance.}
\end{itemize}
Note that \gCen{} is similar to our method, except that it removes large gradients in terms of 
$\ell_2$-norm, rather than in terms of projection onto the top singular vector. 
As before, for classification we compute these metrics separately for each class.

Both ridge regression and SVM have a single hyperparameter (the regularization coefficient). 
We optimized this based on the uncorrupted data and then kept it fixed throughout our 
experiments. In addition, since the data do not already have outliers, we added varying 
amounts of outliers (ranging from $0.5\%$ to $10\%$ of the clean data); this process is
described in more detail below.

\subsection{Ridge Regression}
For ridge regression, we tested our method on a synthetic Gaussian dataset as well as a 
drug discovery dataset.
The synthetic dataset consists of observations $(x_i, y_i)$ where 
$x_i \in \mathbb{R}^{500}$ has independent standard Gaussian entries, 
and $y_i = \langle x_i,  w^* \rangle + 0.1 z_i$, where $z_i$ is also Gaussian. 
We generated $5000$ training points and $100$ test points.
The drug discovery dataset was obtained from the ChEMBL database and was originally curated by 
\citet{olier2018qsar}; it consists of $4084$ data points in $410$ dimensions; we split 
this into a training set of $3084$ points and a test set of $1000$ points.

\paragraph{Centering}
% \js{I thought we were getting rid of this}
We found that centering the data points decreased error noticeably on the drug discovery 
dataset, while scaling each coordinate to have variance $1$ decreased error by a small 
amount on the synthetic data.
To center in the presence of outliers, we used the robust mean estimation algorithm 
from \citep{DKK+17}. 
% To rescale coordinates, we used the median absolute deviation of each 
% coordinate as a robust estimate of scale.

\paragraph{Adding outliers.}
We devised a method of generating outliers that fools all of the baselines while still 
inducing high test error. At a high level, the outliers cause ridge regression to output 
$w = 0$ (so the model always predicts $y = 0$).

If $(X, y)$ are the true data points and responses, this can be achieved by 
setting each outlier point $(X_{\mathrm{bad}}, y_{\mathrm{bad}})$ as
\[ X_{\mathrm{bad}} = \frac{1}{\alpha \cdot n_{\mathrm{bad}}}y^\top X ~~\mbox{and}~~~y_{\mathrm{bad}} = -\beta \; , \] where $n_{\mathrm{bad}}$ is the number of outliers we add, and $\alpha$ and $\beta$ are hyperparameters.

If $\alpha = \beta$, one can check that $w = 0$ is the 
unique minimizer for ridge regression on the perturbed dataset. 
By tuning $\alpha$ and $\beta$, we can then obtain attacks that fool all the baselines while 
damaging the model (we tune $\alpha$ and $\beta$ separately to give an additional degree of 
freedom to the attack).
To increase the error, we also found it useful to perturb each individual 
$X_{\mathrm{bad}}$ by a small amount of Gaussian noise. %\new{save line}
%Then it is easily verified that, so long as $\alpha = \beta$, the gradient of the ridge regression function on the poisoned data set is zero at $w = 0$, and hence $w = 0$ is the unique minimizer of the ridge regression loss on this poisoned data set.
%Then, by playing with $\alpha$ and $\beta$, so long as we can ensure that our points have comparatively better scores at $w = 0$ than the real data points, we know that they will at least not be filtered out in the first iteration of the defense.

In our experiments we found that this method generated successful attacks as long as 
the fraction of outliers was at least roughly $2\%$ for synthetic data, 
and roughly $5\%$ for the drug discovery data.
%Moreover, beyond these thresholds, our experiments show that iterating the baselines do not help detect these points.

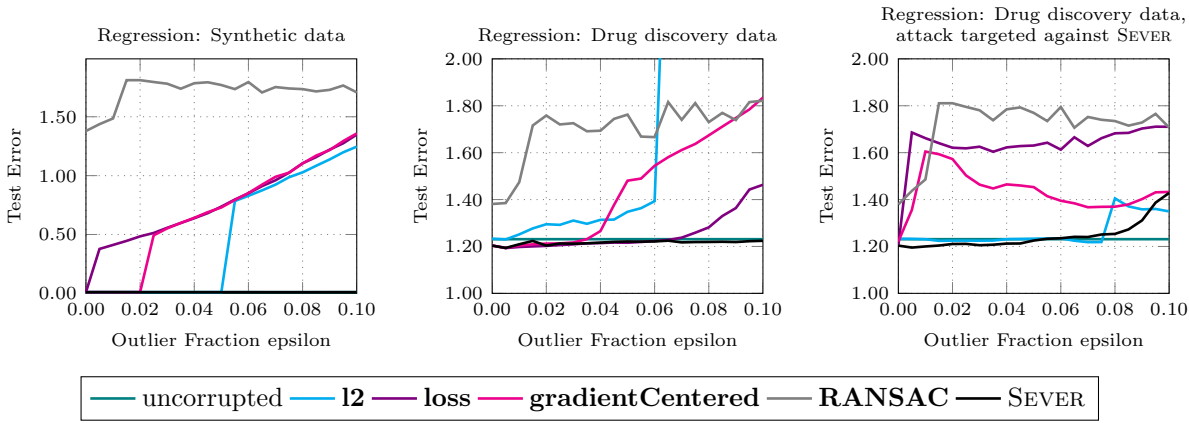
\begin{figure}[h!]
\centering
\begin{tikzpicture}
\begin{axis}[errplottriple,name=linreg_synth, align=center, title={Regression: Synthetic data}, xticklabel style={/pgf/number format/.cd, fixed, fixed zerofill, precision=2,/tikz/.cd}, legend columns= 5, legend style={anchor=north west, xshift=-0.6 \plotwidth, yshift=-0.8\plotheight}, legend columns=6]
\addplot[teal] table[x=eps, y=err] {figures/linreg_synth/uncorrupted.txt};
\addplot[cyan] table[x=eps, y=err] {figures/linreg_synth/l2.txt};
\addplot[violet] table[x=eps, y=err] {figures/linreg_synth/loss.txt};
\addplot[magenta] table[x=eps, y=err] {figures/linreg_synth/gradient.txt};
\addplot[gray] table[x=eps, y=err] {figures/linreg_synth/ransac.txt};
\addplot[black] table[x=eps, y=err] {figures/linreg_synth/sever.txt};
\legend{uncorrupted, \xCen{}, \Loss{}, \gCen{},\ransac{}, \sever{}}
\end{axis}

\begin{axis}[errplottriple,name=linreg_qsar, align=center, at=(linreg_synth.north east),anchor=north west, xticklabel style={/pgf/number format/.cd, fixed, fixed zerofill, precision=2,/tikz/.cd}, xshift=\plotxspacing,ignore legend, title={Regression: Drug discovery data}, ymin = 1, ymax = 2]
\addplot[teal] table[x=eps, y=err] {figures/linreg_qsar/uncorrupted.txt};
\addplot[cyan] table[x=eps, y=err] {figures/linreg_qsar/l2.txt};
\addplot[violet] table[x=eps, y=err] {figures/linreg_qsar/loss.txt};
\addplot[magenta] table[x=eps, y=err] {figures/linreg_qsar/gradient.txt};
\addplot[gray] table[x=eps, y=err] {figures/linreg_qsar/ransac.txt};
\addplot[black] table[x=eps, y=err] {figures/linreg_qsar/sever.txt};
\end{axis}

\begin{axis}[errplottriple,name=linreg_qsar_worst, align=center, at=(linreg_qsar.north east),anchor=north west, xticklabel style={/pgf/number format/.cd, fixed, fixed zerofill, precision=2,/tikz/.cd}, xshift=\plotxspacing,ignore legend, title={Regression: Drug discovery data, \\ attack targeted against \sever{}}, ymin = 1, ymax = 2]
\addplot[teal] table[x=eps, y=err] {figures/linreg_qsar_worst/uncorrupted.txt};
\addplot[cyan] table[x=eps, y=err] {figures/linreg_qsar_worst/l2.txt};
\addplot[violet] table[x=eps, y=err] {figures/linreg_qsar_worst/loss.txt};
\addplot[magenta] table[x=eps, y=err] {figures/linreg_qsar_worst/gradient.txt};
\addplot[gray] table[x=eps, y=err] {figures/linreg_qsar_worst/ransac.txt};
\addplot[black] table[x=eps, y=err] {figures/linreg_qsar_worst/sever.txt};
\end{axis}

\end{tikzpicture}
\caption{$\epsilon$ vs test error for baselines and \sever{} on synthetic data and the drug discovery dataset. The left and middle figures show that \sever{} continues to maintain statistical accuracy against our attacks which are able to defeat previous baselines. The right figure shows an attack with parameters chosen to increase the test error \sever{} on the drug discovery dataset as much as possible. Despite this, \sever{} still has relatively small test error.} 
\label{label:acc-vs-eps-linreg} 
\end{figure}

\begin{figure*}[h!]
\includegraphics[width=0.33\textwidth]{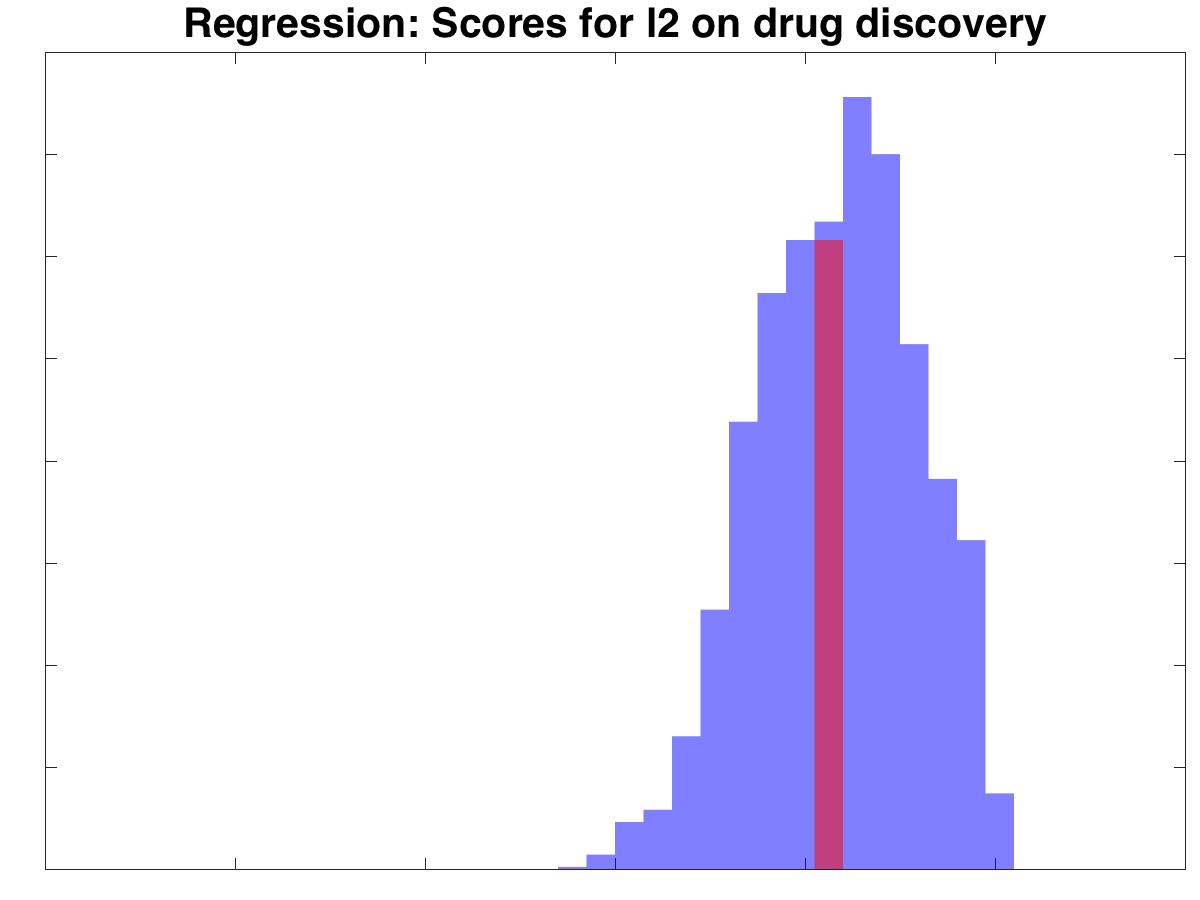}
\includegraphics[width=0.33\textwidth]{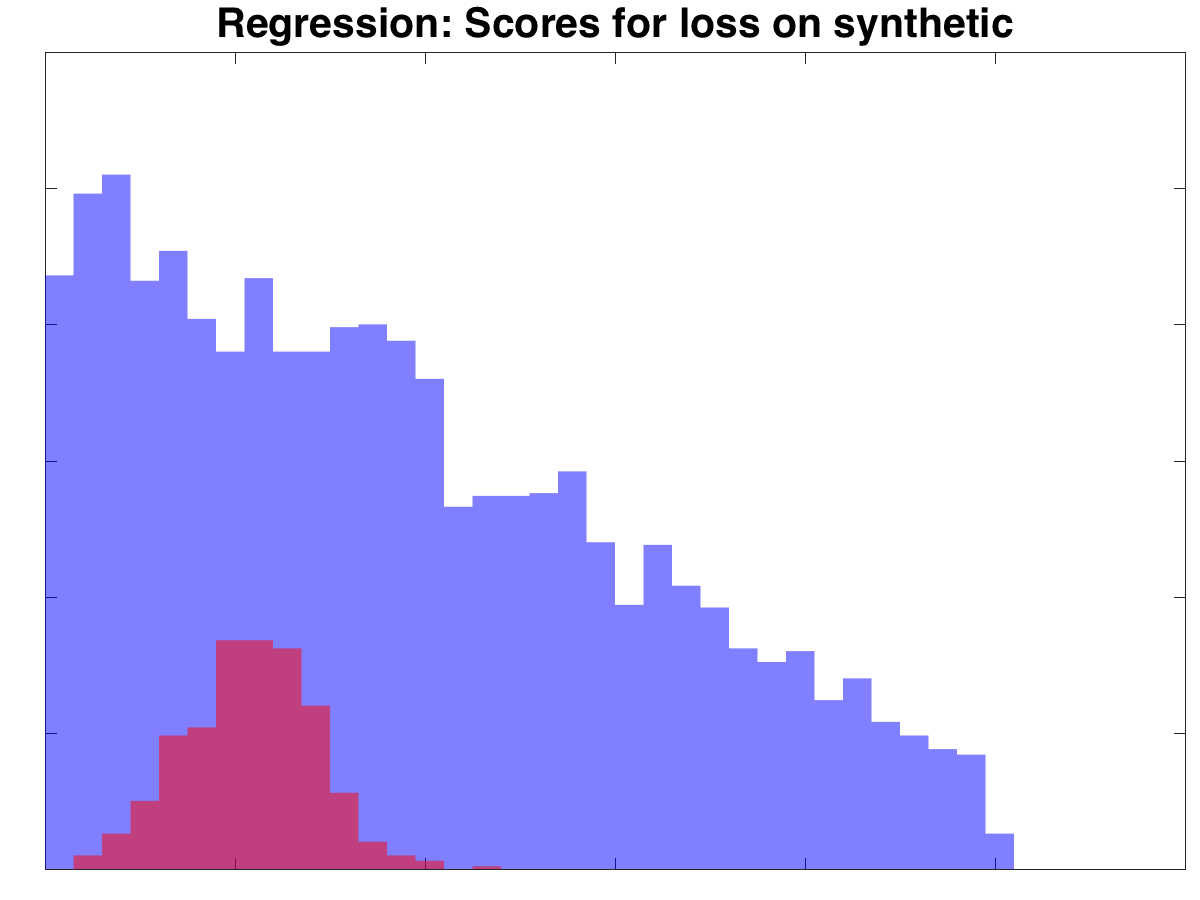}
\includegraphics[width=0.33\textwidth]{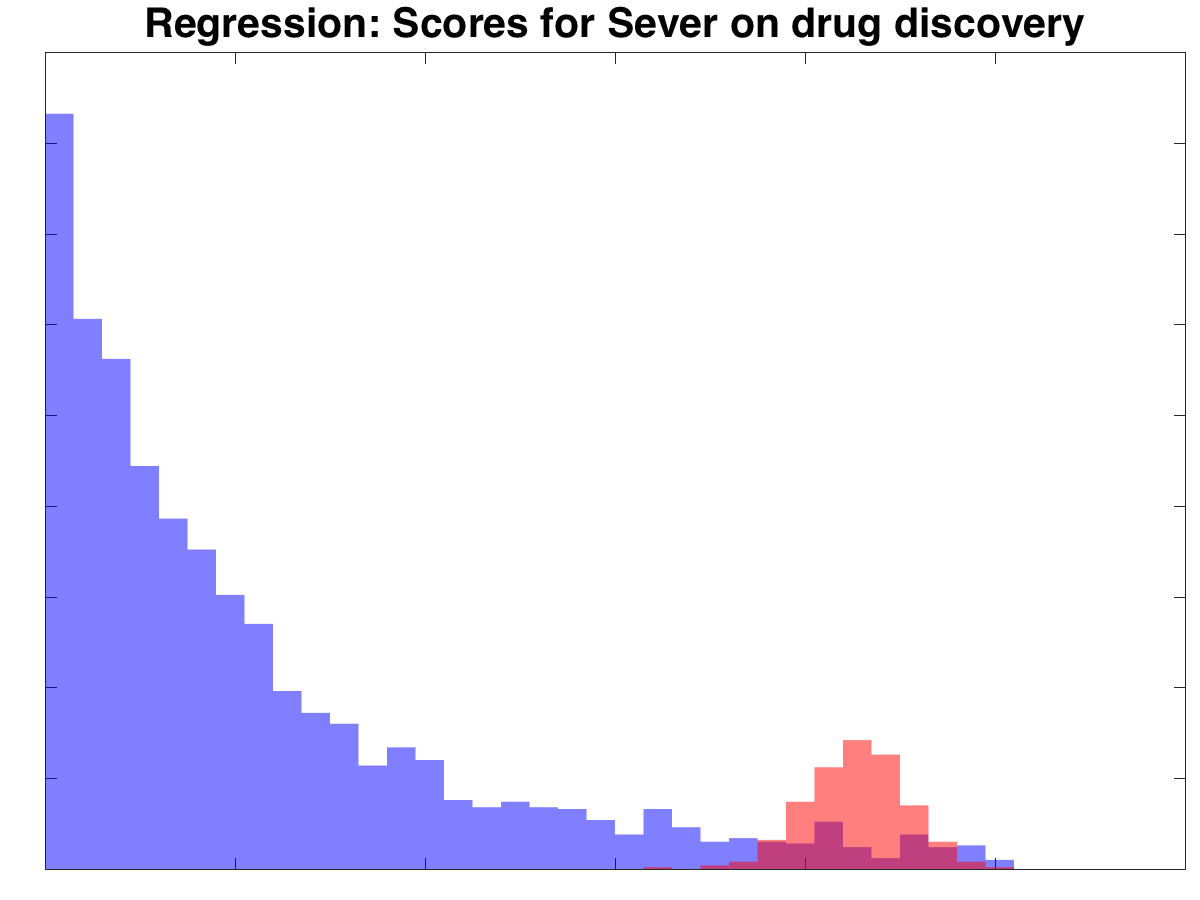}
\caption{A representative set of histograms of scores for baselines and \sever{} on synthetic data and a drug discovery dataset. From left to right: scores for the \xCen{} defense on the drug discovery dataset, scores for \Loss{} on synthetic data, and scores for \sever{} on the drug discovery dataset, all with the addition of 10\% outliers.
The scores for the true dataset are in blue, and the scores for the outliers are in red.
For the baselines, the scores for the outliers are inside the bulk of the distribution and thus hard to detect, whereas the scores for the outliers assigned by \sever{} are clearly within the tail of the distribution and easily detectable. } 
\label{label:hist-linreg} 
\end{figure*}

\paragraph{Results.}
In Figure \ref{label:acc-vs-eps-linreg} we compare the test error of our defense against the baselines 
as we increase the fraction $\epsilon$ of added outliers.
To avoid cluttering the figure, we only show the performance of \xCen, \Loss, \gCen, \ransac, 
and \sever{}; the performance of the remaining baselines is qualitatively similar to the 
baselines in Figure \ref{label:acc-vs-eps-linreg}.

For both the baselines and our algorithms, we iterate the defense $r=4$ times, 
each time removing the $p=\epsilon/2$ fraction of points with largest score.
For consistency of results, for each defense and each value of $\epsilon$ we ran the defense 
3 times on fresh attack points and display the median of the 3 test errors.

%We remark that the choice of these parameters was somewhat arbitrary, and that similar parameters had similar results.
When the attack parameters $\alpha$ and $\beta$ are tuned to defeat the baselines 
(Figure~\ref{label:acc-vs-eps-linreg} left and center), 
our defense substantially outperforms the baselines as soon as we cross $\epsilon \approx 1.5\%$ 
for synthetic data, and $\epsilon \approx 5.5\%$ for the drug discovery data.
In fact, most of the baselines do worse than not removing any outliers at all
(this is because they end up mostly removing good data points, which causes the 
outliers to have a larger effect).
Even when $\alpha$ and $\beta$ are instead tuned to defeat \sever{}, its resulting error remains 
small (Figure~\ref{label:acc-vs-eps-linreg} right).
% on QSAR, as shown in the rightmost figure, 
%it seems unable to simultaneously fool the previous baselines.
%Moreover, the test error for \sever{} that it induces is still relatively small.
%We remark that on synthetic data we were unable to qualitatively change the error of \sever{} 
%when tuning the attack to defeat it.

%To compare the performance of our baselines at a finer level, figures ?? and ?? are zoomed in versions of Figures ?? and ?? respectively. 
%As the figures show, the simple filter consistently performs comparably or better than the MAD filter on both tasks.
%Interestingly, at times the simple filter also outperforms the uncorrupted benchmark for both tasks, though we speculate this is mostly due to noise.

To understand why the baselines fail to detect the outliers,
in Figure \ref{label:hist-linreg} we show a representative sample of the histograms of scores of the uncorrupted points overlaid 
with the scores of the outliers, for both synthetic data and the drug discovery dataset with $\epsilon = 0.1$, 
after one run of the base learner.
The scores of the outliers lie well within the distribution of scores of the uncorrupted points.
Thus, it would be impossible for the baselines to remove them 
without also removing a large fraction of uncorrupted points.
%In contrast, the scores generated by our defense show a clear delineation between the uncorrupted points and the poisoned data.
%This in turn allows us to reliably detect and remove the poisoned data.

Interestingly, for small $\epsilon$ all of the methods improve upon the uncorrupted test 
error for the drug discovery data; this appears to be due to the presence of a small number of
natural outliers in the data that all of the methods successfully remove.

%!TEX root = ./main.tex
\subsection{Support Vector Machines}
We next describe our experimental results for SVMs; we tested our method on 
a synthetic Gaussian dataset as well as a spam classification task.
Similarly to before, the synthetic data consists of observations $(x_i, y_i)$, where $x_i \in \mathbb{R}^{500}$ has independent standard Gaussian entries,
and $y_i = \operatorname{sign}(\langle x_i, w^* \rangle + 0.1z_i)$, where $z_i$ is also Gaussian 
and $w^*$ is the true parameters (drawn at random from the unit sphere).
The spam dataset comes from the Enron corpus~\citet{metsis2006spam}, and 
consists of $4137$ training points and $1035$ test points in $5116$ dimensions.
To generate attacks, we used the data poisoning algorithm presented in 
\citet{koh2018stronger}.

%To generate attacks, we used the data poisoning algorithm presented in 
%\citet{steinhardt2017certified}; the authors provided us with an improved version of 
%their algorithm that can circumvent the \xCen{} and \Loss{} baselines and partially circumvents 
%the gradient baselines as well. 
%By performing a sweep across hyperparameters, we generated 
%a number of attacks ($X$ for synthetic and $48$ for Enron) in order to see cases where 
%each method does relatively well or poorly.

In contrast to ridge regression, we did not perform centering and rescaling for these 
datasets as it did not seem to have a large effect on results.

In all experiments for this section, each method removed the top $p=\frac{n_- + n_+}{\min\{n_+,n_-\}} \cdot \frac{\epsilon}{r}$ of highest-scoring 
points for each of $r = 2$ iterations, where $n_+$ and $n_-$ are the number of positive and negative training points respectively.
This expression for $p$ is chosen in order to account for class imbalance, which is extreme in the case of the Enron dataset -- if the attacker plants all the outliers in the smaller class, then a smaller value of $p$ would remove too few points, even with a perfect detection method.
%The exception is for the Enron dataset, where due to 
%class imbalance this would result in removing too few points for even a perfect detection 
%method to remove all the outliers. We thus took $p = \max(0.025, 1.75\epsilon)$ to 
%account for this. 

\begin{figure}[h!]
\centering
\begin{tikzpicture}
\begin{axis}[errplot,name=svm_synth_loss, legend style={anchor=north west, xshift=-0.35 \plotwidth, yshift=-1.2\plotheight}, legend columns=4, title={SVM: Strongest attacks against \Loss{} on synthetic data}]
\addplot[teal, mark=|] table[x=eps, y=err] {figures/svm_synth_loss/uncorrupted.txt};
\addplot[violet, mark=|] table[x=eps, y=err] {figures/svm_synth_loss/loss.txt};
\addplot[black, mark=|] table[x=eps, y=err] {figures/svm_synth_loss/sever.txt};
\legend{uncorrupted, \Loss{}, \sever{}}
\end{axis}

\begin{axis}[errplot,name=svm_synth_sever, at=(svm_synth_loss.north east),anchor=north west, xshift=\plotxspacing,ignore legend, title={SVM: Strongest attacks against \sever{} on synthetic data}]
\addplot[teal, mark=|] table[x=eps, y=err] {figures/svm_synth_sever/uncorrupted.txt};
\addplot[violet, mark=|] table[x=eps, y=err] {figures/svm_synth_sever/loss.txt};
\addplot[black, mark=|] table[x=eps, y=err] {figures/svm_synth_sever/sever.txt};
\end{axis}

\end{tikzpicture}
\caption{$\epsilon$ versus test error for \Loss{} baseline and \sever{} on synthetic data. The left figure demonstrates that \sever{} is accurate when outliers manage to defeat \Loss{}.
The right figure shows the result of attacks which increased the test error the most against \sever{}. Even in this case, \sever{} performs much better than the baselines.}
\label{fig:svm-synthetic}
\end{figure}
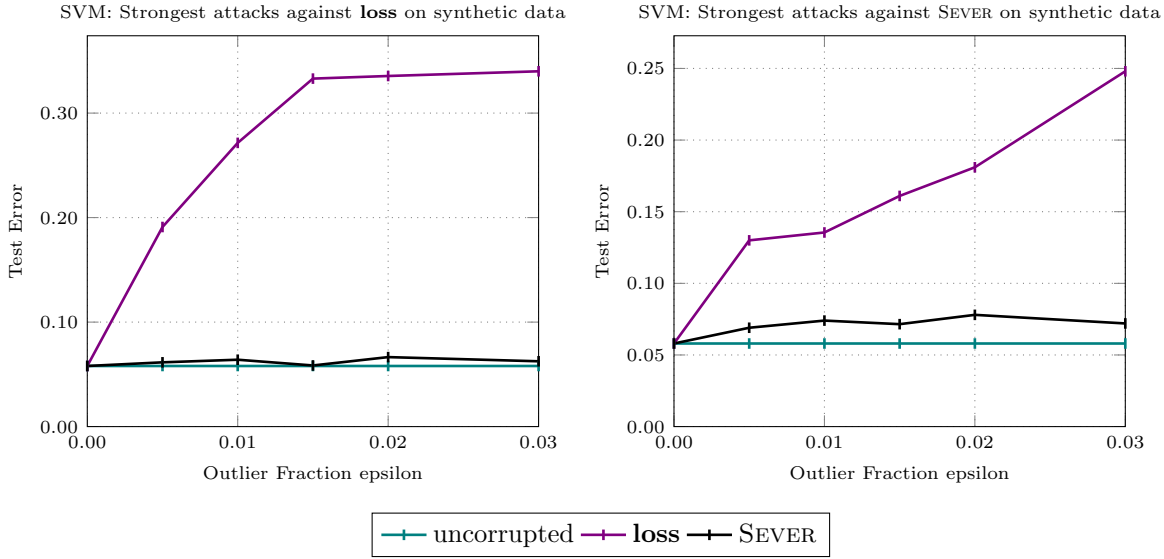

\paragraph{Synthetic results.}
We considered fractions of outliers ranging from $\epsilon = 0.005$ to $\epsilon = 0.03$. 
By performing a sweep across hyperparameters of the attack, we generated 
$56$ distinct sets of attacks for each value of $\epsilon$. 
In Figure~\ref{fig:svm-synthetic}, we show results for the attack where the \Loss{} baselines does the 
worst, as well as for the attack where our method does the worst. 
When attacks are most effective against \Loss{}, \sever{} substantially outperforms it, nearly matching the test accuracy of $5.8\%$ on the uncorrupted data, while \Loss{} performs worse than $30\%$ error at just a $1.5\%$ fraction of injected outliers.
Even when attacks are most effective against \sever{}, it still outperforms \Loss{}, achieving a test error of at most $9.05\%$.
We note that other baselines behaved qualitatively similarly to \Loss{}, and the results are displayed in Section~\ref{sec:experiments-app}.

\begin{figure}[h!]
\centering
\begin{tikzpicture}

\begin{axis}[errplottriple,name=svm_enron_gradientCentered, align=center, title={SVM: Strongest attacks against \\ \gCen{} on Enron}, legend columns= 3, legend style={anchor=north west, xshift=-0.15 \plotwidth, yshift=-0.8\plotheight}]
\addplot[teal, mark=|] table[x=eps, y=err] {figures/svm_enron_gradientCentered/uncorrupted.txt};
\addplot[violet, mark=|] table[x=eps, y=err] {figures/svm_enron_gradientCentered/loss.txt};
\addplot[magenta, mark=|] table[x=eps, y=err] {figures/svm_enron_gradientCentered/gradient.txt};
\addplot[orange, mark=|] table[x=eps, y=err] {figures/svm_enron_gradientCentered/gradientCentered.txt};
\addplot[black, mark=|] table[x=eps, y=err] {figures/svm_enron_gradientCentered/sever.txt};
\legend{uncorrupted, \Loss{}, \gUnc{}, \gCen{}, \sever{}}
\end{axis}

\begin{axis}[errplottriple,name=svm_enron_loss, align=center, at=(svm_enron_gradientCentered.north east),anchor=north west, xshift=\plotxspacing,ignore legend, title={SVM: Strongest attacks \\ against \Loss{} on Enron}]
\addplot[teal, mark=|] table[x=eps, y=err] {figures/svm_enron_loss/uncorrupted.txt};
\addplot[violet, mark=|] table[x=eps, y=err] {figures/svm_enron_loss/loss.txt};
\addplot[magenta, mark=|] table[x=eps, y=err] {figures/svm_enron_loss/gradient.txt};
\addplot[orange, mark=|] table[x=eps, y=err] {figures/svm_enron_loss/gradientCentered.txt};
\addplot[black, mark=|] table[x=eps, y=err] {figures/svm_enron_loss/sever.txt};
\end{axis}
%\node [at=(svm_enron_loss.north east),anchor=south west,xshift=-6.5cm,yshift=.2cm] {SVM: Strongest attacks against loss on Enron};

\begin{axis}[errplottriple,name=svm_enron_sever, align=center, at=(svm_enron_loss.north east),anchor=north west, xshift=\plotxspacing,ignore legend, title={SVM: Strongest attacks \\ against \sever{} on Enron}]
\addplot[teal, mark=|] table[x=eps, y=err] {figures/svm_enron_sever/uncorrupted.txt};
\addplot[violet, mark=|] table[x=eps, y=err] {figures/svm_enron_sever/loss.txt};
\addplot[magenta, mark=|] table[x=eps, y=err] {figures/svm_enron_sever/gradient.txt};
\addplot[orange, mark=|] table[x=eps, y=err] {figures/svm_enron_sever/gradientCentered.txt};
\addplot[black, mark=|] table[x=eps, y=err] {figures/svm_enron_sever/sever.txt};
\end{axis}
%\node [at=(svm_enron_sever.north east),anchor=south west,xshift=-6.5cm,yshift=.2cm] {SVM: Strongest attacks against \sever{} on Enron};

\end{tikzpicture}
\caption{$\epsilon$ versus test error for baselines and \sever{} on the Enron spam corpus. 
The left and middle figures are the attacks which perform best against two baselines, while the right figure performs best against \sever{}. 
Though other baselines may perform well in certain cases, only \sever{} is consistently accurate. 
The exception is for certain attacks at $\epsilon = 0.03$, which, as shown in Figure~\ref{fig:spam-histogram}, require three rounds of outlier removal for any method to obtain reasonable test error -- in these plots, our defenses perform only two rounds.}
\label{fig:spam-results}
\end{figure}
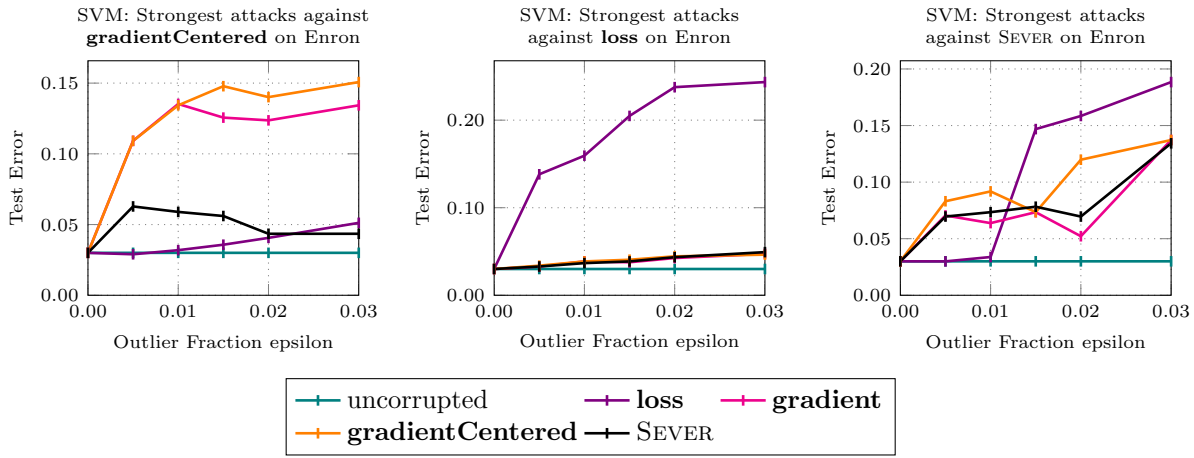

\begin{figure*}[h!]
\includegraphics[width=0.33\textwidth]{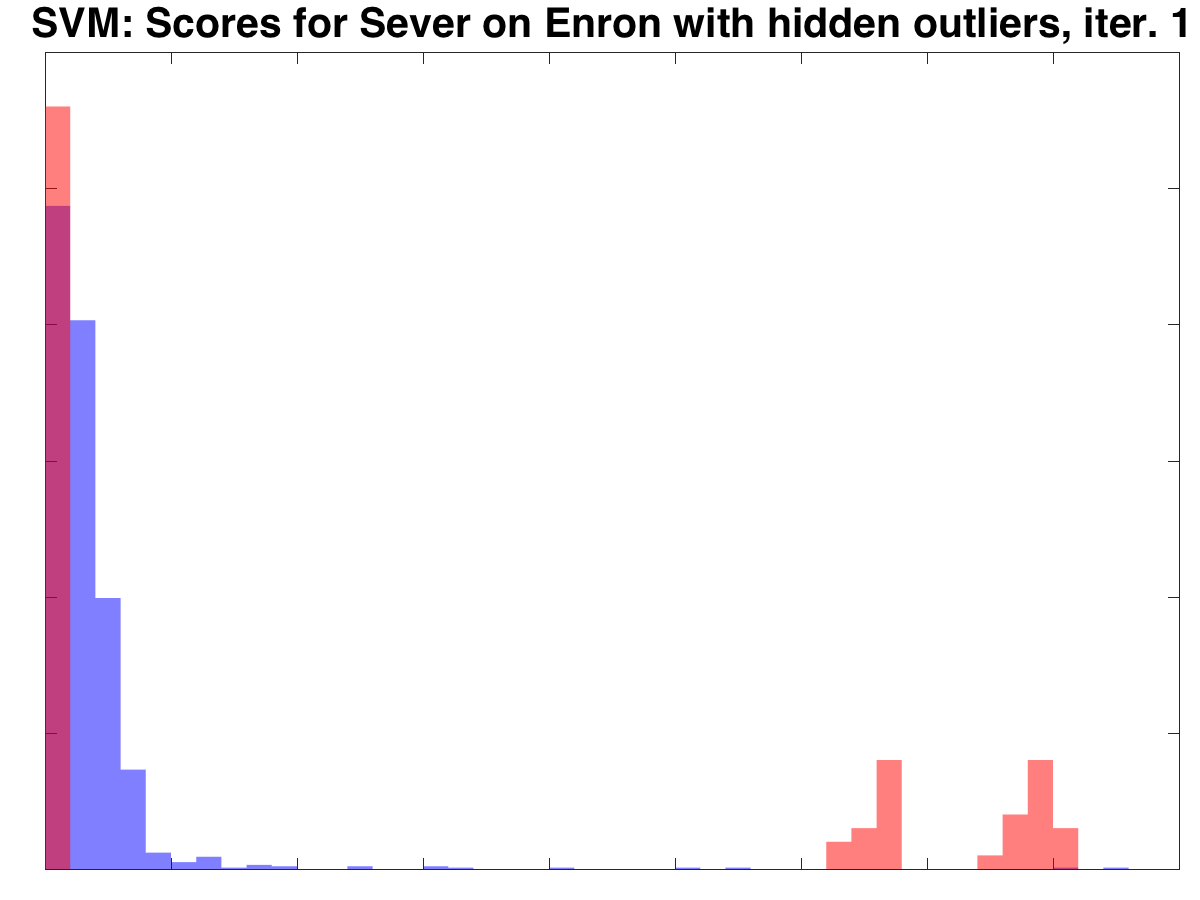} 
\includegraphics[width=0.33\textwidth]{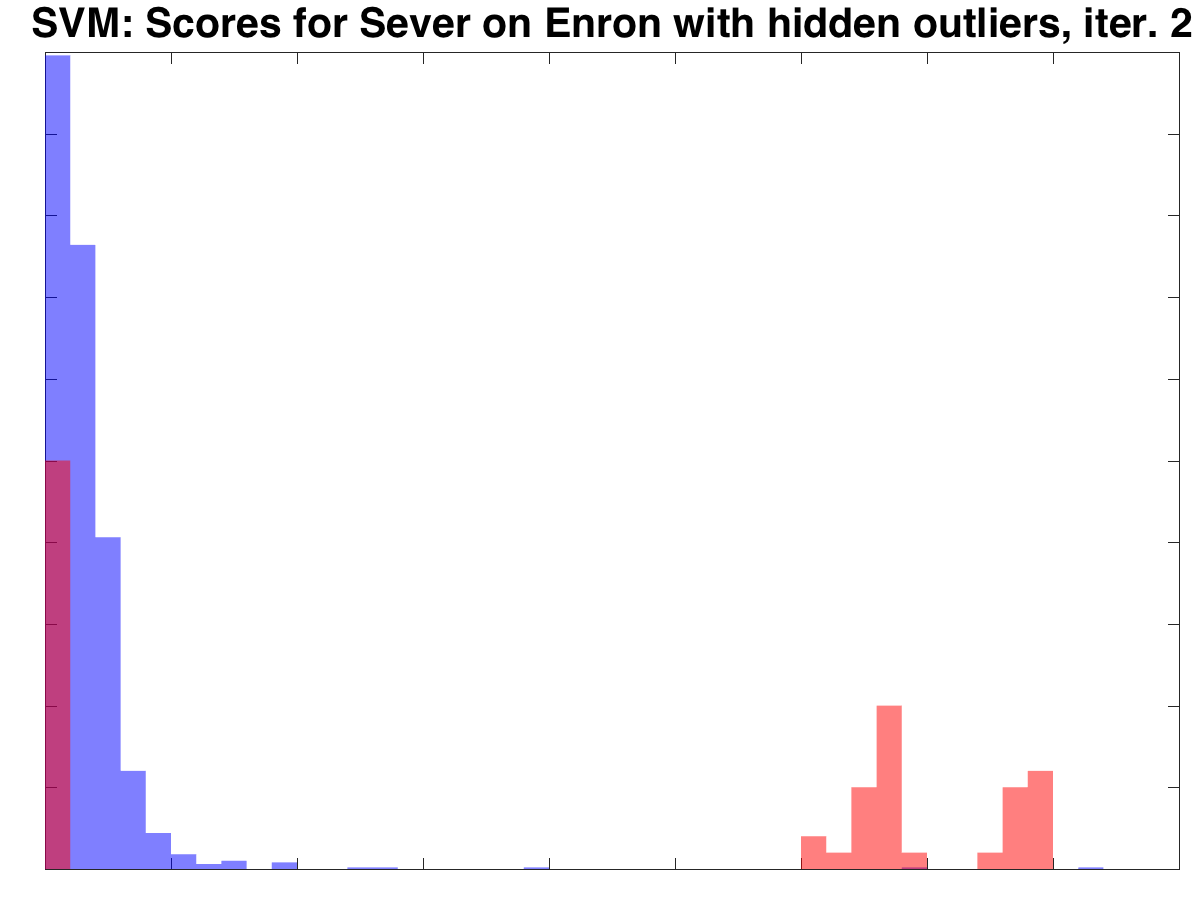} 
\includegraphics[width=0.33\textwidth]{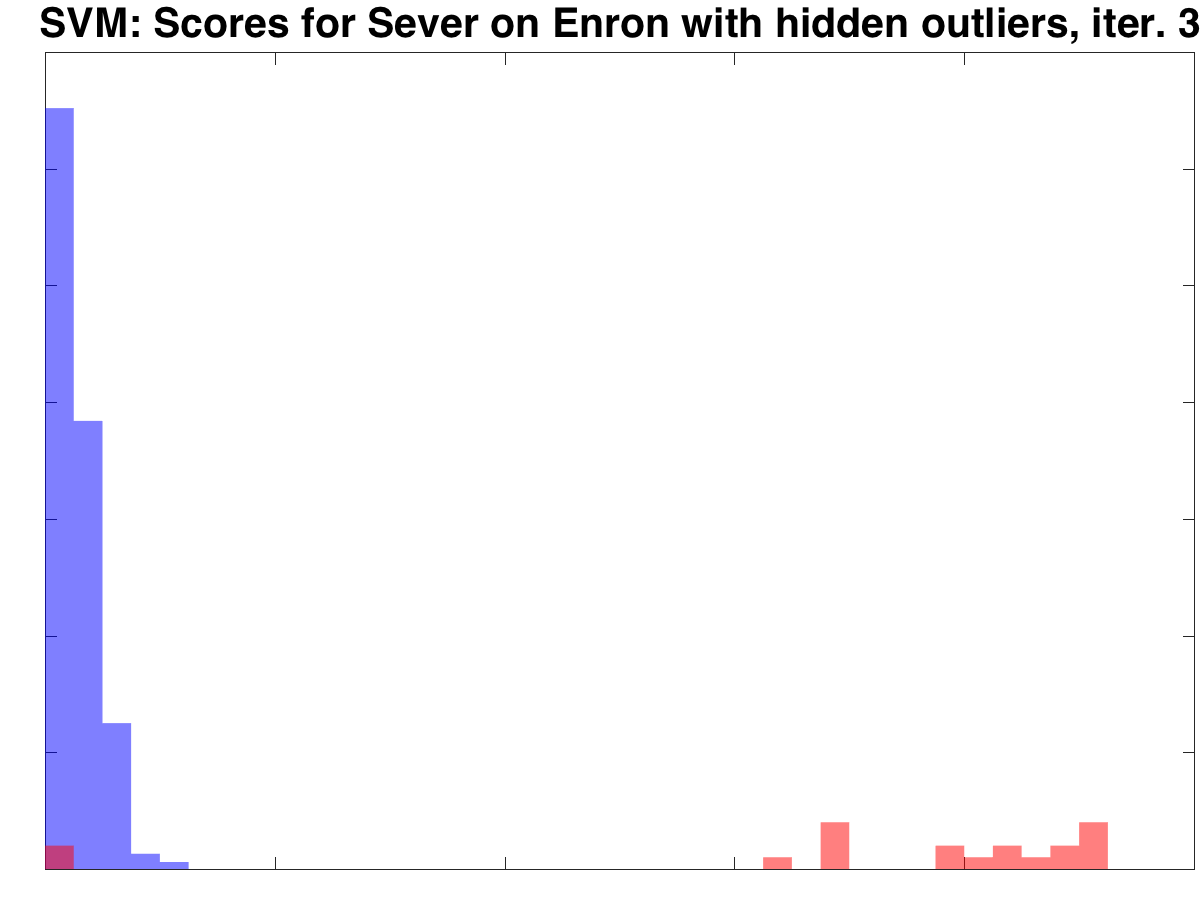} 
\caption{An illustration of why multiple rounds of filtering are necessary. Histograms of scores assigned by \sever{} in three subsequent iterations of outlier removal. Inliers are blue, and outliers are red (scaled up by a factor of 10). In early iterations, a significant fraction of outliers may be ``hidden'' (i.e. have 0 loss) by being correctly classified in one iteration. However, once previous outliers are removed, these points may become incorrectly classified, thus significantly degrading the quality of our solution but simultaneously becoming evident to \sever{}.}
\label{fig:spam-histogram}
\end{figure*}

\paragraph{Spam results.}
For results on Enron, we used the same values of $\epsilon$, and considered $96$ distinct 
hyperparameters for the attack. There was not a single attack that simultaneously defeated 
all of the baselines, so in Figure~\ref{fig:spam-results} we show two attacks that do well 
against different sets of baselines, as well as the attack that performs best against our method. 

At $\epsilon = 0.01$, the worst performance of our method against all attacks was $7.34\%$, 
in contrast to $13.43\%-20.48\%$ for the baselines (note that the accuracy is $3\%$ in the absence of outliers). However, at $\epsilon = 0.03$, while we still outperform the baselines, our error is 
relatively large---$13.53\%$.

To investigate this further, we looked at all $48$ attacks and found that while on 
$42$ out of $48$ attacks our error never exceeded $7\%$, on $6$ of the attacks (including 
the attack in Figure~\ref{fig:spam-results}) the error was substantially higher.
Figure~\ref{fig:spam-histogram} shows what is happening.
The leftmost figure displays the scores assigned by \sever{} after the first iteration, where red bars indicate outliers.
While some outliers are assigned extremely large scores and thus detected, several outliers are correctly classified and thus have 0 gradient.
However, once we remove the first set of outliers, some outliers which were previously correctly classified now have large score, as displayed in the middle figure.
Another iteration of this process produces the rightmost figure, where almost all the remaining outliers have large score and will thus be removed by \sever{}.
This demonstrates that some outliers may be hidden until other outliers are removed, necessitating multiple iterations.

Motivated by this, we re-ran our method against the $6$ attacks using $r = 3$ iterations instead of 
$2$ (and decreasing $p$ as per the expression above). After this 
change, all $6$ of the attacks had error at most $7.4\%$.

%With some attacks, we discovered that the pre-filtering solution to the SVM problem (including the poisoned data) would correctly classify certain poisoned datapoints.
%As such, these points would have 0 loss and gradient, and no reasonable filtering method would detect them as outliers.
%However, after a single iteration of filtering, these points would now be incorrectly classified and greatly decrease the quality of solution.
%This indicates that multiple rounds of filtering are necessary to achieve satisfactory test error.
%For our experiments, we allowed all methods to perform two iterations of filtering.

\section{Discussion}

In this paper we have presented an algorithm, \sever{}, that has both strong 
theoretical robustness properties in the presence of outliers, and performs well 
on real datasets. \sever{} is based on the idea that learning can often be cast as the 
problem of finding an approximate stationary point of the loss, which can in turn 
be cast as a robust mean estimation problem, allowing us to leverage existing 
techniques for efficient robust mean estimation.

There are a number of directions along which \sever{} could be improved: first, 
it could be extended to handle more general assumptions on the data; second, 
it could be strengthened to achieve better error bounds in terms of the fraction 
of outliers; finally, one could imagine \emph{automatically learning} a feature representation 
in which \sever{} performs well. We discuss each of these ideas in detail below.

\paragraph{More general assumptions.} 
The main underlying assumption on which \sever{} rests is that the top singular value 
of the gradients of the data is small. While this appeared to hold true on the datasets 
we considered, a common occurence in practice is for there to be \emph{a few} large singular 
values, together with \emph{many} small singular values. It would therefore be 
desirable to design a version of \sever{} that can take advantage of such phenomena. In addition, 
it would be worthwhile to do a more detailed empirical analysis across a wide variety of 
datasets investigating properties 
that can enable robust estimation (the notion of \emph{resilience} in~\cite{steinhardt2018resilience} could provide a template for finding 
such properties).

\paragraph{Stronger robustness to outliers.}
In theory, \sever{} has a $\oo(\sqrt{\epsilon})$ dependence in error 
on the fraction $\epsilon$ of outliers (see Theorem~\ref{thm:stationary-point-inf}). 
While without stronger assumptions this is 
likely not possible to improve, 
in practice we would prefer to have a dependence closer to $\oo(\epsilon)$. 
Therefore, it would also be useful to improve \sever{} to have such an $\oo(\epsilon)$-dependence 
under stronger but realistic assumptions. Unfortunately, all existing algorithms for robust mean 
estimation that achieve error better than $\oo(\sqrt{\epsilon})$ either rely on strong 
distributional assumptions such as Gaussianity \citep{DKKLMS16,LaiRV16}, 
or else require expensive computation involving e.g.~sum-of-squares optimization 
\citep{HL17,KStein17,KS17}. Improving the robustness of \sever{} thus requires improvements 
on the robust mean estimation algorithm that \sever{} uses as a primitive.

\paragraph{Learning a favorable representation.}
Finally, we note that \sever{} performs best when the features have small covariance and 
strong predictive power. One situation in particular where this holds is when there are 
many approximately independent features that are predictive of the true signal. 

It would be interesting to try to 
learn a representation with such a property. This could be done, for instance, by training a neural network with some cost function that encourages independent features (some ideas along these general lines are discussed in \citet{bengio2017consciousness}). An issue is how to learn such a representation robustly; one idea is learn a representation on a dataset that is known to be free of outliers, and hope that the representation is useful on other datasets in the same application domain.\\
\\
Beyond these specific questions, we view the general investigation of robust methods (both empirically and theoretically) as an important step as machine learning moves forwards. 
Indeed, as machine learning is applied in increasingly many situations and in 
increasingly automated ways, it is important to attend to robustness considerations 
so that machine learning systems behave reliably and 
avoid costly errors. While the bulk of recent work has highlighted the vulnerabilities 
of machine learning (e.g.~\citep{szegedy2014intriguing,li2016data,steinhardt2017certified,evtimov2017robust,chen2017targeted}), we are optimistic that practical algorithms backed by principled 
theory can finally patch these vulnerabilities and lead to truly reliable systems.

\bibliographystyle{alpha}
\bibliography{allrefs,refdb/all}

\newcommand{\etalchar}[1]{$^{#1}$}
\begin{thebibliography}{{{Gur}}16}

\bibitem[ABL14]{awasthi2014power}
P.~Awasthi, M.~F. Balcan, and P.~M. Long.
\newblock The power of localization for efficiently learning linear separators
  with noise.
\newblock In {\em Symposium on Theory of Computing (STOC)}, pages 449--458,
  2014.

\bibitem[BDLS17]{BDLS17}
S.~Balakrishnan, S.~S. Du, J.~Li, and A.~Singh.
\newblock Computationally efficient robust sparse estimation in high
  dimensions.
\newblock In {\em Proceedings of the 30th Conference on Learning Theory, {COLT}
  2017}, pages 169--212, 2017.

\bibitem[Ben17]{bengio2017consciousness}
Y.~Bengio.
\newblock The consciousness prior.
\newblock {\em arXiv preprint arXiv:1709.08568}, 2017.

\bibitem[BJK15]{BhatiaJK15}
K.~Bhatia, P.~Jain, and P.~Kar.
\newblock Robust regression via hard thresholding.
\newblock In {\em Advances in Neural Information Processing Systems 28: Annual
  Conference on Neural Information Processing Systems 2015}, pages 721--729,
  2015.

\bibitem[BJKK17]{BhatiaJKK17}
K.~Bhatia, P.~Jain, P.~Kamalaruban, and P.~Kar.
\newblock Consistent robust regression.
\newblock In {\em Advances in Neural Information Processing Systems 30: Annual
  Conference on Neural Information Processing Systems 2017}, pages 2107--2116,
  2017.

\bibitem[BKNS00]{breunig2000lof}
M.~M> Breunig, H.~Kriegel, R.~T. Ng, and J.~Sander.
\newblock Lof: identifying density-based local outliers.
\newblock In {\em ACM sigmod record}, volume~29, pages 93--104. ACM, 2000.

\bibitem[BNJT10]{barreno2010security}
M.~Barreno, B.~Nelson, A.~D. Joseph, and J.~D. Tygar.
\newblock The security of machine learning.
\newblock {\em Machine Learning}, 81(2):121--148, 2010.

\bibitem[BNL12]{biggio2012poisoning}
B.~Biggio, B.~Nelson, and P.~Laskov.
\newblock Poisoning attacks against support vector machines.
\newblock In {\em International Conference on Machine Learning (ICML)}, pages
  1467--1474, 2012.

\bibitem[CLL{\etalchar{+}}17]{chen2017targeted}
X.~Chen, C.~Liu, B.~Li, K.~Lu, and D.~Song.
\newblock Targeted backdoor attacks on deep learning systems using data
  poisoning.
\newblock {\em arXiv preprint arXiv:1712.05526}, 2017.

\bibitem[CSV17]{CSV17}
M.~Charikar, J.~Steinhardt, and G.~Valiant.
\newblock Learning from untrusted data.
\newblock In {\em Proceedings of STOC 2017}, pages 47--60, 2017.

\bibitem[DKK{\etalchar{+}}16]{DKKLMS16}
I.~Diakonikolas, G.~Kamath, D.~M. Kane, J.~Li, A.~Moitra, and A.~Stewart.
\newblock Robust estimators in high dimensions without the computational
  intractability.
\newblock In {\em Proceedings of FOCS'16}, pages 655--664, 2016.

\bibitem[DKK{\etalchar{+}}17]{DKK+17}
I.~Diakonikolas, G.~Kamath, D.~M. Kane, J.~Li, A.~Moitra, and A.~Stewart.
\newblock Being robust (in high dimensions) can be practical.
\newblock In {\em Proceedings of the 34th International Conference on Machine
  Learning, {ICML} 2017}, pages 999--1008, 2017.
\newblock Full version available at https://arxiv.org/abs/1703.00893.

\bibitem[DKK{\etalchar{+}}18]{DKKLMS17}
I.~Diakonikolas, G.~Kamath, D.~M. Kane, J.~Li, A.~Moitra, and A.~Stewart.
\newblock Robustly learning a gaussian: Getting optimal error, efficiently.
\newblock In {\em Proceedings of the Twenty-Ninth Annual {ACM-SIAM} Symposium
  on Discrete Algorithms, {SODA} 2018}, pages 2683--2702, 2018.
\newblock Full version available at https://arxiv.org/abs/1704.03866.

\bibitem[DKS16]{DiakonikolasKS16b}
I.~Diakonikolas, D.~M. Kane, and A.~Stewart.
\newblock Robust learning of fixed-structure bayesian networks.
\newblock {\em CoRR}, abs/1606.07384, 2016.

\bibitem[DKS17a]{DKS17-nasty}
I.~Diakonikolas, D.~M. Kane, and A.~Stewart.
\newblock Learning geometric concepts with nasty noise.
\newblock {\em CoRR}, abs/1707.01242, 2017.

\bibitem[DKS17b]{DKS17-mixtures}
I.~Diakonikolas, D.~M. Kane, and A.~Stewart.
\newblock List-decodable robust mean estimation and learning mixtures of
  spherical gaussians.
\newblock {\em CoRR}, abs/1711.07211, 2017.

\bibitem[DKS17c]{DKS17-sq}
I.~Diakonikolas, D.~M. Kane, and A.~Stewart.
\newblock Statistical query lower bounds for robust estimation of
  high-dimensional gaussians and gaussian mixtures.
\newblock In {\em 58th {IEEE} Annual Symposium on Foundations of Computer
  Science, {FOCS} 2017}, pages 73--84, 2017.
\newblock Full version available at http://arxiv.org/abs/1611.03473.

\bibitem[DKS19]{DKoS18}
I.~Diakonikolas, W.~Kong, and A.~Stewart.
\newblock Efficient algorithms and lower bounds for robust linear regression.
\newblock In {\em Proceedings of the Twenty-Ninth Annual {ACM-SIAM} Symposium
  on Discrete Algorithms, {SODA} 2019}, pages 2745--2754, 2019.

\bibitem[EEF{\etalchar{+}}18]{evtimov2017robust}
K.~Eykholt, I.~Evtimov, E.~Fernandes, B.~Li, A.~Rahmati, C.~Xiao, A.~Prakash,
  T.~Kohno, and D.~Song.
\newblock Robust physical-world attacks on deep learning visual classification.
\newblock In {\em The IEEE Conference on Computer Vision and Pattern
  Recognition (CVPR)}, June 2018.

\bibitem[FB81]{fischler1981random}
M.~A. Fischler and R.~C. Bolles.
\newblock Random sample consensus: a paradigm for model fitting with
  applications to image analysis and automated cartography.
\newblock {\em Communications of the ACM}, 24(6):381--395, 1981.

\bibitem[{{Gur}}16]{gurobi2016}
{{Gurobi {Optimization}, Inc.}}
\newblock Gurobi optimizer reference manual, 2016.

\bibitem[HA04]{hodge2004survey}
V.~Hodge and J.~Austin.
\newblock A survey of outlier detection methodologies.
\newblock {\em Artificial intelligence review}, 22(2):85--126, 2004.

\bibitem[HL17]{HL17}
S.~B. Hopkins and J.~Li.
\newblock Mixture models, robustness, and sum of squares proofs.
\newblock {\em CoRR}, abs/1711.07454, 2017.

\bibitem[Hub64]{Huber64}
P.~J. Huber.
\newblock Robust estimation of a location parameter.
\newblock {\em Ann. Math. Statist.}, 35(1):73--101, 03 1964.

\bibitem[KL17]{koh2017understanding}
P.~W. Koh and P.~Liang.
\newblock Understanding black-box predictions via influence functions.
\newblock In {\em International Conference on Machine Learning (ICML)}, 2017.

\bibitem[KLS09]{klivans2009learning}
A.~R. Klivans, P.~M. Long, and R.~A. Servedio.
\newblock Learning halfspaces with malicious noise.
\newblock {\em Journal of Machine Learning Research (JMLR)}, 10:2715--2740,
  2009.

\bibitem[KS17a]{KStein17}
P.~K. Kothari and J.~Steinhardt.
\newblock Better agnostic clustering via relaxed tensor norms.
\newblock {\em CoRR}, abs/1711.07465, 2017.

\bibitem[KS17b]{KS17}
P.~K. Kothari and D.~Steurer.
\newblock Outlier-robust moment-estimation via sum-of-squares.
\newblock {\em CoRR}, abs/1711.11581, 2017.

\bibitem[KSL18]{koh2018stronger}
P.~W. Koh, J.~Steinhardt, and P.~Liang.
\newblock Stronger data poisoning attacks break data sanitization defenses.
\newblock {\em arXiv preprint arXiv:1811.00741}, 2018.

\bibitem[LAT{\etalchar{+}}08]{Li-Science08}
J.Z. Li, D.M. Absher, H.~Tang, A.M. Southwick, A.M. Casto, S.~Ramachandran,
  H.M. Cann, G.S. Barsh, M.~Feldman, L.L. Cavalli-Sforza, and R.M. Myers.
\newblock Worldwide human relationships inferred from genome-wide patterns of
  variation.
\newblock {\em Science}, 319:1100--1104, 2008.

\bibitem[L{\"{o}}f04]{lofberg2004}
J.~L{\"{o}}fberg.
\newblock {YALMIP}: A toolbox for modeling and optimization in {MATLAB}.
\newblock In {\em CACSD}, 2004.

\bibitem[LRV16]{LaiRV16}
K.~A. Lai, A.~B. Rao, and S.~Vempala.
\newblock Agnostic estimation of mean and covariance.
\newblock In {\em Proceedings of FOCS'16}, 2016.

\bibitem[LWSV16]{li2016data}
B.~Li, Y.~Wang, A.~Singh, and Y.~Vorobeychik.
\newblock Data poisoning attacks on factorization-based collaborative
  filtering.
\newblock In {\em Advances in Neural Information Processing Systems (NIPS)},
  2016.

\bibitem[MAP06]{metsis2006spam}
V.~Metsis, I.~Androutsopoulos, and G.~Paliouras.
\newblock Spam filtering with naive {B}ayes -- which naive {B}ayes?
\newblock In {\em CEAS}, volume~17, pages 28--69, 2006.

\bibitem[MV18]{meister2017data}
M.~Meister and G.~Valiant.
\newblock A data prism: Semi-verified learning in the small-alpha regime.
\newblock In {\em Proceedings of the 31st Conference On Learning Theory},
  volume~75 of {\em Proceedings of Machine Learning Research}, pages
  1530--1546. PMLR, 06--09 Jul 2018.

\bibitem[NPXNR14]{newell2014practicality}
A.~Newell, R.~Potharaju, L.~Xiang, and C.~Nita-Rotaru.
\newblock On the practicality of integrity attacks on document-level sentiment
  analysis.
\newblock In {\em Workshop on Artificial Intelligence and Security (AISec)},
  pages 83--93, 2014.

\bibitem[NT13]{nguyen2013exact}
N.~H. Nguyen and T.~D. Tran.
\newblock Exact recoverability from dense corrupted observations via
  $\ell_1$-minimization.
\newblock {\em IEEE Transactions on Information Theory}, 59(4):2017--2035,
  2013.

\bibitem[NTN11]{nasrabadi2011robust}
N.~M. Nasrabadi, T.~D. Tran, and N.~Nguyen.
\newblock Robust lasso with missing and grossly corrupted observations.
\newblock In {\em Advances in Neural Information Processing Systems (NIPS)},
  2011.

\bibitem[OSB{\etalchar{+}}18]{olier2018qsar}
I.~Olier, N.~Sadawi, G.~R. Bickerton, J.~Vanschoren, C.~Grosan, L.~Soldatova,
  and Ross~D. King.
\newblock Meta-qsar: a large-scale application of meta-learning to drug design
  and discovery.
\newblock {\em Machine Learning}, 107(1):285--311, Jan 2018.

\bibitem[Owe07]{owen2007robust}
A.~B. Owen.
\newblock A robust hybrid of lasso and ridge regression.
\newblock {\em Contemporary Mathematics}, 443(7):59--72, 2007.

\bibitem[PLJD10]{Pas-MG10}
P.~Paschou, J.~Lewis, A.~Javed, and P.~Drineas.
\newblock Ancestry informative markers for fine-scale individual assignment to
  worldwide populations.
\newblock {\em Journal of Medical Genetics}, 47:835--847, 2010.

\bibitem[PSBR18]{PSBR18}
A.~Prasad, A.~S. Suggala, S.~Balakrishnan, and P.~Ravikumar.
\newblock Robust estimation via robust gradient estimation.
\newblock {\em CoRR}, abs/1802.06485, 2018.

\bibitem[QV18]{qiao2017learning}
M.~Qiao and G.~Valiant.
\newblock Learning discrete distributions from untrusted batches.
\newblock In {\em Innovations in Theoretical Computer Science (ITCS)}, 2018.

\bibitem[RD99]{rousseeuw1999fast}
P.~J. Rousseeuw and K.~V. Driessen.
\newblock A fast algorithm for the minimum covariance determinant estimator.
\newblock {\em Technometrics}, 41(3):212--223, 1999.

\bibitem[RPW{\etalchar{+}}02]{RP-Gen02}
N.~Rosenberg, J.~Pritchard, J.~Weber, H.~Cann, K.~Kidd, L.A. Zhivotovsky, and
  M.W. Feldman.
\newblock Genetic structure of human populations.
\newblock {\em Science}, 298:2381--2385, 2002.

\bibitem[SCV18]{steinhardt2018resilience}
J.~Steinhardt, M.~Charikar, and G.~Valiant.
\newblock Resilience: A criterion for learning in the presence of arbitrary
  outliers.
\newblock In {\em Innovations in Theoretical Computer Science (ITCS)}, 2018.

\bibitem[SKL17]{steinhardt2017certified}
J.~Steinhardt, P.~W. Koh, and P.~Liang.
\newblock Certified defenses for data poisoning attacks.
\newblock In {\em Advances in Neural Information Processing Systems (NIPS)},
  2017.

\bibitem[Ste17]{steinhardt2017clique}
J.~Steinhardt.
\newblock Does robustness imply tractability? {A} lower bound for planted
  clique in the semi-random model.
\newblock {\em arXiv}, 2017.

\bibitem[SZS{\etalchar{+}}14]{szegedy2014intriguing}
C.~Szegedy, W.~Zaremba, I.~Sutskever, J.~Bruna, D.~Erhan, I.~Goodfellow, and
  R.~Fergus.
\newblock Intriguing properties of neural networks.
\newblock In {\em International Conference on Learning Representations (ICLR)},
  2014.

\bibitem[Tuk60]{tukey1960survey}
J.~W. Tukey.
\newblock A survey of sampling from contaminated distributions.
\newblock {\em Contributions to probability and statistics}, 2:448--485, 1960.

\bibitem[Tuk75]{Tukey75}
J.W. Tukey.
\newblock Mathematics and picturing of data.
\newblock In {\em Proceedings of ICM}, volume~6, pages 523--531, 1975.

\end{thebibliography}

\appendix

\section*{Appendix}

%!TEX root = ./main.tex

The appendix is organized as follows.
In Section~\ref{ssec:notation}, we describe additional preliminaries required for our technical arguments.
In Section~\ref{sec:sever-general}, we analyze our main algorithm, \sever{}.
In Section~\ref{sec:glms}, we specialize our analysis to the important case of Generalized Linear Models (GLMs).
In Section~\ref{sec:general-algo}, we describe a variant of our algorithm which performs robust filtering on each iteration of projected gradient descent, and works under more general assumptions.
In Section~\ref{sec:app-general}, we describe concrete applications of \sever{} -- in particular, how it can be used to robustly optimize in the settings of linear regression, logistic regression, and support vector machines.
Finally, in Section~\ref{sec:experiments-app}, we provide additional plots from our experimental evaluations.

\section{Preliminaries} \label{ssec:notation}

In this section, we formally introduce our setting for robust stochastic optimization.

\paragraph{Notation.} For $n \in \Z_+$, we will denote $[n] \eqdef \{1, \ldots, n\}$.
For a vector $v$, we will let $\| v \|_2$ denote its Euclidean norm.
For any $r \geq 0$ and any $x \in \R^d$, let $B(x, r)$ be the $\ell_2$ ball of radius $r$ around $x$.
If $M$ is a matrix, we will let $\| M \|_2$ denote its spectral norm and $\| M \|_F$ denote its Frobenius norm.
We will write $X \sim_u S$ to denote that $X$ is drawn from the empirical distribution defined by $S$.
We will sometimes use the notation $\E_{S}$, instead of $\E_{X\sim S}$, for the corresponding expectation.
We will also use the same convention for the covariance, i.e. we let $\Cov_S$ denote the covariance over the empirical distribution.

\paragraph{Setting.}
We consider a stochastic optimization setting with outliers.
Let $\dom \subseteq \bR^d$ be a space of parameters. 
We observe $n$ functions $f_1, \ldots, f_n: \dom \to \bR$
and we are interested in (approximately) minimizing some target function 
$\Ef: \dom \to \bR$, \new{related to the $f_i$'s.}
We will assume for simplicity that the $f_i$'s are differentiable with gradient $\nabla f_i$. 
(Our results can be easily extended for the case that only a sub-gradient is available.)

In most concrete applications we will consider, there is some true underlying distribution $p^{\ast}$
over functions $f : \sH \to \bR$, and our goal is to find a parameter vector 
$w^{\ast} \in \sH$ minimizing $\overline{f}(w) \eqdef \bE_{f \sim p^{\ast}}[f(w)]$.
Unlike the classical \new{realizable} setting, where we assume that
$f_1, \ldots, f_n \sim p^{\ast}$, we allow for an $\epsilon$-fraction of the points to be
arbitrary outliers. This is captured in the following definition (Definition~\ref{def:eps-contam})
that we restate for convenience:

\begin{definition}[$\epsilon$-corruption model]
Given $\eps > 0$ and a distribution $p^{\ast}$ over functions $f : \sH \to \bR$, data is generated as follows:
first, $n$ clean samples $f_1, \ldots, f_{n}$ are drawn from $p^{\ast}$.
Then, an \emph{adversary} is allowed to inspect the samples and replace
any $\epsilon n$ of them with arbitrary samples.
The resulting set of points is then given to the algorithm.
\end{definition}

In addition, some of our bounds will make use of the following quantities:
\begin{itemize}
\item The $\ell_2$-radius $\radius$ of the domain $\dom$: $\radius = \max_{w \in \dom} \|w\|_2$.
\item The strong convexity parameter $\strconv$ of $\Ef$, if it exists.
      This is the maximal $\strconv$ such that $\Ef(w) \geq \Ef(w_0) + \langle w - w_0, \nabla \Ef(w_0) \rangle + \frac{\strconv}{2} \|w-w_0\|_2^2$ for all $w, w_0 \in \dom$.
\item The strong smoothness parameter $\strsmooth$ of $\Ef$, if it exists.
      This is the minimal $\strsmooth$ such that $\Ef(w) \leq \Ef(w_0) + \langle w - w_0, \nabla \Ef(w_0) \rangle + \frac{\strsmooth}{2} \|w-w_0\|_2^2$ for all $w, w_0 \in \dom$.
\item The Lipschitz constant $\lip$ of $\Ef$, if it exists.
      This is the minimal $\lip$ such that $\Ef(w) - \Ef(w_0) \leq \lip \|w-w_0\|_2$ for
      all $w, w_0 \in \dom$.
\end{itemize}

%!TEX root = ./main.tex

\section{General Analysis of \sever{}} \label{sec:sever-general}

This section is dedicated to the analysis of Algorithm~\ref{alg:sever}, where
we do not make convexity assumptions about the underlying functions $f_1, \ldots, f_n$.
In this case, we can show that our algorithm finds an approximate critical point of
$\Ef$.
When we specialize to convex functions, this immediately implies that we find an approximate minimal point of $\Ef$.

Our proof proceeds in two parts.
First, we define a set of deterministic conditions under which our algorithm finds an approximate minimal point of $\Ef$.
We then show that, under mild assumptions on our functions, this set of deterministic conditions holds with high probability after polynomially many samples.

For completeness, we recall the definitions of a $\gamma$-approximate critical
point and a $\gamma$-approximate learner:

\approxcrit*

\approxlearner*

\paragraph{Deterministic Regularity Conditions}

We first explicitly demonstrate a set of deterministic conditions on the (uncorrupted) data points.
%Then, our proof will proceed in two steps.
%First, we will show that these conditions hold with high probability after polynomially many samples.
%Then, we will show that under these conditions, the algorithm succeeds (Lemma~\ref{lem:good-set-implies-learner}).
Our deterministic regularity conditions are as follows:
\begin{assumption}
\label{ass:one-good-set}
Fix $0<\eps<1/2$. There exists an unknown set $\goodset \subseteq [n]$ with $|\goodset| \geq (1-\eps)n$
of ``good'' functions $\{f_i\}_{i \in \goodset}$ and parameters $\sigma_0, \sigma_1 \in \R_+$
such that:
\begin{equation} \label{eq:norm-bound}
\Big\|\E_{\goodset}\big[\big(\nabla f_i(w) - \nabla \Ef(w)\big)\big(\nabla f_i(w) - \nabla \Ef(w)\big)^T\big]\Big\|_2 \leq (\sigma_0 + \sigma_1 \|\wstar - w\|_2)^2, \textrm{ for all } w \in \dom  \;,
\end{equation}
and
\begin{equation} \label{eq:mean-everywhere}
\|\nabla \hatf(w) - \nabla \Ef(w)\|_2 \leq (\sigma_0 + \sigma_1 \|\wstar - w\|_2)\sqrt{\badfrac},  \textrm{ for all } w \in \dom,
\textrm{ where } \hatf \eqdef \frac{1}{|\goodset|} \sum_{i \in \goodset} f_i \;.
\end{equation}
\end{assumption}

% Condition~\eqref{eq:norm-bound} is critical in order to ensure that our robust algorithms apply.
% Condition~\eqref{eq:mean-everywhere} is necessary to ensure that the average of the good $f_i$'s
% and $\Ef$ have similar minima.
%For our applications to generalized linear models, it will be more convenient
%to use Conditions \eqref{eq:mean-at-min} and \eqref{eq:loss-close},
%instead of Condition~\eqref{eq:mean-everywhere}.

In Section~\ref{sec:stationary-point}, we prove the following theorem, 
which shows that under Assumption~\ref{ass:one-good-set} our algorithm succeeds:

\begin{restatable}{theorem}{stationarypoint}
\label{thm:stationary-point}
Suppose that the functions $f_1,\ldots,f_n,\Ef:\dom\rightarrow \R$ are bounded below, and
that Assumption~\ref{ass:one-good-set} is satisfied, where $\sigma \eqdef \sigma_0 + \sigma_1 \|\wstar - w\|_2.$
Then \sever{} applied to $f_1,\ldots,f_n, \sigma$ returns a point $w \in \dom$
that, with probability at least $9/10$, 
is a $(\gamma+O(\sigma \sqrt{\eps}))$-approximate critical point of $\Ef$.
\end{restatable}

Observe that the above theorem holds quite generally; in particular, it holds for non-convex functions.
As a corollary of this theorem, in Section~\ref{sec:sever-convex} we show that this immediately 
implies that \sever{} robustly minimizes convex functions, if Assumption~\ref{ass:one-good-set} holds:
\begin{corollary}\label{cor:convex-sever}
 For functions $f_1, \ldots, f_n: \dom \to \R$, suppose that Assumption~\ref{ass:one-good-set}
 holds and that $\dom$ is convex.
 Then, with probability at least $9/10$, for some universal constant $\badfrac_0$, if $\badfrac < \badfrac_0$, 
 the output of \sever{} satisfies the following:
 \begin{enumerate}
 \item[(i)] If $\Efunc$ is convex, the algorithm finds a $w \in \dom$ such that
       $\Ef(w) - \Ef(\wstar) = O((\sigma_0\radius + \sigma_1\radius^2)\sqrt{\badfrac} + \gamma r)$.
 \item[(ii)] If $\Efunc$ is $\strconv$-strongly convex,  the algorithm finds a $w \in \dom$ such that
       \[
       \Ef(w) - \Ef(\wstar) = O\left(\frac{\eps}{\strconv} (\sigma_0 + \sigma_1 r)^2 + \frac{\gamma^2}{\strconv} \right) \; .
       \]
       % ,provided that $\strconv \geq C \cdot \sigma_1 \sqrt{\badfrac}$ for an absolute constant $C$.
 \end{enumerate}
 % For 1, the number of calls to the filter, Algorithm ${\cal A}'$, that do not remove samples is $1$.
 % For 2, The number of calls to ${\cal A}'$ that do not remove samples is $O(\log(\radius\strconv/\sigma_0\sqrt{\badfrac}))$
 % if $\Ef$ is $\lip$-Lipschitz and $O(\log(\radius\strconv/\sigma_0\sqrt{\badfrac}))$ if $\Ef$ is $\strsmooth$-smooth.
 \end{corollary}
 
 \new{In the strongly convex case and when $\sigma_1 > 0$, we can remove the dependence on $\sigma_1$ and $r$ in the above by repeatedly applying  \sever{} with decreasing $r$:
 
 \begin{corollary}\label{cor:strongly-convex-sever}
 For functions $f_1, \ldots, f_n: \dom \to \R$, suppose that Assumption~\ref{ass:one-good-set}
 holds, that $\dom$ is convex and that $\Efunc$ is $\strconv$-strongly convex for $\strconv \geq C \sigma_1 \sqrt{\eps}$ for some absolute constant $C$.
 Then, with probability at least $9/10$, for some universal constant $\badfrac_0$, if $\badfrac < \badfrac_0$, 
 we can find a $\what$ with 
  \[
       \Ef(\what) - \Ef(\wstar) = O\left(\frac{\eps \sigma_0^2 +\gamma^2}{\strconv} \right) \; .
  \]
  and 
  \[
  \|\what-\wstar\|_2 = O\left(\frac{\sqrt{\eps} \sigma_0 + \gamma}{\strconv}\right)
  \]
 
 using at most $O(\log(\radius\strconv/(\gamma +\sigma_0\sqrt{\badfrac})))$ calls to \sever{}.
\end{corollary}
 }

% The following proposition establishes that Assumption~\ref{ass:one-good-set} is satisfied with high probability
% under some niceness assumptions:
To concretely use Theorem~\ref{thm:stationary-point}, Corollary~\ref{cor:convex-sever}, and Corollary~\ref{cor:strongly-convex-sever}, in Section~\ref{sec:sever-sample-complexity} we
show that the Assumption~\ref{ass:one-good-set} is satisfied with high probability under mild conditions on the distribution over the functions, after drawing polynomially many samples:

\begin{proposition} \label{prop:sample-bound}
Let $\dom \subset \R^d$ be a closed bounded set with diameter at most $r$.
Let $p^{\ast}$ be a distribution over functions $f:\dom\rightarrow \R$ with $\Ef=\E_{f \sim p^{\ast}}[f]$
so that $f-\Ef$ is $L$-Lipschitz and $\beta$-smooth almost surely.
Assume furthermore that for each $w\in \dom$ and unit vector $v$ that
$\E_{f \sim p^{\ast}}[(v\cdot (\nabla f(w)-\Ef(w)))^2] \leq \sigma^2 /2.$
Then for 
\[
n = \Omega \left( \frac{d L^2 \log (r \beta L / \sigma^2 \eps)}{\sigma^2 \eps} \right) \; ,
\]
an $\eps$-corrupted set of points $f_1,\ldots,f_n$  with high probability
satisfy Assumption~\ref{ass:one-good-set}.
\end{proposition}

\noindent
The remaining subsections are dedicated to the proofs of Theorem~\ref{thm:stationary-point}, Corollary~\ref{cor:convex-sever}, Corollary~\ref{cor:strongly-convex-sever}, and Proposition~\ref{prop:sample-bound}.

\subsection{Proof of Theorem~\ref{thm:stationary-point}}
\label{sec:stationary-point}
\noindent Throughout this proof we let $\goodset$ be as in Assumption~\ref{ass:one-good-set}.
We require the following two lemmata.
Roughly speaking, the first states that on average, we remove more corrupted points than uncorrupted points, and the second states that at termination, and if we have not removed too many points, then we have reached a point at which the empirical gradient is close to the true gradient.
Formally:

\begin{lemma}\label{lem:bad-elts}
If the samples satisfy \eqref{eq:norm-bound} of Assumption~\ref{ass:one-good-set}, and if $|S|\geq 2n/3$ then if $S'$ is the output of $\textsc{Filter}(S, \tau, \sigma)$, we have that
$$
\E[|\goodset\cap (S\backslash S')|] \leq \E[|([n]\backslash\goodset)\cap(S\backslash S')|].
$$
\end{lemma}

\begin{lemma}\label{lem:final-set}
If the samples satisfy Assumption~\ref{ass:one-good-set}, $\textsc{Filter}(S, \tau, \sigma) = S$, and $n-|S| \leq 11 \eps n$, then
$$
\left\|\nabla \Ef({w}) - \frac{1}{|\goodset|}\sum_{i\in S} \nabla f_i({w})\right\|_2 \leq O(\sigma \sqrt{\eps})
$$
\end{lemma}

Before we prove these lemmata, we show how together they imply Theorem~\ref{thm:stationary-point}.

\begin{proof}[{\bf Proof of Theorem~\ref{thm:stationary-point} assuming Lemma~\ref{lem:bad-elts} and Lemma~\ref{lem:final-set}}]
First, we note that the algorithm must terminate in at most $n$ iterations.
This is easy to see as each iteration of the main loop except for the last must decrease the size of $S$ by at least $1$.

It thus suffices to prove correctness.
Note that Lemma \ref{lem:bad-elts} says that each iteration will on average throw out as many elements not in $\goodset$ from $S$ as elements in $\goodset$. In particular, this means that $|([n]\backslash\goodset)\cap S| + |\goodset\backslash S|$ is a supermartingale. Since its initial size is at most $\eps n$, with probability at least $9 / 10$, it never exceeds $10\eps n$, and therefore at the end of the algorithm, we must have that $n-|S| \leq \eps n + |\goodset\backslash S| \leq 11\eps n$. This will allow us to apply Lemma \ref{lem:final-set} to complete the proof, using the fact that $w$ is a $\gamma$-approximate critical point of $\frac{1}{|\goodset|}\sum_{i\in S} \nabla f_i({w})$.
\end{proof}

\noindent
Thus it suffices to prove these two lemmata.
We first prove Lemma~\ref{lem:bad-elts}:

\begin{proof}[{\bf Proof of Lemma \ref{lem:bad-elts}}]
Let $\Sgood = S\cap \goodset$ and $\Sbad =S\backslash \goodset$. We wish to show that the expected number of elements thrown out of $\Sbad$ is at least the expected number thrown out of $\Sgood$.
We note that our result holds trivially if $\textsc{Filter}(S, \tau, \sigma) = S$.
Thus, we can assume that $\E_{i\in S}[\tau_i] \geq 12\sigma$.

It is easy to see that the expected number of elements thrown out of $\Sbad$ is proportional to $\sum_{i\in \Sbad}\tau_i$, while the number removed from $\Sgood$ is proportional to $\sum_{i\in \Sgood}\tau_i$ (with the same proportionality).
Hence, it suffices to show that $\sum_{i\in \Sbad}\tau_i \geq  \sum_{i\in \Sgood}\tau_i$.

We first note that since $\Cov_{i\in \goodset} [ \nabla f_i(w) ] \preceq \sigma^2 I$, we have that
\begin{align*}
\Cov_{i\in \Sgood} [ v\cdot \nabla f_i(w)] &\stackrel{(a)}{\leq} \frac{3}{2} \Cov_{i \in \goodset} [v \cdot \nabla f_i (w)] \\
&= \frac{3}{2} \cdot v^\top \Cov_{i \in \goodset} [\nabla f_i (w)] v \leq 2 \sigma^2 \; ,
\end{align*}
where (a) follows since $|\Sgood| \geq \frac{3}{2} \goodset$.

 Let $\mugood =\E_{i\in \Sgood}[v\cdot \nabla f_i(w)]$ and $\mu=\E_{i\in S}[v\cdot \nabla f_i(w)]$.
 Note that
 \[
 \E_{i\in \Sgood } [\tau_i] = \Cov_{i\in \Sgood}[ v\cdot \nabla f_i(w)] + (\mu-\mugood)^2 \leq 2\sigma + (\mu-\mugood)^2 \; .
 \]
 \noindent We now split into two cases.

Firstly, if $(\mu-\mugood)^2 \geq 4\sigma^2$, we let $\mubad=\E_{i\in \Sbad}[v\cdot \nabla f_i(w)]$, and note that $| \mu -\mubad | |\Sbad| = |\mu-\mugood||\Sgood|$. We then have that
\begin{align*}
\E_{i\in \Sbad} [\tau_i] &\geq (\mu-\mubad)^2 \\
&\geq (\mu-\mugood)^2 \left( \frac{|\Sgood|}{|\Sbad|} \right)^2 \\
&\geq 2 \left( \frac{|\Sgood|}{|\Sbad|} \right) (\mu-\mugood)^2 \\
&\geq \left( \frac{|\Sgood|}{|\Sbad|} \right)\E_{i\in \Sgood} [\tau_i].
\end{align*}
Hence, $\sum_{i\in \Sbad}\tau_i \geq  \sum_{i\in \Sgood}\tau_i$.

On the other hand, if $(\mu-\mugood)^2 \leq 4\sigma^2$, then $\E_{i\in \Sgood} [\tau_i] \leq 6\sigma^2 \leq \E_{i\in S} [\tau_i]/2$. Therefore $\sum_{i\in \Sbad}\tau_i \geq  \sum_{i\in \Sgood}\tau_i$ once again.
This completes our proof.
\end{proof}

\noindent
We now prove Lemma \ref{lem:final-set}.
\begin{proof}[{\bf Proof of Lemma \ref{lem:final-set}}]
We need to show that
$$
\delta := \left\| \sum_{i\in S} (\nabla f_i(w) -\nabla \Ef(w)) \right\|_2 = O(n \sigma \sqrt{ \eps}).
$$
We note that
\begin{align*}
& \left\| \sum_{i\in S} (\nabla f_i(w) -\nabla \Ef(w)) \right\|_2\\
\leq & \left\|\sum_{i\in \goodset} (\nabla f_i(w) -\nabla \Ef(w)) \right\|_2 + \left\|\sum_{i\in (\goodset\backslash S)} (\nabla f_i(w) -\nabla \Ef(w)) \right\|_2 + \left\| \sum_{i\in (S\backslash \goodset)} (\nabla f_i(w) -\nabla \Ef(w)) \right\|_2\\
= & \left\|\sum_{i\in (\goodset\backslash S)} (\nabla f_i(w) -\nabla \Ef(w)) \right\|_2 + \left\|\sum_{i\in (S\backslash \goodset)} (\nabla f_i(w) -\nabla \Ef(w)) \right\|_2 + O(n\sqrt{\sigma^2 \eps}).\\
\end{align*}
First we analyze
$$
\left\|\sum_{i\in (\goodset\backslash S)} (\nabla f_i(w) -\nabla \Ef(w)) \right\|_2.
$$
This is the supremum over unit vectors $v$ of
$$
\sum_{i\in (\goodset\backslash S)} v\cdot (\nabla f_i(w) -\nabla \Ef(w)).
$$
However, we note that
$$
\sum_{i\in \goodset} (v\cdot (\nabla f_i(w) -\nabla \Ef(w)))^2 = O(n \sigma^2).
$$
Since $|\goodset\backslash S| = O( n\eps)$, we have by Cauchy-Schwarz that
$$
\sum_{i\in (\goodset\backslash S)} v\cdot (\nabla f_i(w) -\nabla \Ef(w)) = O(\sqrt{(n\sigma^2)(n\eps)}) = O(n\sqrt{\sigma^2 \eps}),
$$
as desired.

We note that since for any such $v$ that
$$
\sum_{i\in S} (v\cdot (\nabla f_i(w) -\nabla \Ef(w)))^2 = \sum_{i\in S} (v\cdot (\nabla f_i(w) -\nabla \hat{f}(w)))^2 + \delta^2 = O(n \sigma^2) + \delta^2
$$
(or otherwise our filter would have removed elements) and since $|S\backslash \goodset| = O(n\eps)$, and so we have similarly that
$$
\left\|\sum_{i\in (S\backslash \goodset)} \nabla f_i(w) -\nabla \Ef(w)\right\|_2 = O(n \sigma \sqrt{\eps}+\delta\sqrt{n\eps}).
$$
Combining with the above we have that
$$
\delta = O(\sigma\sqrt{\eps}+\delta\sqrt{\eps/n}), 
$$
and therefore, $\delta=O(\sigma\sqrt{\eps})$ as desired.
\end{proof}

\subsection{Proof of Corollary~\ref{cor:convex-sever}}
\label{sec:sever-convex}
In this section, we show that the  \sever{} algorithm finds an approximate global optimum
for convex optimization in various settings, under Assumption~\ref{ass:one-good-set}.
We do so by simply applying the guarantees of Theorem~\ref{thm:stationary-point} in a fairly black box manner.

 %We now describe the version of our algorithm that only filters at the empirical critical points.
 %We note that this is the algorithm implemented for our experimental evaluation.
 %There are two main motivations for this modified algorithm:
 %Firstly, we empirically observed that if we iteratively filter samples,
 %keeping the subset with the samples removed, then few iterations of the filter remove points.
 %Secondly, an iteration of Algorithm $A$ is much more expensive than an iteration of gradient descent.
 %Therefore, it makes sense to run many steps of our gradient descent between filtering steps. In fact,
 %we can only apply our filters once our gradient descent has reached an empirical minimum.
 %This idea is further improved in practice by using stochastic gradient descent, rather than computing the average at each step.
 %t turns out that this algorithm has similar provable guarantees as the one in the previous subsection
 %and we provide an analysis here.}

% For technical reasons, in addition to Assumption~\ref{ass:one-good-set} given in the previous subsection,
% we will also require the following variant set of deterministic conditions:

% %eq:norm-bound2

% Condition~\eqref{eq:norm-bound} is critical in order to ensure that our robust algorithms apply.
% For our applications to generalized linear models, it will be more convenient
% to use Conditions \eqref{eq:mean-at-min} and \eqref{eq:loss-close},
% instead of Condition~\eqref{eq:mean-everywhere}.

Before we proceed with the proof of Corollary~\ref{cor:convex-sever}, we record
a simple lemma that allows us to translate an approximate critical point guarantee to an approximate
global optimum guarantee:

\begin{lemma}\label{lem:derivatives-to-sizes}
Let $f: \dom \to \R$ be a convex function and let $x \neq y \in \dom$. Let $v = y-x / \|y-x\|_2$
be the unit vector in the direction of $y-x$. Suppose that for some $\delta$ that $v\cdot (\nabla f(x)) \geq -\delta$ and $ -v\cdot (\nabla f(y)) \geq -\delta$ . Then we have that:
\begin{enumerate}
\item $|f(x)-f(y)|\leq \|x-y\|_2 \delta$.
\item If $f$ is $\strconv$-strongly convex, then $|f(x)-f(y)| \leq 2\delta^2 / \xi$ \new{ and $\|x-y\|_2 \leq 2\delta/\xi$}.
\end{enumerate}
\end{lemma}
\begin{proof}
Let $r=\|x-y\|_2 >0$ and $g(t)=f(x+tv)$.
We have that $g(0)=f(x),g(r)=f(y)$ and that $g$ is convex (or $\strconv$-strongly convex)
with $g'(0)\geq -\delta$ and $g'(r)\leq \delta$. By convexity, the derivative of $g$ is increasing on $[0,r]$
and therefore $|g'(t)|\leq \delta$ for all $t\in [0,r]$. This implies that
$$
|f(x)-f(y)| = |g(r)-g(0)| = \left| \int_0^r g'(t)dt \right| \leq r\delta \;.
$$
To show the second part of the lemma, we note that if $g$ is $\strconv$-strongly convex that $g''(t)\geq \xi$ for all $t$.
This implies that $g'(r)>g'(0)+\xi r$. Since $g'(r)-g'(0)\leq 2\delta$, we obtain
that $r\leq 2\delta/\xi$, from which the second statement follows.
\end{proof}

\begin{proof}[Proof of Corollary~\ref{cor:convex-sever}]
 By applying the algorithm of Theorem~\ref{thm:stationary-point},
we can find a point $w$ that is a $\gamma' \eqdef (\gamma + O(\sigma \sqrt{\eps}))$-approximate critical
point of $\Ef$, where $\sigma \eqdef \sigma_0+ \sigma_1 \|w^{\ast}-w\|_2$. That is,
for any unit vector $v$ pointing towards the interior of $\dom$,
we have that $v\cdot \nabla \Ef(w) \geq - \gamma'.$

To prove (i), we apply Lemma~\ref{lem:derivatives-to-sizes} to $\Ef$ at $w$ which gives that
$$
|\Ef(w)-\Ef(w^{\ast})| \leq r \cdot \gamma'.
$$
To prove (ii), we apply Lemma~\ref{lem:derivatives-to-sizes} to $\Ef$ at $w$ which gives that
$$
|\Ef(w)-\Ef(w^{\ast})| \leq 2  {\gamma'}^2/\strconv.
$$
Plugging in parameters appropriately then immediately gives the desired bound.
 \end{proof}

\subsection{Proof of Corollary~\ref{cor:strongly-convex-sever}}

 We apply $\sever{}$ iteratively starting with a domain $\dom_1=\dom$ and radius $r_1=r$. After each iteration, we know the resulting point is close to $w^{\ast}$ will be able to reduce the search radius.
 
 At step $i$, we have a domain of radius $r_i$.
As in the proof of Corollary~\ref{cor:convex-sever} above, we apply algorithm of Theorem~\ref{thm:stationary-point},
we can find a point $w_i$ that is a $\gamma'_i \eqdef (\gamma + O(\sigma'_i \sqrt{\eps}))$-approximate critical
point of $\Ef$, where $\sigma'_i \eqdef \sigma_0+ \sigma_1 r_i$.
 Then using Lemma~\ref{lem:derivatives-to-sizes}, we obtain that $\|w_i-w^{\ast}\|_2 \leq 2\gamma'_i/\strconv$.
 
 Now we can define $\dom_{i+1}$ as the intersection of $\dom$ and the ball of radius $r_{i+1} = 2\gamma'_i/\strconv$ around $w_i$ and repeat using this domain.
 We have that $r_{i+1} = 2\gamma'_i/\strconv= 2\gamma/\strconv + O(\sigma_0 \sqrt{\eps}/\strconv + \sigma_1 \sqrt{\eps} r_i/\strconv)$.  Now if we choose the constant $C$ such that the constant in this $O()$ is $C/4$, then using our assumption that $\strconv \geq 2 \sigma_1 \sqrt{\eps}$, we obtain that
 $$r_{i+1} \leq 2\gamma/\strconv + C \sigma_0 \sqrt{\eps}/4\strconv+ C\sigma_1 \sqrt{\eps} r_i/4\strconv \leq 
 2\gamma/\strconv + C \sigma_0 \sqrt{\eps}/4 + r_i/4$$
 Now if $r_i \geq 8\gamma/\strconv + 2C \sigma_0 \sqrt{\eps}/\strconv$, then we have $r_{i+1} \leq r_i/2$ and if
 $r_i \leq 8\gamma/\strconv + 2C \sigma_0 \sqrt{\eps}/\strconv$ then we also have  $r_{i+1} \leq 8\gamma/\strconv + 2C \sigma_0 \sqrt{\eps}/\strconv$ . When $r_i$ is smaller than this we stop and output $w_i$.
 Thus we stop in at most $O(\log(r) -\log(8\gamma/\strconv + 2C \sigma_0 \sqrt{\eps}/\strconv))=O(\log(r\strconv/(\gamma + \sigma_0 \sqrt{\eps}))$ iterations and have $r_i=O(\gamma/\strconv + C \sigma_0 \sqrt{\eps})$. But then $\gamma'_i =\gamma + O(\sigma'_i \sqrt{\eps})) \leq \gamma + C(\sigma_0 + \sigma_1 r'_i) \sqrt{\eps}/8 = O(\gamma + \sigma_0 \sqrt{\eps}).$ Using Lemma~\ref{lem:derivatives-to-sizes} we obtain that
 $$
|\Ef(w_i)-\Ef(w^{\ast})| \leq 2  \gamma'^2_i/\strconv = O(\gamma^2/\strconv + \sigma_0^2 \eps/\strconv).
$$
as required.
The bound on $\|\what - \wstar\|_2$ follows similarly.

\begin{remark}
 While we don't give explicit bounds on the number of calls to the approximate learner needed by \sever{},
 such bounds can be straightforwardly obtained under appropriate assumptions on the $f_i$ (see, e.g., the following subsection).
 Two remarks are in order. First, in this case we cannot
 take advantage of assumptions that only hold at $\Ef$ but might not on the corrupted average $f$.
 Second, our algorithm can take advantage of a closed form for the minimum.
 For example, for the case of linear regression considered in Section~\ref{sec:app-general},
 $f_i$ is not Lipschitz with a small constant if $x_i$ is far from the mean,
 but there is a simple closed form for the minimum of the least squares loss.
 \end{remark}

\subsection{Proof of Proposition~\ref{prop:sample-bound}}
\label{sec:sever-sample-complexity}

We let $\goodset$ be the set of uncorrupted functions $f_i$. It is then the case that $|\goodset|\geq (1-\eps)n$. We need to show that for each $w\in\dom$ that
\begin{equation}\label{eqn:cov-bound}
\Cov_{i\in \goodset}[\nabla f_i(w)] \leq 3\sigma^2 I/4
\end{equation}
and
\begin{equation}\label{eqn:average-grad-error-bound}
\left\|\nabla \Ef(w) - \frac{1}{|\goodset|}\sum_{i\in\goodset} \nabla f_i(w)\right\|_2 \leq O(\sigma^2 \sqrt{\eps}).
\end{equation}
We will proceed by a cover argument. First we claim that for each $w\in \dom$ that \eqref{eqn:cov-bound} and \eqref{eqn:average-grad-error-bound} hold with high probability. For Equation \eqref{eqn:cov-bound}, it suffices to show that for each unit vector $v$ in a cover $\mathcal{N}$ of size $2^{O(d)}$ of the sphere that
\begin{equation}\label{eqn:direction-var-bound}
\E_{i\in\goodset}[(v\cdot (\nabla f_i(w)-\Ef))^2] \leq 2\sigma^2 /3.
\end{equation}
However, we note that
$$
\E_{p^\ast}[(v\cdot (\nabla f(w)-\Ef))^2] \leq \sigma^2/2.
$$
Since $|v\cdot (\nabla f(w)-\Ef)|$ is always bounded by $L$, Equation \eqref{eqn:direction-var-bound} holds for each $v,w$ with probability at least $1-\exp(-\Omega(n\sigma^2 /L^2))$ by a Chernoff bound (noting that the removal of an $\eps$-fraction of points cannot increase this by much). Similarly, to show Equation \ref{eqn:average-grad-error-bound}, it suffices to show that for each such $v$ that
\begin{equation}\label{eqn:directional-average-grad-error-bound}
\E_{i\in\goodset}[(v\cdot (\nabla f_i(w)-\Ef))] \leq O(\sigma \sqrt{\eps}).
\end{equation}
Noting that
$$
\E_{p^\ast}[(v\cdot (\nabla f(w)-\Ef))]=0
$$
A Chernoff bound implies that with probability $1-\exp(-\Omega(n\sigma^2 \eps/L^2))$ that the average over our original set of $f$'s of $(v\cdot (\nabla f(w)-\Ef))$ is $O(\sigma \sqrt{\eps})$.
Assuming that Equation \eqref{eqn:direction-var-bound} holds, removing an $\eps$-fraction of these $f$'s cannot change this value by more than $O(\sigma \sqrt{\eps})$.
By union bounding over $\mathcal{N}$ and standard net arguments, this implies that 
Equations \eqref{eqn:cov-bound} and \eqref{eqn:average-grad-error-bound} hold with probability $1-\exp(\Omega(d - n\sigma^2 \eps/L^2))$ for any given $w$.

To show that our conditions hold for all $w\in \dom$, we note that by $\beta$-smoothness, if Equation \eqref{eqn:average-grad-error-bound} holds for some $w$, it holds for all other $w'$ in a ball of radius $\sqrt{\sigma^2 \eps}/\beta$ (up to a constant multiplicative loss). Similarly, if Equation \eqref{eqn:cov-bound} holds at some $w$, it holds with bound $\sigma^2 I$ for all $w'$ in a ball of radius $\sigma^2 /(2L\beta)$. Therefore, if Equations \eqref{eqn:cov-bound} and \eqref{eqn:average-grad-error-bound} hold for all $w$ in a $\min(\sqrt{\sigma^2 \eps}/\beta,\sigma/(2L\beta))$-cover of $\dom$, the assumptions of Theorem  \ref{thm:stationary-point} will hold everywhere. Since we have such covers of size $\exp(O(d\log(r\beta L/(\sigma^2 \eps))))$, by a union bound, this holds with high probability if 
\[
n = \Omega \left( \frac{d L^2 \log (r \beta L / \sigma^2 \eps)}{\sigma^2 \eps} \right) \; ,
\]
as claimed.

%!TEX root = ./main.tex

\section{Analysis of \sever{} for GLMs} \label{sec:glms}
A case of particular interest is that of Generalized Linear Models (GLMs):
\begin{definition}
\label{def:glm}
Let $\dom \subseteq \R^d$ and $\cY$ be an arbitrary set.
Let  $D_{xy}$ be a distribution over $\dom \times \cY$.
For each $Y \in \cY$, let $\sigma_Y: \R \to \R$ be a convex function. 
The \emph{generalized linear model} (GLM) over $\dom \times \cY$ with \emph{distribution} $D_{xy}$ and \emph{link functions} $\sigma_Y$ is the function $\Ef: \R^d \to \R$ defined by $\Ef (w) =\E_{X,Y}[f_{X,Y} (w)]$, where
\[
f_{X,Y}(w) := \sigma_Y(w\cdot X) \; .
\]
A \emph{sample} from this GLM is given by $f_{X, Y} (w)$ where $(X, Y) \sim D_{xy}$.
\end{definition}
\noindent
Our goal, as usual, is to approximately minimize $\Ef$ given $\eps$-corrupted samples from $D_{xy}$.
Throughout this section we assume that $\dom$ is contained in the ball of radius $r$ around $0$, i.e. $\dom \subseteq B(0, r)$.
Moreover, we will let $w^* = \argmin_{w \in \dom} \Ef (w)$ be a minimizer of $\Ef$ in $\dom$.

This case covers a number of interesting applications, including SVMs and logistic regression.
Unfortunately, the tools developed in Appendix~\ref{sec:sever-general} do not seem to be able to cover this case in a simple manner.
In particular, it is unclear how to demonstrate that Assumption~\ref{ass:one-good-set} holds after taking polynomially many samples from a GLM.
To rectify this, in this section, we demonstrate a different deterministic regularity condition under which we show \sever{} succeeds, and we show that this condition holds after polynomially many samples from a GLM.
Specifically, we will show that \sever{} succeeds under the following deterministic condition:

\begin{assumption} \label{ass:one-good-set-glm}
Fix $0<\eps<1/2$. There exists an unknown set $\goodset \subseteq [n]$ with $|\goodset| \geq (1-\eps)n$
of ``good'' functions $\{f_i\}_{i \in \goodset}$ and parameters $\sigma_0, \sigma_2 \in \R_+$ 
such that such that the following conditions simultanously hold:
\begin{itemize}
	\item
Equation (\ref{eq:norm-bound}) holds with $\sigma_1 = 0$ and the same $\sigma_0$, and
	\item
	The following equations hold:
\begin{align} 
\|\nabla \hatf(\wstar) - \nabla \Ef(\wstar)\|_2 \leq \sigma_0 \sqrt{\badfrac} \; ~\mathrm{, and} \label{eq:mean-at-min} \\
|\hatf(w)-\Ef(w)| \leq \sigma_2\sqrt{\eps},  \textrm{ for all } w \in \dom \; ,  \label{eq:loss-close}
\end{align}
where $\hatf \eqdef \frac{1}{|\goodset|} \sum_{i \in \goodset} f_i$.
\end{itemize}
\end{assumption}

\noindent In this section, we will show the following two statements.
The first demonstrates that Assumption~\ref{ass:one-good-set-glm} implies that \sever{} succeeds, and the second shows that Assumption~\ref{ass:one-good-set-glm} holds after polynomially many samples from a GLM.
Formally:
\begin{theorem}\label{thm:glms-sever}
 For functions $f_1, \ldots, f_n: \dom \to \R$, suppose that Assumption~\ref{ass:one-good-set-glm} 
 holds and that $\dom$ is convex.
 Then, for some universal constant $\badfrac_0$, if $\badfrac < \badfrac_0$, 
 there is an algorithm which, with probability at least $9/10$, finds a $w \in \dom$ such that
       \[
       \Ef(w) - \Ef(\wstar) = r (\gamma + O(\sigma_0 \sqrt{\eps})) + O(\sigma_2 \sqrt{\eps}) \; .
       \]

If the link functions are $\strconv$-strongly convex,  the algorithm finds a $w \in \dom$ such that
       \[
       \Ef(w) - \Ef(\wstar) = 2 \frac{(\gamma + O(\sigma_0 \sqrt{\eps}))^2}{\strconv} + O(\sigma_2 \sqrt{\eps}) \; .
       \]
       % ,
       % provided that $\strconv \geq C \cdot \sigma_1 \sqrt{\badfrac}$ for an absolute constant $C$,
 \end{theorem}

\begin{proposition} \label{prop:sample-GLM}
Let $\dom \subseteq \R^d$ and let $\cY$ be an arbitrary set.
Let $f_1,\ldots,f_n$ be obtained by picking $f_i$ i.i.d. at random from a GLM $\Ef$ over $\dom \times \cY$ with distribution $D_{xy}$ and link functions $\sigma_Y$, where
\[
n = \Omega \left( \frac{d\log(dr/\eps)}{\eps} \right) \; .
\]
Suppose moreover that the following conditions all hold:
\begin{enumerate}
	\item
 $E_{X \sim D_{xy}} [XX^T] \preceq I$,
 \item
 $|\sigma_Y' (t)| \leq 1$ for all $Y \in \cY$ and $t \in \R$, and
 \item
 $|\sigma_Y (0)| \leq 1$ for all $Y \in \cY$.
\end{enumerate}
 Then with probability at least $9 / 10$ over the original set of samples, there is a set of $(1-\eps)n$ of the $f_i$ that satisfy Assumption \ref{ass:one-good-set-glm} on $\mathcal{H}$ with $\sigma_0=2$, $\sigma_1=0$ and $\sigma_2=1+r$.and $\sigma_2=1+r$.
\end{proposition}

\subsection{Proof of Theorem~\ref{thm:glms-sever}}
As before, since \sever{} either terminates or throws away at least one sample, clearly it cannot run for more than $n$ iterations.
Thus the runtime bound is simple, and it suffices to show correctness.

We first prove the following lemma:
\begin{lemma}
\label{lem:stationary-point-glm}
Let $f_1, \ldots, f_n$ satisfy Assumption~\ref{ass:one-good-set-glm}.
Then with probability at least $9 / 10$, \sever{} applied to $f_1, \ldots, f_n, \sigma_0$  returns a point $w \in \dom$ which is a $(\gamma + O(\sigma_0 \sqrt{\eps}))$-approximate critical point of $\fhat$.
\end{lemma}
\begin{proof}
We claim that the empirical distribution over $f_1, \ldots, f_n$ satisfies Assumption~\ref{ass:one-good-set} for the function $\fhat$ with $\sigma_0$ as stated and $\sigma_1 = 0$, with the $\goodset$ in Assumption~\ref{ass:one-good-set} being the same as in the definition of Assumption~\ref{ass:one-good-set-glm}.
Clearly these functions satisfy \eqref{eq:mean-everywhere} (since the LHS is zero), so it suffices to show that they satisfy \eqref{eq:norm-bound}
Indeed, we have that for all $w \in \dom$,
\[
\E_{\goodset} [(\nabla f_i (w) - \nabla \fhat (w)) (\nabla f_i (w) - \nabla \fhat (w))^\top ] \preceq \E_{\goodset} [(\nabla f_i (w) - \nabla \fbar (w)) (\nabla f_i (w) - \nabla \fbar (w))^\top ] \; ,
\]
so they satisfy \eqref{eq:norm-bound}, since the RHS is bounded by Assumption~\ref{ass:one-good-set-glm}.
Thus this lemma follows from an application of Theorem~\ref{thm:stationary-point}.
\end{proof}

\noindent With this critical lemma in place, we can now prove Theorem~\ref{thm:glms-sever}:
\begin{proof}[Proof of Theorem~\ref{thm:glms-sever}]
Condition on the event that Lemma~\ref{lem:stationary-point-glm} holds, and let $w \in \dom$ be the output of \sever.
By Assumption~\ref{ass:one-good-set-glm}, we know that $\fhat(w^*) \geq \fbar(w^*) - \sigma_2 \sqrt{\eps}$, and moreover, $w^*$ is a $\gamma +\sigma_0 \sqrt{\eps}$-approximate critical point of $\fhat$.

Since each link function is convex, so is $\fhat$.
Hence, by Lemma~\ref{lem:derivatives-to-sizes}, since $w$ is a $(\gamma + O(\sigma_0 \sqrt{\eps}))$-approximate critical point of $\fhat$, we have $\fhat(w) - \fhat(w^*) \leq r (\gamma + O(\sigma_0 \sqrt{\eps}))$.
By Assumption~\ref{ass:one-good-set}, this immediately implies that $\Ef(w) - \Ef (w^*) \leq r (\gamma + O(\sigma_0 \sqrt{\eps})) + O(\sigma_2 \sqrt{\eps})$, as claimed.

The bound for strongly convex functions follows from the exact argument, except using the statement in Lemma~\ref{lem:derivatives-to-sizes} pertaining to strongly convex functions.
\end{proof}

\subsection{Proof of Proposition~\ref{prop:sample-GLM}}

\begin{proof}
We first note that $\nabla f_{X,Y}(w) = X \sigma_Y'(w\cdot X).$
Thus, under Assumption~\ref{ass:one-good-set-glm}, we have for any $v$ that
$$
\E_i[(v\cdot(\nabla f_i(w) -\nabla \bar{f}(w)))^2] \ll \E_i[(v \cdot \nabla f_i(w))^2]+1 \ll \E_i[(v\cdot X_i)^2]+1 \;.
$$
In particular, since this last expression is independent of $w$, we only need to check this single matrix bound.

We let our good set be the set of samples with $|X|\leq 80\sqrt{d/\eps}$ that were not corrupted. We use Lemma A.18 of~\cite{DKK+17}. This shows that with $90\%$ probability that the non-good samples make up at most an $\eps/2+\eps/160$-fraction of the original samples, and that $\E[XX^T]$ over the good samples is at most $2I$. This proves that the spectral bound holds everywhere. Applying it to the $\nabla f_{X,Y}(w^{\ast})$, we find also with $90\%$ probability that the expectation over all samples of $\nabla f_{X,Y}(w^{\ast})$ is within $\sqrt{\eps}/3$ of $\nabla \bar{f}(w^{\ast})$. Additionally, throwing away the samples with $|\nabla f_{X,Y}(w^{\ast})-\nabla \bar{f}(w^{\ast})| > 80\sqrt{d/\eps}$ changes this by at most $\sqrt{\eps}/2$. Finally, it also implies that the variance of $\nabla f_{X,Y}(w^{\ast})$ is at most $3/2I$, and therefore, throwing away any other $\eps$-fraction of the samples changes it by at most an additional $\sqrt{3\eps/2}$.

We only need to show that $\left|\E_{i \ \mathrm{good}}[f_i(w)]-\E_X[f_X(w)]\right|\leq \sqrt{\eps}$ for all $w\in \mathcal{H}$. For this we note that since the $f_X$ and $f_i$ are all $1$-Lipschitz, it suffices to show that $\left|\E_{i \ \mathrm{good}}[f_i(w)]-\E_X[f_X(w)]\right|\leq (1+|w|)\sqrt{\eps}/2$ on an $\eps/2$-cover of $\mathcal{H}$. For this it suffices to show that the bound will hold pointwise except with probability $\exp(-\Omega(d\log(r/\eps)))$. We will want to bound this using pointwise concentration and union bounds, but this runs into technical problems since very large values of $X\cdot w$ can lead to large values of $f$, so we will need to make use of the condition above that the average of $X_iX_i^T$ over our good samples is bounded by $2I$. In particular, this implies that the contribution to the average of $f_i(w)$ over the good $i$ coming from samples where $|X_i\cdot {w}| \geq 10|w|/\sqrt{\eps}$ is at most $\sqrt{\eps}(1+|w|)/10$. We consider the average of $f_i(w)$ over the remaining $i$. Note that these values are uniform random samples from $f_X(w)$ conditioned on $|X|\leq 80\sqrt{d/\eps}$ and $|X_i\cdot {w}| < 10|w|/\sqrt{\eps}$. It will suffices to show that taking $n$ samples from this distribution has average within $(1+|w|)\sqrt{\eps}/2$ of the mean with high probability. However, since $|f_X(w)|\leq O(1+|X\cdot w|)$, we have that over this distribution $|f_X(w)|$ is always $O(1+|w|)/\sqrt{\eps}$, and has variance at most $O(1+|w|)^2$.
Therefore, by Bernstein's Inequality, the probability that $n$ random samples from $f_{X}(w)$ (with the above conditions on $X$) differ from their mean by more than $(1+|w|)\sqrt{\eps}/2$ is
$$\exp(-\Omega(n^2(1+|w|)^2\eps/((1+|w|)^2+n(1+|w|)^2)))=\exp(-\Omega(n\eps)).$$
Thus, for $n$ at least a sufficiently large multiple of $d\log(dr/\eps)/\eps$, this holds for all $w$ in our cover of $\mathcal{H}$ with high probability. This completes the proof.
\end{proof}

%!TEX root = ./main.tex

\section{An Alternative Algorithm: Robust Filtering in Each Iteration}
\label{sec:general-algo}

In this section, we describe another algorithm for robust stochastic optimization. This algorithm uses standard robust mean estimation techniques to compute approximate gradients pointwise, which it then feeds into a standard projective gradient descent algorithm. This algorithm in practice turns out to be somewhat slower than the one employed in the rest of this paper, because it employs a filtering algorithm at every step of the projective gradient descent, and does not remember which points were filtered between iterations. On the other hand, we present this algorithm for two reasons. Firstly, because it is a conceptually simpler interpretation of the main ideas of this paper, and secondly, because the algorithm works under somewhat more general assumptions. In particular, this algorithm only requires that for each $w\in\dom$ that there is a corresponding good set of functions, rather than that there exists a single good set that works simultaneously for all $w$.

In particular, we can make do with the following somewhat weaker assumption:
\begin{assumption}
\label{ass:many-good-sets}
Fix $0<\eps<1/2$ \new{and parameter $\sigma \in \R_+$.}
For each $w \in \dom$, there exists an unknown set $\goodset \new{= \goodset(w)}  \subseteq [n]$
with $|\goodset| \geq (1-\eps)n$
 of ``good'' functions $\{f_i\}_{i \in \goodset}$ such that:
\begin{equation}
\Big\|\E_{\goodset}\big[\big(\nabla f_i(w) - \nabla \Ef(w)\big)\big(\nabla f_i(w) - \nabla \Ef(w)\big)^T\big]\Big\|_2 \leq \sigma^2 \;,
\end{equation}
and
\begin{equation}
\|\nabla \hatf(w) - \nabla \Ef(w)\|_2 \leq \sigma\sqrt{\badfrac},
\textrm{ where } \hatf \eqdef \frac{1}{|\goodset|} \sum_{i \in \goodset} f_i \;.
\end{equation}
\end{assumption}

We make essential use of the following result, which appears in both~\cite{DKK+17, steinhardt2018resilience}:
\begin{theorem}
\label{thm:mean-estimation}
Let $\mu \in \bR^d$ and a collection of points $x_i \in \bR^d$, $i \in [n]$ and $\sigma>0$.
Suppose that there exists $\goodset \subseteq [n]$ \new{with $|\goodset|  \geq (1-\eps)n$}
satisfying the following:
\begin{equation}
\label{eq:mean-estimation-assumptions}
\frac{1}{|\goodset|} \sum_{i \in \goodset} (x_i - \mu)(x_i - \mu)^{\top} \preceq \sigma^2 I \text{ and } \big\|\frac{1}{|\goodset|} \sum_{i \in \goodset} (x_i - \mu)\big\|_2 \leq \sigma \sqrt{\epsilon}.
\end{equation}
Then, if $\badfrac < \badfrac_0$ for some universal constant $\badfrac_0$,
there is an efficient algorithm, Algorithm ${\cal A}$, which outputs an estimate $\hat{\mu} \in \bR^d$ such that
$\|\hat{\mu} - \mu\|_2 = O(\sigma \sqrt{\epsilon})$.
\end{theorem}

Our general robust algorithm for stochastic optimization will make calls to Algorithm ${\cal A}$
in a black-box manner, as well as to the projection operator onto $\dom$.
We will measure the cost of our algorithm by the total number of such calls.

\begin{remark}
While it is not needed for the theoretical results established in this subsection,
we note that the robust mean estimation algorithm of~\cite{DKK+17}
relies on an iterative outlier removal method only requiring basic eigenvalue computations
(SVD), while the~\cite{steinhardt2018resilience} algorithm employs semidefinite programming.
In our experiments, we use the algorithm in~\cite{DKK+17} and variants thereof.
\end{remark}

\medskip
Using the above black-box, together with known results on convex optimization with errors, we obtain
the following meta-theorem:
\begin{theorem}\label{thm:opt-theorem}
For functions $f_1, \ldots, f_n: \dom \to \R$, bounded below on a closed domain $\dom$, suppose that either
Assumption \ref{ass:many-good-sets} is satisfied with some parameters $\eps, \sigma>0$. Then there exists an efficient algorithm that finds an $O(\sigma\sqrt{\eps})$-approximate critical point of $\Ef$.
\end{theorem}
\begin{proof}
We note that by applying Algorithm $\mathcal{A}$ on $\{\nabla f_i(w)\}$, we can find an approximation to $\nabla \Ef(w)$ with error $O(\sigma\sqrt{\eps})$. We note that standard projective gradient descent algorithms can be made to run efficiently even if the gradients given are only approximate, and this can be used to find our $O(\sigma\sqrt{\eps})$-approximate critical point.
\end{proof}

\subsection{Comparison with \sever{}}
\label{sec:generic-vs-sever}

In this section we give a brief comparison of this algorithm to \sever{}.
The algorithm presented in this section is much simpler to state, and also requires weaker conditions on the data.
However, because the algorithms work in somewhat different settings, the comparison is a bit delicate, so we explain in more detail below.

There is a major conceptual difference between these two algorithms: namely, \sever{} works with a \emph{black-box non-robust learner}, and requires the filter algorithm for robust mean estimation.
In contrast, the algorithm in this section works with a \emph{black-box robust mean estimation algorithm}, and plugs it into a specific non-robust learning algorithm, specifically (approximate) stochastic gradient descent.
When instantiated with the same primitives, these algorithms have similar theoretical runtime guarantees.

However, in the practical implementation, we prefer \sever{} for a couple of reasons.
First, in practice we find that in practice, \sever{} often only requires a constant number of runs of the base black-box learner, and so incurs only a constant factor overhead.
In contrast, the algorithm presented in this section requires at least linear time per iteration of SGD, since it needs to run a robust mean estimation algorithm on the entire dataset (and the total number of iterations needed is comparable).
In contrast, SGD typically runs in constant time per iteration, so this presents a major bottleneck for scalability.

Second, we find it is much more useful from a practical point of view to allow for black-box non-robust learners, than to allow for black-box robust mean estimation algorithms.
This is simply because the former allows \sever{} to be much more problem-specific, and allow for optimizations for each individual learning problem.
For instance, it is what allows us to use optimized libraries in our experiments for linear regression and SVM.
In contrast, there is relatively little reward to allow for black-box robust mean estimation algorithms, as not only are there relatively few options, but also this does not allow us really to specialize the algorithm to the problem at hand.
\section{Applications of the General Algorithm}
\label{sec:app-general}

In this section, we present three concrete applications of our general robust algorithm.
In particular, we describe how to robustly optimize models for linear regression, support vector machines, and logistic regression, in Sections~\ref{sec:app-linreg},~\ref{sec:app-svm},~\ref{sec:app-logreg}, respectively.

\subsection{Linear Regression}
\label{sec:app-linreg}
In this section, we demonstrate how our results apply to linear regression.
We are given pairs $(X_i, Y_i) \in \R^{d} \times \R$ for $i \in [n]$.
The $X_i$'s are drawn i.i.d.\ from a distribution $D_x$, and $Y_i = \langle w^*, X_i \rangle + e_i$, 
for some unknown $w^* \in \R^d$ and the noise random variables $e_i$'s are drawn i.i.d.\ from some distribution $D_e$.
Given $(X_i, Y_i) \sim D_{xy}$, the joint distribution induced by this process, let $f_i(w) = (Y_i - \langle w, X_i \rangle)^2$.
The goal is then to find a $\what$ approximately minimizing the objective function
\[
\Efunc(w) = \E_{(X, Y) \sim D_{xy}} [(Y - \langle w, X \rangle)^2] \;.
\]

We work with the following assumptions:
\begin{assumption}
\label{ass:linreg}
Given the model for linear regression described above, assume the following conditions for $D_e$ and $D_x$:
\begin{itemize}
\item $\E_{e \sim D_e}[e] = 0$;
\item $\Var_{e \sim D_e}[e] \leq \xi$;
\item $\E_{X \sim D_x}[XX^T] \preceq \sigma^2 I $ for some $\sigma > 0$;
\item There is a constant $C > 0$, such that for all unit vectors $v$, $\E_{X \sim D_x} \left[ \langle v, X \rangle^4 \right] \leq C \sigma^4.$
\end{itemize}
\end{assumption}

Our main result for linear regression is the following:
\begin{theorem}\label{thm:linreg}
Let $\eps > 0$, and let $D_{xy}$ be a distribution over pairs $(X,Y)$ which satisfies the conditions of Assumption~\ref{ass:linreg}.
Suppose we are given $O\left(\frac{d^5}{\eps^2}\right)$ $\eps$-noisy samples from $D_{xy}$.
Then in either of the following two cases, there exists an algorithm that, with probability at least $9/10$, produces a $\what$ with the following guarantees:
\begin{enumerate}
\item If $\E_{X \sim D_x}[XX^T] \succeq \gamma I$ for $\gamma \geq \oo(1) \cdot \sigma \sqrt{C \eps}$, then $\Efunc(\what) \leq \Efunc(\wstar) + O\left(\frac{(\xi+\eps)\eps}{\gamma}\right)$ and $\|\what-\wstar\|_2 = O\left(\frac{\sqrt{\xi\eps}+\eps}{\gamma}\right)$.
\item If $\|\wstar\|_2 \leq r$, then $\Efunc(\what) \leq \Efunc(\wstar) + O(((\sqrt{\xi}+\sqrt{\eps}) \radius + \sqrt{C}\radius^2)\sqrt{\badfrac})$.
\end{enumerate}
\end{theorem}

The proof will follow from two lemmas (proved in Section~\ref{sec:linreg-cov-bound} and~\ref{sec:linreg-sample-complexity}, respectively).
First, we will bound the covariance of the gradient, in Lemma~\ref{lem:lin-reg-cov-bound}:
\begin{lemma}
\label{lem:lin-reg-cov-bound}
Suppose $D_{xy}$ satisfies the conditions of Assumption~\ref{ass:linreg}.
Then for all unit vectors $v \in \R^d$, we have
\[
v^\top \Cov_{(X, Y) \sim D_{xy}} \left[ \nabla f_i (w,(X,Y)) \right] v \leq   4\sigma^2 \xi + 4C\sigma^4 \|w^* - w\|_2^2 \;.
\]
\end{lemma}

With this in hand, we can prove Lemma~\ref{lem:linreg-sample-complexity}, giving us a polynomial sample complexity which is sufficient to satisfy the conditions of Assumption~\ref{ass:one-good-set}.
\begin{lemma} \label{lem:linreg-sample-complexity}
Suppose $D_{xy}$ satisfies the conditions of Assumption~\ref{ass:linreg}.
Given $O(d^5/\eps^2)$ $\eps$-noisy samples from $D_{xy}$, then with probability at last $9/10$, they satisfy Assumption \ref{ass:one-good-set} with parameters $\sigma_0=30\sqrt{\xi}+\sqrt{\eps}$ and $\sigma_1=18\sqrt{C+1}$.
\end{lemma}

The proof concludes by applying Corollary~\ref{cor:strongly-convex-sever} or case (i) of Corollary~\ref{cor:convex-sever} for the first and second cases respectively.

\subsubsection{Proof of Lemma~\ref{lem:lin-reg-cov-bound}}
\label{sec:linreg-cov-bound}
Note that for this setting we have that $f(w, z) = f(w, x, y) = (y - \langle w, x \rangle)^2$.
We then have that $\nabla_w f(w, z) = -2 (\langle w^{\ast}-w, x \rangle + e) x$.
Our main claim is the following:

\begin{claim} \label{claim:cov-grad-formula}
We have that $\Cov[\nabla_w f(w, z)] = 4 \E_{X \sim D} \left[ \langle w^{\ast}-w, x \rangle^2 (xx^T) \right]  + 4\Var[E] \Sigma
- 4 \Sigma (w^{\ast}-w)(w^{\ast}-w)^T \Sigma$.
\end{claim}

\begin{proof}
Let us use the notation $A = \nabla_w f(w, z)$ and $\mu = \E[A]$.
By definition, we have that
$\Cov[A] = \E[AA^T] - \mu \mu^T$.

Note that $\mu = \E_z [\nabla_w f(w, z)] = \E_z [ (- 2 \langle w^{\ast}-w, x \rangle + e) x] = -2 \Sigma (w^{\ast}-w)$,
where we use the fact that $\E_z[e] = 0$ and $e$ is independent of $x$.
Therefore, $\mu \mu^T = 4 \Sigma (w^{\ast}-w) (w^{\ast}-w)^T \Sigma$.

To calculate $\E[AA^T]$, note that
$A =  \nabla_w f(w, z) = -2 (\langle w^{\ast}-w,  x \rangle + e)x$, and
$A^T = -2 (\langle w^{\ast}-w, x \rangle + e) x^T$.
Therefore,
$AA^T = 4 (\langle w^{\ast}-w, x \rangle^2 + e^2 + 2 \langle w^{\ast}-w, x \rangle e) (xx^T)$
and
$$\E_z [AA^T] = 4 \E_x [\langle w^{\ast}-w, x \rangle^2 (xx^T)] + 4 \Var[e] \Sigma + 0 \;,$$
where we again used the fact that the noise $e$ is independent of $x$ and its expectation is zero.

By gathering terms, we get that
$$\Cov[\nabla_w f(w, z)] =  4 \E_x [\langle w^{\ast}-w, x \rangle^2 (xx^T)] +
4 \Var[e] \Sigma - 4 \Sigma (w^{\ast}-w) (w^{\ast}-w)^T \Sigma \;.$$
This completes the proof.
\end{proof}

Given the above claim, we can bound from above the spectral norm of the covariance matrix
of the gradients as follows: Specifically, for a unit vector $v$,
the quantity $v^T \Cov[\nabla_w f(w, z)] v$ is bounded from above by a constant times the following
quantities:
\begin{itemize}
\item The first term is $v^T \E_x [\langle w^{\ast}-w, x \rangle^2 (xx^T)] v =
\E_x [\langle w^{\ast}-w, x \rangle^2  \cdot \langle v, x \rangle^2)]$.  By Cauchy-Schwarz and our 4th moment bound,
this is at most $C\sigma^4 \|w^{\ast}-w\|_2^2$, where $\Sigma \preceq \sigma^2 I$.
\item The second term is at most the upper bound of the variance of the noise $\xi$ times $\sigma^2$.
\item The third term is at most $v^T \Sigma (w^{\ast}-w) (w^{\ast}-w)^T \Sigma v$, which by our
bounded covariance assumption is at most $\sigma^4  \|w^{\ast}-w\|_2^2.$
\end{itemize}
This gives the parameters in the meta-theorem.

\subsubsection{Proof of Lemma~\ref{lem:linreg-sample-complexity}}
\label{sec:linreg-sample-complexity}

Let $S$ be the set of uncorrupted samples and $I$ be the subset of $S$ 
with $\|X\|_2 \leq 2\sqrt{d}/\eps^{1/4}$. We will take $\goodset$ to be the subset of $I$ that are not corrupted.

Firstly, we show that with probability at least $39/40$, at most an $\eps/2$-fraction of points in 
$S$ have $\|X\|_2 > 2\sqrt{d}/\eps^{1/4}$, and so $|\goodset| \geq (1-\eps)|S|$. 
Note that $\E_D[\|X\|_2^4]= \E_D[(\sum_{j=1}^d X_j^2)^2] \leq \sum_{j=1}^d \sum_{k=1}^d \sqrt{\E_D[X_j^2]\E[X_k^2]} \leq C d^2$, 
since $\E_{D_x}[X X^T] \preceq I$. Thus, by Markov's inequality, 
$\Pr_D[\|X\|_2 > 2\sqrt{d}(C/\eps)^{1/4}]=\Pr_D[\|X\|_2^4 > 16d^2/\eps] \leq \eps/16$. 
By a Chernoff bound, since $n \geq 10\eps^2$ this probability is at most $\eps/2$ for the uncorrupted samples with probability at least $39/40$. 
%This gives condition (i).

Next, we show that (\ref{eq:norm-bound}) holds with probability at least $39/40$. 
To do this, we will apply Lemma \ref{lem:lin-reg-cov-bound} to $\goodset$. 
Since $S$ consists of independent samples, 
the variance over the randomness of $S$ of $|S|\E_S[e^2]$ is at most $|S|\xi$. 
By Chebyshev's inequality, except with probability $1/99$, 
we have that  $\E_S[e^2] \leq 99\xi$ and since $\goodset \subset S$,  
$\E_{\goodset}[e^2] \leq |S|\E_S[e^2]/|I| \leq 100 \xi$. 
This is condition (i) of Lemma \ref{lem:lin-reg-cov-bound}.

We note that $I$ consists of $\Omega(d^5/\eps^2)$ independent samples from $D$ 
conditioned on $\|X\|_2 < 2\sqrt{d}/\eps^{1/4}$, a distribution that we will call $D'$.
Since the VC-dimension of all halfspaces in $\R^d$ is $d+1$, 
by the VC inequality, we have that, except with probability $1/80$,  
for any unit vector $v$ and $T \in \R$ that $|\Pr_I[v \cdot X > T] - \Pr_{D'}[v \cdot X > T]| \leq \eps/d^2$.
Note that for unit vector $v$ and positive integer $m$, $\E[(v.X)^m]=\int_0^\infty m(v \cdot X)^{m-1} \Pr[v \cdot X > T] dT$.
Thus we have that
\begin{align*}
\E_I[(v.X)^m] & = \int_0^\infty m(v \cdot X)^{m-1} \Pr_I[v \cdot X > T] dT \\
& \leq  \int_0^{2d^{1/2}(C/\eps)^{1/4}} m(v \cdot X)^{m-1} (\Pr_{D'}[v \cdot X > T] + \eps/d^2) dT \\
& = \E_{D'}[(v.X)^m] + (2d^{1/2}(C/\eps)^{1/4})^m (\eps/d^2) \\
& \leq (1+\eps)\E_D[(v.X)^m] + 2^m C^{m/4} (\eps/d^2)^{1-m/4} \;.
\end{align*}
Applying this for $m=2$ gives $\E_I[X X^T] \preceq (1 + \eps + 4\sqrt{C}\eps/d^2) I \preceq 2 I$ 
and with $m=4$ gives $ \E_I[(v.X)^4] \leq (1+\eps)C + 16C$.
Similar bounds apply to $\goodset$, with an additional $1+\eps$ factor.

Thus, with probability at least $39/40$, $\goodset$ satisfies the conditions of 
Lemma \ref{lem:lin-reg-cov-bound} with $\xi := 100 \xi$, $\sigma^2 := 2$ and $C := 5C$. 
Hence, it satisfies (\ref{eq:norm-bound}) with $\sigma_0=20\sqrt{\xi}$ and $\sigma_1=18\sqrt{C+1}$.

For (\ref{eq:mean-everywhere}), note that $\nabla_w f_i(w)=(w \cdot x_i - y_i)x_i =((w - \wstar ) \cdot x_i)x_i - e_i x_i$. 
We will separately bound $\|\E_{\goodset}[((w - \wstar ) \cdot X)X] - \E_D[((w - \wstar ) \cdot X)X]\|_2$ and $\|\E_{\goodset}[eX]-\E_D[eX]\|_2$.

We will repeatedly make use of the following, which bounds how much removing points or probability mass affects an expectation in terms of its variance:
\begin{claim} \label{clm:expect-remove-points}
For a mixture of distributions $P=(1-\delta)Q + \delta R$ for distributions $P,Q,R$ and a real valued function $f$, we have that $|\E_{X \sim P}[f(X)]-\E_{X \sim Q}[f(X)]| \leq 2\sqrt{\delta \E_{X \sim P} [f(X)^2]}/(1-\delta)$\end{claim}
\begin{proof} By Cauchy-Schwarz $|\E_{X \sim R} [f(X)]| \leq \sqrt{\E_{X \sim R} [f(X)^2]} \leq \sqrt{\E_{X \sim P} [f(X)^2]/\delta}$. Since $\E_{X \sim P}[f(X)]=(1-\delta)\E_{X \sim Q}[f(X)] + \delta \E_{X \sim R} [f(X)]$, this implies that $|\E_{X \sim P}[f(X)]/(1-\delta) -\E_{X \sim Q}[f(X)]| \leq \sqrt{\delta \E_{X \sim P} [f(X)^2]}/(1-\delta)$. However $|\E_{X \sim P}[f(X)]/(1-\delta) - \E_{X \sim P}[f(X)]| = (\delta/(1-\delta)) |\E_{X \sim P}[f(X)]| \leq \sqrt{\delta \E_{X \sim P} [f(X)^2]}/(1-\delta)$ and the triangle inequality gives the result.
\end{proof}
 We can apply this to $P=I$ and $Q=\goodset$ with $\delta=\eps/2$ 
 and also to $P=D$ and $Q=D'$ with $\delta=\eps/16$, with error $2\sqrt{\delta}/(1-\delta) \leq 2\sqrt{\eps}$ in either case.

For the first of term we wanted to bound, we have $\|\E_{\goodset}[((w - \wstar ) \cdot X)X] - \E_D[((w - \wstar ) \cdot X)X]\|_2 = \|(w-\wstar)^T \left(\E_{\goodset}[X X^T] - \E_D[X X^T]\right)\|_2 \leq \| w - \wstar\|_2 \|\E_{\goodset}[X X^T] - \E_D[X X^T]\|_2$. For any unit vector $v$, the VC dimension argument above gave that
$|\E_I[(v \cdot X)^2] - \E_{D'}[(v \cdot X)^2)]| \leq 4\sqrt{C}\eps/d^2$ and Claim \ref{clm:expect-remove-points} both gives that $|\E_I[(v \cdot X)^2] - \E_{\goodset}[(v \cdot X)^2)]| \leq 2 \sqrt{\eps \E_I[(v.X)^4]} \leq 10 \sqrt{C \eps}$ and that $|\E_D[(v \cdot X)^2] - \E_{D'}[(v \cdot X)^2)]| \leq 2 \sqrt{\eps \E_D[(v.X)^4]} \leq 2 \sqrt{C \eps}$. By the triangle inequality, we have that  $|\E_D[(v \cdot X)^2] - \E_{\goodset}[(v \cdot X)^2)]| \leq 16\sqrt{C}\eps$. Since this holds for all unit $v$ and the matrices involved are symmetric, we have that $ \|\E_{\goodset}[X X^T] - \E_D[X X^T]\|_2 \leq 16\sqrt{C\eps}$. The overall first term is bounded by $\|\E_{\goodset}[((w - \wstar ) \cdot X)X] - \E_D[((w - \wstar ) \cdot X)X]\|_2 \leq 16\sqrt{C\eps} \| w - \wstar\|_2$.

Now we want to bound the second term, $\|\E_{\goodset}[eX]-\E_D[eX]\|_2$. Note that $\E_D[eX]=\E_D[e]\E_D[X]=0$. So we need to bound $\E_{\goodset}[eX]$.
 First we bound the expectation and variance on $D'$ using Claim \ref{clm:expect-remove-points}. It yields that, for any unit vector $v$, $|\E_{D'}[e(v \cdot X)]| \leq 2 \sqrt{\eps \E_D[e^2(v \cdot X)^2]} \leq 2\sqrt{\eps \xi}$.

 Next we bound the expectation on $I$. Since $I$ consists of independent samples from $D'$, the covariance matrix over the randomness on $I$ of $|I|\E_I[eX-\E_{D'}[eX]]$ is $|I|\E_{D'}[(eX-\E_{D'}[eX])(eX-\E_{D'}[eX])^T] \leq |I|\E_{D'}[X X^T] \leq |I|(1+\eps)I$ and its expectation is $0$. Thus the expectation over the randomness of $I$ of $(|I|^2 \|\E_I[eX]-\E_{D'}[eX]\|_2)^2$ is $\mathrm{Tr}(|I|\E_{D'}[(eX-\E_{D'}[eX])(eX-\E_{D'}[eX])^T]) \leq |I|(1+\eps+4\xi\eps|I|)d$. By Markov's inequality, except with probability $1/40$, $\Pr[\|\E_I[eX]\|_2 \geq 2\sqrt{\xi \eps} + \eps] \leq d/|I|\eps^2$. Since $|I| \geq 40 d/\eps^2$. This happens with probability at least $1/40$.

 Next we bound the expectation on $\goodset$ which follows by a slight variation of Claim \ref{clm:expect-remove-points}.  Let $J=I-\goodset$. Then, for any $v$, $\E_J[e(v \cdot X)] \leq \sqrt{\E_J[e^2]\E_J[(v \cdot X)^2]} \leq \sqrt{\E_S[e^2]\E_I[(v \cdot X)^2]}|J|/\sqrt{|S||I|} \leq \sqrt{100\xi (1 + \eps + 4\sqrt{C}\eps/d^2)} |J|/\sqrt{|S||I|} \leq 20 |J|\sqrt{\xi/|S||I|}$ by bounds we obtained earlier. Now $\|\E_{\goodset}[eX]\|_2 = \|(|I|/|\goodset|)\E_{I}[eX] - (|J|/|\goodset|)\E_J[e X] \|_2 \leq 20\sqrt{\xi\eps}+\eps + (1+\eps)\sqrt{\xi\eps/16} \leq 30 \sqrt{\xi\eps} + \eps$.

 We can thus take $\sigma_0 = 30\sqrt{\xi}+\sqrt{\eps}$ and $\sigma_1=18\sqrt{C+1} \geq 16\sqrt{C}$ to get (\ref{eq:mean-everywhere}).

To get both (\ref{eq:mean-everywhere}) and (\ref{eq:norm-bound}) hold with $\sigma_0=30\sqrt{\xi}+\sqrt{\eps}$ and $\sigma_1=18\sqrt{C+1})$.
This happens with probability at least $9/10$ by a union bound on the probabilistic assumptions above.

\subsection{Support Vector Machines}
\label{sec:app-svm}
In this section, we demonstrate how our results apply to learning support vector machines (i.e., halfspaces under hinge loss).
In particular, we describe how SVMs fit into the GLM framework described in Section~\ref{sec:glms}.

We are given pairs $(X_i, Y_i) \in \R^{d} \times \{\pm 1\}$ for $i \in [n]$, which are drawn from some distribution $D_{xy}$.
Let $L(w,(x,y))=\max\{0, 1 - y(w \cdot x) \}$, and $f_i(w) = L(w, (x_i, y_i))$.
The goal is to find a $\what$ approximately minimizing the objective function
\[
\Efunc(w) = \E_{(X, Y) \sim D_{xy}} [L(w, (X, Y))].
\]

One technical point is that $f_i$ does not have a gradient everywhere -- instead,  we will be concerned with the sub-gradients of the $f_i$'s.
All our results which operate on the gradients also work for sub-gradients.
To be precise, we will take the sub-gradient to be $0$ when the gradient is undefined:
\begin{definition}
Let $\nabla f_i$ be the \emph{sub-gradient} of $f_i(w)$ with respect to $w$, where $\nabla f_i = -y_ix_i$ if $y_i (w \cdot x_i) < 1$, and $0$ otherwise.
\end{definition}

To get a bound on the error of hinge loss, we will need to assume the marginal distribution $D_x$ is anti-concentrated.
\begin{definition}
A distribution is \emph{$\delta$-anticoncentrated} if at most
an $O(\delta)$-fraction of its probability mass is within Euclidean distance
$\delta$ of any hyperplane.
\end{definition}
%G's note: the next line, while true, doesn't seem necessary?
%Note that if any one-dimensional projection of $D$ has probability density function bounded by a constant, then it is $\delta$-anticoncentrated for any $\delta > 0$.

We work with the following assumptions:
\begin{assumption}
\label{ass:svm}
Given the model for SVMs as described above, assume the following conditions for the marginal distribution $D_x$:
\begin{itemize}
\item $\E_{X \sim D_x}[XX^T] \preceq I $;
\item $D_x$ is $\eps^{1/4}$-anticoncentrated.
\end{itemize}
\end{assumption}

Our main result on SVMs is the following:

\begin{theorem} \label{thm:result-svm}
Let $\eps > 0$, and let $D_{xy}$ be a distribution over pairs $(X,Y)$, where the marginal distribution $D_x$ satisfies the conditions of Assumption~\ref{ass:svm}.
Then there exists an algorithm that with probability $9/10$, given $O(d\log(d/\eps)/\eps)$ $\eps$-noisy samples from $D_{xy}$, returns a $\what$ such that for any $w^{\ast}$, 
$$\E_{(X,Y) \sim D_{xy}} [ L(\what,(X,Y)) ] \leq \E_{(X,Y) \sim D_{xy}}[ L(w^{\ast},(X,Y)) ] + O(\eps^{1/4}).$$
\end{theorem}

Our approach will be to fit this problem into the GLM framework developed in Section~\ref{sec:glms}.
First, we will restrict our search over $w$ to $\dom$, a ball of radius $r=\eps^{-1/4}$.
As we argue in Lemma~\ref{lem:antigood-svm}, this restriction comes at a cost of at most $O(\eps^{1/4})$ in our algorithm's loss.
With this restriction, we will argue that the problem satisfies the conditions of Proposition~\ref{prop:sample-GLM}.
This allows us to argue that, with a polynomial number of samples, we can obtain a set of $f_i$'s satisfying the conditions of Assumption~\ref{ass:one-good-set-glm}.
This will allow us to apply Theorem~\ref{thm:glms-sever}, concluding the proof.

We start by showing that, due to anticoncentration of $D$, there is a $w' \in \dom$ with loss close to $\wstar$:
\begin{lemma} \label{lem:antigood-svm}
Let $w'$ be a rescaling of $\wstar$, such that $\|w'\|_2 \leq \eps^{-1/4}$ (i.e. $w' = \min\{1, \eps^{-1/4}/\|\wstar\|_2\} \wstar$). Then $\E_{(X,Y) \sim D_{xy}} [ L(w',(X,Y)) ] \leq \E_{(X,Y) \sim D_{xy}}[ L(w^{\ast},(X,Y)) ] + O(\eps^{1/4})$.
\end{lemma}
\begin{proof}
If $w'=w^{\ast}$, then $\E_{(X,Y) \sim D_{xy}} [ L(w',(X,Y)) ] = \E_{(X,Y) \sim D_{xy}}[ L(w^{\ast},(X,Y)) ]$.

Otherwise, we break into case analysis, based on the value of $(x,y)$:
\begin{itemize}
\item $|w' \cdot x| > 1$: If $y(w' \cdot x) > 1$, then $L(w',(x,y)) = L(\wstar, (x,y))= 0$. 
If $y(w' \cdot x) < -1$, then $L(w',(x,y)) = 1 - y(w' \cdot x) \leq 1 -y (\wstar \cdot x) = L(\wstar, (x,y))$.
Both cases use the fact that $\|w'\|_2 < \|w^{\ast}\|_2$.
\item $|w' \cdot x| \leq 1$: In this case, we have that $L(w', (x,y)) \leq 2$.
Since $L(\wstar, (x,y)) \geq 0$, we have that $L(w', (x,y)) \leq L(\wstar, (x,y)) + 2$.
\end{itemize}
Note that if $|w' \cdot x| \leq 1$, then $x$ is within $1/\|w'\|_2=\eps^{1/4}$ of the hyperplane defined by the normal vector $w'$.
Since $D_x$ is $\eps^{1/4}$-anticoncentrated, we have that $\Pr_{X \sim D_x}[|w' \cdot X| \leq 1] \leq \eps^{1/4}$. 
Thus, we have that $\E_{(X,Y) \sim D_{xy}}[L(w',(X,Y))] \leq \E_{(X,Y) \sim D_{xy}}[L(\wstar,(X,Y)) + 2 \cdot \mathbbm{1}(|w' \cdot X| \leq 1)] \leq \E_{(X,Y) \sim D_{xy}}[L(w^{\ast},(X,Y))] + O(\eps^{1/4})$.
\end{proof}

\begin{proof}[Proof of Theorem~\ref{thm:result-svm}]
We first show that this problem fits into the GLM framework, in particular, satisfying the conditions of Proposition~\ref{prop:sample-GLM}.
The link function is $\sigma_y(t) = \max\{0, 1 - y t \}$, giving us the loss function $L(w,(x,y)) = \sigma_y(w \cdot x)$.
We let $\dom$ be the set $\|w\|_2 \leq \eps^{-1/4}$, giving us the parameter $r = \eps^{-1/4}$. 
Condition 1 is satisfied by Assumption~\ref{ass:svm}.
For $y \in \{-1,1\}$, $\sigma'_y(t)=0$ for $yt \geq 1$ and $\sigma'_y(t)= -y$ for $yt < 1$. 
Thus we have that $|\sigma'_1(t)| \leq 1$ for all $t$ and $y$, satisfying Condition 2.
Finally, one can observe that $\sigma_y(0)=1$ for all $y$, satisfying Condition 3. 
Thus we can apply Proposition \ref{prop:sample-GLM}: if we take $O(d\log(dr/\eps)/\eps)$ $\eps$-corrupted samples, then they satisfy  Assumption \ref{ass:one-good-set-glm} on $\mathcal{H}$ with $\sigma_0=2$, $\sigma_1=0$ and $\sigma_2=1+\eps^{-1/4}$, with probability $9/10$.

Now we can apply the algorithm of Theorem~\ref{thm:glms-sever}.
Since the loss is convex, we get a vector $\what$ with
$\Ef(\what) - \Ef({\wstar}') = O((\sigma_0\radius + \sigma_1\radius^2+\sigma_2)\sqrt{\badfrac})
=O((2 \eps^{-1/4} +\eps^{-1/4}) \sqrt{\eps}) = O(\eps^{1/4})$ where ${\wstar}'$ is the minimizer of $\Ef$ on $\dom$.

We thus have that $\Ef(\hat w) \leq \Ef({\wstar}') + O(\eps^{1/4}) \leq \Ef(w')+O(\eps^{1/4})  \leq \Ef(\wstar)+O(\eps^{1/4})$.
The second inequality follows because ${\wstar}'$ is the minimizer of $\Ef$ on $\dom$, and the third inequality follows from Lemma~\ref{lem:antigood-svm}.
\end{proof}

\subsection{Logistic Regression}
\label{sec:app-logreg}
In this section, we demonstrate how our results apply to logistic regression.
In particular, we describe how logistic regression fits into the GLM framework described in Section~\ref{sec:glms}.

We are given pairs $(X_i, Y_i) \in \R^{d} \times \{\pm 1\}$ for $i \in [n]$, which are drawn from some distribution $D_{xy}$.
Let $\phi(t)=\frac{1}{1+\exp(-t)}$.
Logistic regression is the model where $y = 1$ with probability $\phi(w \cdot x)$, and $y = -1$ with probability $\phi(-w \cdot x)$.
We define the loss function to be the log-likelihood of $y$ given $x$.
More precisely, we let $f_i(w,(x_i,y_i)) = L(w,(x_i,y_i))$, which is defined as follows:
$$L(w,(x,y))= \frac{1+y}{2}\ln \left(\frac{1}{\phi(w \cdot x)}\right) + \frac{1-y}{2}\ln \left(\frac{1}{\phi(-w \cdot x)}\right)=\frac{1}{2}\left(-\ln\left(\phi(w \cdot x)\phi(-w \cdot x)\right) -y(w \cdot x)\right).$$
The gradient of this function is $\nabla L(w, (x,y)) = \frac12(\phi(w \cdot x)-\phi(-w \cdot x)-y)x$.
The goal is to find a $\what$ approximately minimizing the objective function
\[
\Efunc(w) = \E_{(X, Y) \sim D_{xy}} [L(w, (X, Y))].
\]

We work with the following assumptions:
\begin{assumption}
\label{ass:logreg}
Given the model for logistic regression as described above, assume the following conditions for the marginal distribution $D_x$:
\begin{itemize}
\item $\E_{X \sim D_x}[XX^T] \preceq I $;
\item $D_x$ is $\eps^{1/4}\sqrt{\log(1/\eps)}$-anticoncentrated.
\end{itemize}
\end{assumption}

We can get a similar result to that for hinge loss for logistic regression:
 \begin{theorem}\label{thm:logreg}
Let $\eps > 0$, and let $D_{xy}$ be a distribution over pairs $(X,Y)$, where the marginal distribution $D_x$ satisfies the conditions of Assumption~\ref{ass:logreg}.
Then there exists an algorithm that with probability $9/10$, given $O(d\log(d/\eps)/\eps)$ $\eps$-noisy samples from $D_{xy}$, returns a $\what$ such that for any $w^{\ast}$, 
$$\E_{(X,Y) \sim D_{xy}} [ L(\what,(X,Y)) ] \leq \E_{(X,Y) \sim D_{xy}}[ L(w^{\ast},(X,Y)) ] + O(\eps^{1/4}\sqrt{\log(1/\eps)}).$$
\end{theorem}

The approach is very similar to that of Theorem~\ref{thm:result-svm}, which we repeat here for clarity.
First, we will restrict our search over $w$ to $\dom$, a ball of radius $r=\eps^{-1/4}\sqrt{\log(1/\eps)}$.
As we argue in Lemma~\ref{lem:antigood-logit}, this restriction comes at a cost of at most $O(\eps^{1/4}\sqrt{\log(1/\eps)})$ in our algorithm's loss.
With this restriction, we will argue that the problem satisfies the conditions of Proposition~\ref{prop:sample-GLM}.
This allows us to argue that, with a polynomial number of samples, we can obtain a set of $f_i$'s satisfying the conditions of Assumption~\ref{ass:one-good-set-glm}.
This will allow us to apply Theorem~\ref{thm:glms-sever}, concluding the proof.

We start by showing that, due to anticoncentration of $D$, there is a $w' \in \dom$ with loss close to $\wstar$:
\begin{lemma} \label{lem:antigood-logit}
Let $w'$ be a rescaling of $\wstar$, such that $\|w'\|_2 \leq \eps^{-1/4} \sqrt{\ln(1/\eps)}$ (i.e. $w' = \min\{1, \eps^{-1/4}\sqrt{\ln(1/\eps)}/\|\wstar\|_2\} \wstar$). Then $\E_{(X,Y) \sim D_{xy}} [ L(w',(X,Y)) ] \leq \E_{(X,Y) \sim D_{xy}}[ L(w^{\ast},(X,Y)) ] + O(\eps^{1/4}\sqrt{\ln(1/\eps)})$.
\end{lemma}
\begin{proof}
We need the following claim:
\begin{claim}
\label{clm:logreg-clm}
$$|t| \leq -\ln(\phi(t)\phi(-t)) \leq |t| + 3 \exp(-|t|)$$
\end{claim}
\begin{proof}
Recalling that $\phi=1/(1+\exp(-t))$, we have that $-\ln(\phi(t)\phi(-t)) = \ln(\exp(t)+\exp(-t)+2)$.
Since $\exp(t)+\exp(-t)+2 \geq \exp(|t|)$, we have $|t| \leq -\ln(\phi(t)\phi(-t))$. 
On the other hand,
$\ln(\exp(t)+\exp(-t)+2) = |t|+\ln(1+2\exp(-|t|)+\exp(-2|t|)) \leq |t| + \ln(1+3\exp(-|t|)) \leq |t| + 3 \exp(-|t|)$.
\end{proof}

For any $x \in \R^d$, we have that:
\begin{align*}
-\ln(\phi(w' \cdot x)\phi(-w' \cdot x)) - y(w' \cdot x) -3\exp(-3|w' \cdot x|) 
&\leq |w' \cdot x| - y(w' \cdot x) \\
&\leq |\wstar \cdot x| - y(\wstar \cdot x) \\
&\leq -\ln(\phi(\wstar \cdot x)\phi(-\wstar \cdot x)) - y(\wstar \cdot x)
\end{align*}
The first and last inequality hold by Claim~\ref{clm:logreg-clm}.
For the second inequality, we do a case analysis on $y$.
When $y = \mathrm{sign}(w' \cdot x) = \mathrm{sign}(\wstar \cdot x)$, then both sides of the inequality are $0$.
When $y = -\mathrm{sign}(w' \cdot x) = -\mathrm{sign}(\wstar \cdot x)$, then the inequality becomes $2|w' \cdot x| \leq 2|\wstar \cdot x|$, which holds since $\|w'\|_2 \leq \|\wstar\|_2$.
We thus have that for any $y \in \{\pm 1\}$, $L(w',(x,y)) \leq L(\wstar,(x,y))+ \frac{3}{2}\exp(-3|w' \cdot x|)$.
If $|w' \cdot x| \leq \frac{1}{3}\ln(1/\eps)$, then  $L(w',(x,y)) \leq L(\wstar,(x,y))+ \frac{3}{2}$.
If  $|w '\cdot x| \geq \frac{1}{3}\ln(1/\eps)$, then $L(w',(x,y)) \leq L(\wstar,(x,y))+ \frac{3}{2}\eps$.
Since $\|w'\|_2 \leq \eps^{-1/4} \sqrt{\ln(1/\eps)}$ and $D_x$ is $\eps^{1/4}  \sqrt{\ln(1/\eps)}$-anticoncentrated, we have that $\Pr_{D_x}[|w' \cdot x| \leq \frac{1}{3}\ln(1/\eps)] \leq O(\eps^{1/4} \sqrt{\ln(1/\eps)})$.
Thus, $\E_{(X,Y) \sim D_{xy}} [ L(w',(X,Y)) ] \leq \E_{(X,Y) \sim D_{xy}}[ L(w^{\ast},(X,Y)) ] + O(\eps^{1/4}\sqrt{\ln(1/\eps)})$, as desired.
\end{proof}

With this in hand, we can conclude with the proof of Theorem~\ref{thm:logreg}.

\begin{proof}[Proof of Theorem~\ref{thm:logreg}]
We first show that this problem fits into the GLM framework, in particular, satisfying the conditions of Proposition~\ref{prop:sample-GLM}.
The link function is $\sigma_y(t) = \frac{1}{2}(-\ln(\phi(t)\phi(-t)) -yt)$, giving us the loss function $L(w,(x,y)) = \sigma_y(w \cdot x)$.
We let $\dom$ be the set $\|w\|_2 \leq \eps^{-1/4}\sqrt{\ln(1/\eps)}$, giving us the parameter $r = \eps^{-1/4}\sqrt{\ln(1/\eps)}$. 
Condition 1 is satisfied by Assumption~\ref{ass:logreg}.
For $y \in \{-1,1\}$, $\sigma'_y(t)=\frac{1}{2}(\phi(t)-\phi(-t)-y)$, which gives that $|\sigma'_y(t)| \leq 1$ for all $t$ and $y$, satisfying Condition 2.
Finally, $\sigma_y(0) = \ln 2 < 1$ for all $y$, satisfying Condition 3.
Thus we can apply Proposition \ref{prop:sample-GLM}: if we take $O(d\log(dr/\eps)/\eps)$ $\eps$-corrupted samples, then they satisfy  Assumption \ref{ass:one-good-set-glm} on $\mathcal{H}$ with $\sigma_0=2$, $\sigma_1=0$ and $\sigma_2=1+\eps^{-1/4}\sqrt{\ln(1/\eps)}$, with probability $9/10$.

Now we can apply the algorithm of Theorem~\ref{thm:glms-sever}.
Since the loss is convex, we get a vector $\what$ with
$\Ef(\what) - \Ef({\wstar}') = O((\sigma_0\radius + \sigma_1\radius^2+\sigma_2)\sqrt{\badfrac})
=O((2 \eps^{-1/4}\sqrt{\ln(1/\eps)} +\eps^{-1/4}\sqrt{\ln(1/\eps)}) \sqrt{\eps}) = O(\eps^{1/4}\sqrt{\ln(1/\eps)})$ where ${\wstar}'$ is the minimizer of $\Ef$ on $\dom$.

We thus have that $\Ef(\hat w) \leq \Ef({\wstar}') + O(\eps^{1/4}\sqrt{\ln(1/\eps)}) \leq \Ef(w')+O(\eps^{1/4}\sqrt{\ln(1/\eps)})  \leq \Ef(\wstar)+O(\eps^{1/4}\sqrt{\ln(1/\eps)})$.
The second inequality follows because ${\wstar}'$ is the minimizer of $\Ef$ on $\dom$, and the third inequality follows from Lemma~\ref{lem:antigood-logit}.
\end{proof}

\section{Additional Experimental Results}
\label{sec:experiments-app}
In this section, we provide additional plots of our experimental results, comparing with all baselines considered.

\begin{figure}[h!]
\centering
\begin{tikzpicture}

\begin{axis}[errplottriple,name=linreg_synth_app, align=center, title={Regression: Synthetic data}, xticklabel style={/pgf/number format/.cd, fixed, fixed zerofill, precision=2,/tikz/.cd}, legend columns= 4, legend style={anchor=north west, xshift=-0.15 \plotwidth, yshift=-0.8\plotheight}]
\addplot[teal] table[x=eps, y=err] {figures/linreg_synth/uncorrupted.txt};
\addplot[red] table[x=eps, y=err] {figures/linreg_synth/corrupted.txt};
\addplot[cyan] table[x=eps, y=err] {figures/linreg_synth/l2.txt};
\addplot[violet] table[x=eps, y=err] {figures/linreg_synth/loss.txt};
\addplot[magenta] table[x=eps, y=err] {figures/linreg_synth/gradient.txt};
\addplot[gray] table[x=eps, y=err] {figures/linreg_synth/ransac.txt};
\addplot[black] table[x=eps, y=err] {figures/linreg_synth/sever.txt};
\legend{uncorrupted, \noDef{}, \xCen{}, \Loss{}, \gCen{}, \ransac{}, \sever{}}
\end{axis}

\begin{axis}[errplottriple,name=linreg_qsar_app, align=center, at=(linreg_synth_app.north east),anchor=north west, xticklabel style={/pgf/number format/.cd, fixed, fixed zerofill, precision=2,/tikz/.cd}, xshift=\plotxspacing,ignore legend, title={Regression: Drug discovery data}, ymin = 1, ymax = 2]
\addplot[teal] table[x=eps, y=err] {figures/linreg_qsar/uncorrupted.txt};
\addplot[red] table[x=eps, y=err] {figures/linreg_qsar/corrupted.txt};
\addplot[cyan] table[x=eps, y=err] {figures/linreg_qsar/l2.txt};
\addplot[violet] table[x=eps, y=err] {figures/linreg_qsar/loss.txt};
\addplot[magenta] table[x=eps, y=err] {figures/linreg_qsar/gradient.txt};
\addplot[gray] table[x=eps, y=err] {figures/linreg_qsar/ransac.txt};
\addplot[black] table[x=eps, y=err] {figures/linreg_qsar/sever.txt};
\end{axis}

\begin{axis}[errplottriple,name=linreg_qsar_worst_app, align=center, at=(linreg_qsar_app.north east),anchor=north west, xticklabel style={/pgf/number format/.cd, fixed, fixed zerofill, precision=2,/tikz/.cd}, xshift=\plotxspacing,ignore legend, title={Regression: Drug discovery data, \\ attack targeted against \sever{}}, ymin = 1, ymax = 2]
\addplot[teal] table[x=eps, y=err] {figures/linreg_qsar_worst/uncorrupted.txt};
\addplot[red] table[x=eps, y=err] {figures/linreg_qsar_worst/corrupted.txt};
\addplot[cyan] table[x=eps, y=err] {figures/linreg_qsar_worst/l2.txt};
\addplot[violet] table[x=eps, y=err] {figures/linreg_qsar_worst/loss.txt};
\addplot[magenta] table[x=eps, y=err] {figures/linreg_qsar_worst/gradient.txt};
\addplot[gray] table[x=eps, y=err] {figures/linreg_qsar_worst/ransac.txt};
\addplot[black] table[x=eps, y=err] {figures/linreg_qsar_worst/sever.txt};
\end{axis}

\end{tikzpicture}
\caption{$\epsilon$ vs test error for baselines and \sever{} on synthetic data and the drug discovery dataset. The left and middle figures show that \sever{} continues to maintain statistical accuracy against our attacks which are able to defeat previous baselines. The right figure shows an attack with parameters chosen to increase the test error \sever{} on the drug discovery dataset as much as possible. Despite this, \sever{} still has relatively small test error.} 
\label{label:acc-vs-eps-linreg-app} 
\end{figure}
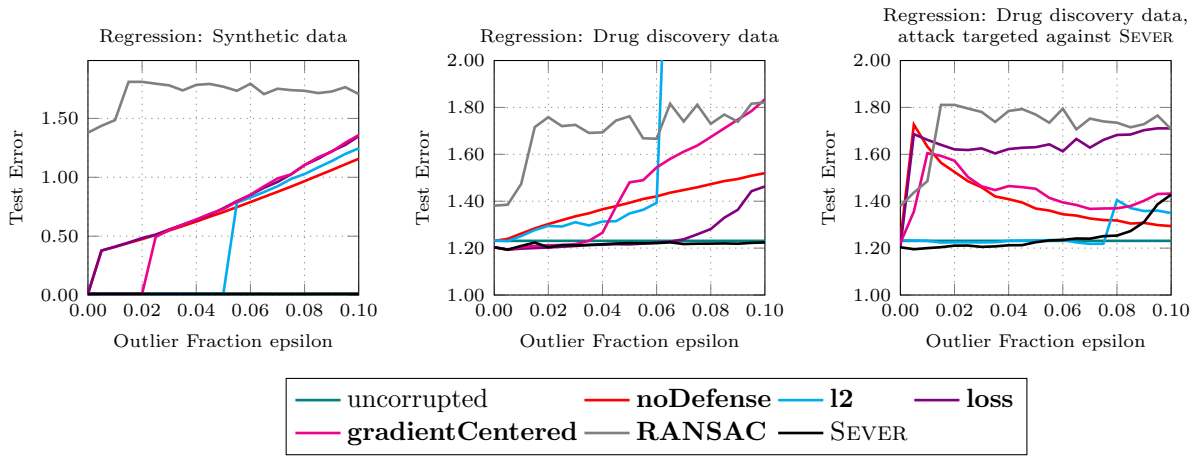

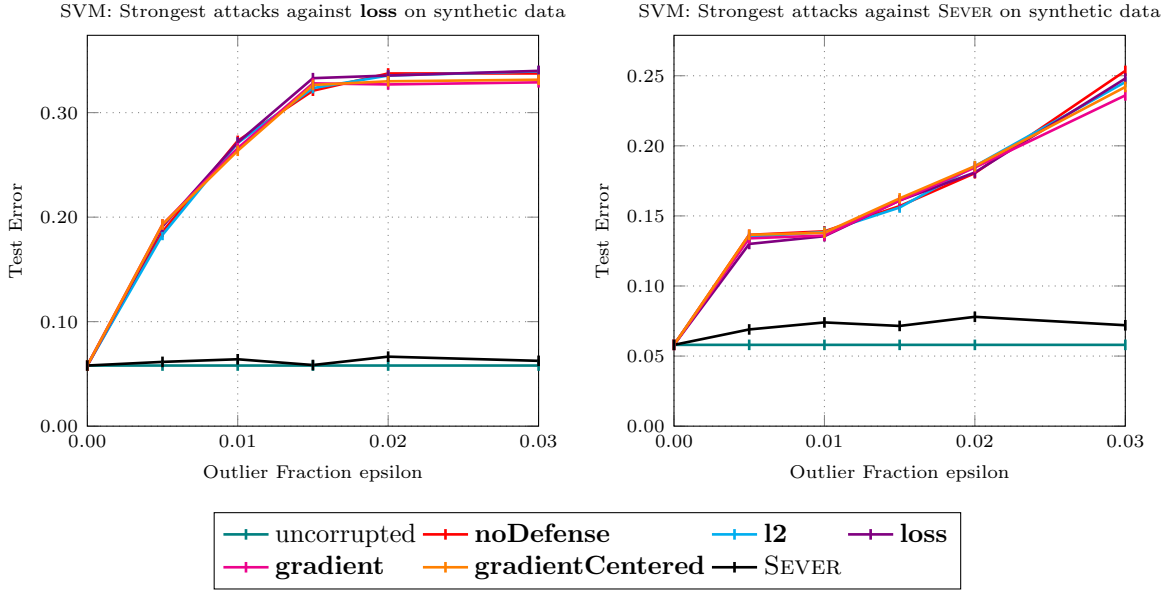
\begin{figure}[h!]
\centering
\begin{tikzpicture}

\begin{axis}[errplot,name=svm_synth_loss_app, legend style={anchor=north west, xshift=-0.7 \plotwidth, yshift=-1.2\plotheight}, legend columns=4, title={SVM: Strongest attacks against \Loss{} on synthetic data}]
\addplot[teal, mark=|] table[x=eps, y=err] {figures/svm_synth_loss/uncorrupted.txt};
\addplot[red, mark=|] table[x=eps, y=err] {figures/svm_synth_loss/corrupted.txt};
\addplot[cyan, mark=|] table[x=eps, y=err] {figures/svm_synth_loss/l2.txt};
\addplot[violet, mark=|] table[x=eps, y=err] {figures/svm_synth_loss/loss.txt};
\addplot[magenta, mark=|] table[x=eps, y=err] {figures/svm_synth_loss/gradient.txt};
\addplot[orange, mark=|] table[x=eps, y=err] {figures/svm_synth_loss/gradientCentered.txt};
\addplot[black, mark=|] table[x=eps, y=err] {figures/svm_synth_loss/sever.txt};
\legend{uncorrupted, \noDef{}, \xCen{}, \Loss{}, \gUnc{}, \gCen{}, \sever{}}
\end{axis}

\begin{axis}[errplot,name=svm_synth_sever_app, at=(svm_synth_loss.north east),anchor=north west, xshift=\plotxspacing,ignore legend, title={SVM: Strongest attacks against \sever{} on synthetic data}]
\addplot[teal, mark=|] table[x=eps, y=err] {figures/svm_synth_sever/uncorrupted.txt};
\addplot[red, mark=|] table[x=eps, y=err] {figures/svm_synth_sever/corrupted.txt};
\addplot[cyan, mark=|] table[x=eps, y=err] {figures/svm_synth_sever/l2.txt};
\addplot[violet, mark=|] table[x=eps, y=err] {figures/svm_synth_sever/loss.txt};
\addplot[magenta, mark=|] table[x=eps, y=err] {figures/svm_synth_sever/gradient.txt};
\addplot[orange, mark=|] table[x=eps, y=err] {figures/svm_synth_sever/gradientCentered.txt};
\addplot[black, mark=|] table[x=eps, y=err] {figures/svm_synth_sever/sever.txt};
\end{axis}

\end{tikzpicture}
\caption{$\epsilon$ vs test error for baselines and \sever{} on synthetic data. The left figure demonstrates that \sever{} is accurate when outliers manage to defeat previous baselines.
The right figure shows the result of attacks which increased the test error the most against \sever{}. Even in this case, \sever{} performs much better than the baselines.}
\label{fig:svm-synthetic-app}
\end{figure}

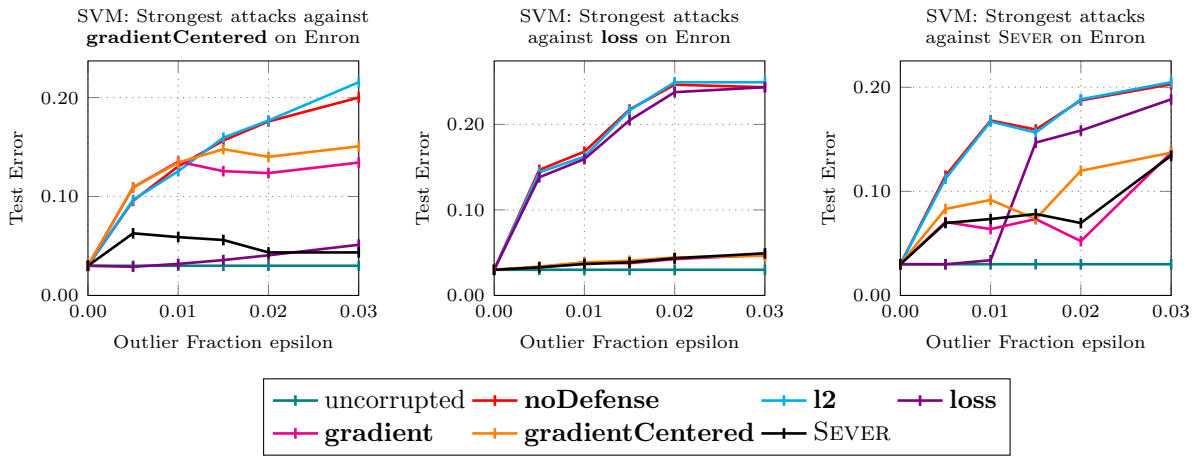
\begin{figure}[h!]
\centering
\begin{tikzpicture}

\begin{axis}[errplottriple,name=svm_enron_gradientCentered-app, align=center, title={SVM: Strongest attacks against \\ \gCen{} on Enron}, legend columns= 4, legend style={anchor=north west, xshift=-0.2 \plotwidth, yshift=-0.8\plotheight}]
\addplot[teal, mark=|] table[x=eps, y=err] {figures/svm_enron_gradientCentered/uncorrupted.txt};
\addplot[red, mark=|] table[x=eps, y=err] {figures/svm_enron_gradientCentered/corrupted.txt};
\addplot[cyan, mark=|] table[x=eps, y=err] {figures/svm_enron_gradientCentered/l2.txt};
\addplot[violet, mark=|] table[x=eps, y=err] {figures/svm_enron_gradientCentered/loss.txt};
\addplot[magenta, mark=|] table[x=eps, y=err] {figures/svm_enron_gradientCentered/gradient.txt};
\addplot[orange, mark=|] table[x=eps, y=err] {figures/svm_enron_gradientCentered/gradientCentered.txt};
\addplot[black, mark=|] table[x=eps, y=err] {figures/svm_enron_gradientCentered/sever.txt};
\legend{uncorrupted, \noDef{}, \xCen{}, \Loss{}, \gUnc{}, \gCen{}, \sever{}}
\end{axis}

\begin{axis}[errplottriple,name=svm_enron_loss-app, align=center, at=(svm_enron_gradientCentered-app.north east),anchor=north west, xshift=\plotxspacing,ignore legend, title={SVM: Strongest attacks \\ against \Loss{} on Enron}]
\addplot[teal, mark=|] table[x=eps, y=err] {figures/svm_enron_loss/uncorrupted.txt};
\addplot[red, mark=|] table[x=eps, y=err] {figures/svm_enron_loss/corrupted.txt};
\addplot[cyan, mark=|] table[x=eps, y=err] {figures/svm_enron_loss/l2.txt};
\addplot[violet, mark=|] table[x=eps, y=err] {figures/svm_enron_loss/loss.txt};
\addplot[magenta, mark=|] table[x=eps, y=err] {figures/svm_enron_loss/gradient.txt};
\addplot[orange, mark=|] table[x=eps, y=err] {figures/svm_enron_loss/gradientCentered.txt};
\addplot[black, mark=|] table[x=eps, y=err] {figures/svm_enron_loss/sever.txt};
\end{axis}
%\node [at=(svm_enron_loss.north east),anchor=south west,xshift=-6.5cm,yshift=.2cm] {SVM: Strongest attacks against loss on Enron};

\begin{axis}[errplottriple,name=svm_enron_sever-app, align=center, at=(svm_enron_loss-app.north east),anchor=north west, xshift=\plotxspacing,ignore legend, title={SVM: Strongest attacks \\ against \sever{} on Enron}]
\addplot[teal, mark=|] table[x=eps, y=err] {figures/svm_enron_sever/uncorrupted.txt};
\addplot[red, mark=|] table[x=eps, y=err] {figures/svm_enron_sever/corrupted.txt};
\addplot[cyan, mark=|] table[x=eps, y=err] {figures/svm_enron_sever/l2.txt};
\addplot[violet, mark=|] table[x=eps, y=err] {figures/svm_enron_sever/loss.txt};
\addplot[magenta, mark=|] table[x=eps, y=err] {figures/svm_enron_sever/gradient.txt};
\addplot[orange, mark=|] table[x=eps, y=err] {figures/svm_enron_sever/gradientCentered.txt};
\addplot[black, mark=|] table[x=eps, y=err] {figures/svm_enron_sever/sever.txt};
\end{axis}
%\node [at=(svm_enron_sever.north east),anchor=south west,xshift=-6.5cm,yshift=.2cm] {SVM: Strongest attacks against \sever{} on Enron};

\end{tikzpicture}
\caption{$\epsilon$ versus test error for baselines and \sever{} on the Enron spam corpus. 
The left and middle figures are the attacks which perform best against two baselines, while the right figure performs best against \sever{}. 
Though other baselines may perform well in certain cases, only \sever{} is consistently accurate. 
The exception is for certain attacks at $\epsilon = 0.03$, which, as shown in Figure~\ref{fig:spam-histogram}, require three rounds of outlier removal for any method to obtain reasonable test error -- in these plots, our defenses perform only two rounds.}
\label{fig:spam-results-app}
\end{figure}

\end{document}